\documentclass{article}

\usepackage{microtype}
\usepackage{graphicx}
\usepackage{booktabs} %
\usepackage{tablefootnote}

\usepackage{hyperref}

\usepackage[accepted]{icml2023}

\usepackage{amsmath}
\usepackage{amssymb}
\usepackage{mathtools}
\usepackage{amsthm}

\usepackage[capitalize,noabbrev]{cleveref}

\theoremstyle{plain}
\newtheorem{theorem}{Theorem}[section]

\newtheorem{lemma}[theorem]{Lemma}

\theoremstyle{definition}

\theoremstyle{remark}

\usepackage[textsize=tiny]{todonotes}

\usepackage{lipsum}  
\usepackage{etoc}
\usepackage{amsmath} 
\usepackage{amsthm}
\usepackage{bbm}
\usepackage{bm}
\usepackage{xspace}
\usepackage{caption}
\usepackage{amssymb, amsmath}
\usepackage{mathtools}
\usepackage{enumitem}
\usepackage[utf8]{inputenc}
\usepackage{graphicx}
\usepackage{nicefrac}
\usepackage{centernot}
\usepackage{array}
\usepackage{mathabx}
\usepackage{sidecap}
\usepackage{hyperref}
\usepackage{subcaption}
\usepackage{titlesec}
\usepackage{hhline}
\usepackage{float}
\usepackage{multirow}
\usepackage{multicol}
\usepackage{soul}

\newcommand{\x}{\mathbf{x}}

\newcommand{\w}{\mathbf{w}}

\newcommand{\sigmoid}{\text{sigmoid}}
\newcommand{\logit}{\text{logit}}
\newcommand{\ce}{\text{CE}}
\newcommand{\shp}{\text{SHP}}
\newcommand{\hops}{\text{HOPS}}

\newcommand{\ops}{\text{OPS}}
\newcommand{\obs}{\text{OBS}}
\newcommand{\tops}{\text{TOPS}}

\newcommand{\uvec}{\mathbf{u}}

\newcommand{\vvec}{\mathbf{v}}

\newcommand{\Rcal}{\mathcal{R}}

\newcommand{\Bcal}{\mathcal{B}}

\newcommand{\Ncal}{\mathcal{N}}

\newcommand{\Xcal}{\mathcal{X}}

\newcommand{\disc}{\text{disc}}

\newcommand{\floor}[1]{\left \lfloor{#1}\right \rfloor}

\newcommand{\abs}[1]{\left\lvert#1\right\rvert}

\newcommand{\enorm}[1]{\left\lVert#1\right\rVert_2}

\DeclareMathOperator*{\argmin}{arg\,min}

\newcommand{\Real}{\mathbb{R}}

\newcommand{\indicator}[1]{\mathbbm{1}\curlybrack{#1}}

\newcommand{\roundbrack}[1]{\left( #1 \right)}
\newcommand{\curlybrack}[1]{\left\lbrace #1 \right\rbrace}

\newcommand{\Exp}[2]{\mathbb{E}_{#1}\left\lbrack#2\right\rbrack}

\icmltitlerunning{Online Platt Scaling}

\date{\today}

\newcommand{\revision}[1]{{\textcolor{black}{#1}}}
\newcommand{\posticml}[1]{{\textcolor{black}{#1}}}

\begin{document}

\twocolumn[
\icmltitle{Online Platt Scaling with Calibeating}

\begin{icmlauthorlist}
\icmlauthor{Chirag Gupta}{cmu}
\icmlauthor{Aaditya Ramdas}{cmu}
\end{icmlauthorlist}

\icmlaffiliation{cmu}{Carnegie Mellon University, Pittsburgh PA, USA}

\icmlcorrespondingauthor{Chirag Gupta}{chiragg@cmu.edu, chiragpvg@gmail.com}

\vskip 0.3in
]
\printAffiliationsAndNotice{}

\begin{abstract}
\revision{We present an online post-hoc calibration method, called Online Platt Scaling (OPS), which combines the Platt scaling technique with online logistic regression. We demonstrate that OPS smoothly adapts between i.i.d.\ and non-i.i.d.\ settings with distribution drift. Further, in scenarios where the best Platt scaling model is itself miscalibrated, we enhance OPS by incorporating a recently developed technique called calibeating to make it more robust. Theoretically, our resulting OPS+calibeating method is guaranteed to be calibrated for adversarial outcome sequences. Empirically, it is effective on a range of synthetic and real-world datasets, with and without distribution drifts, achieving superior performance without hyperparameter tuning. Finally, we extend all OPS ideas to the beta scaling method.}
\end{abstract}

\section{Introduction}
\label{sec:intro}
In the past two decades, there has been significant interest in the ML community on post-hoc calibration of ML classifiers \citep{zadrozny2002transforming, niculescu2005predicting, guo2017calibration}. %
Consider a pretrained classifier $f : \Xcal \to [0,1]$ that produces scores in $[0,1]$ for covariates in $\Xcal$. Suppose $f$ is used to make probabilistic predictions for a sequence of points $(\x_t, y_t)_{t=1}^T$
where $y_t \in \{0,1\}$. Informally, $f$ is said to be calibrated \citep{dawid1982well} if the predictions made by $f$ match the empirically observed frequencies when those predictions are made:  %
\begin{equation}
   \text{ for all }p \in [0,1], \text{Average}\{y_t : f(\x_t) \approx p\} \approx p.%
   \label{eq:cal-informal}
\end{equation}
In practice, for well-trained $f$, larger scores $f(\x)$ indicate higher likelihoods of $y=1$, so that $f$ does well for accuracy or a ranking score like AUROC. Yet we often find that $f$ does not satisfy (some formalized version of) condition~\eqref{eq:cal-informal}. The goal of post-hoc calibration, or recalibration, is to use \textit{held-out data} to learn a low-complexity mapping $m : [0,1] \to [0,1]$ so that $m(f(\cdot))$ retains the good properties of $f$---accuracy, AUROC, sharpness---as much as possible, but is better calibrated than $f$. 

\begin{figure*}
    \begin{subfigure}[b]{0.27\textwidth}
    \centering
    \includegraphics[trim={0 0 0 0}, clip, width=\textwidth]{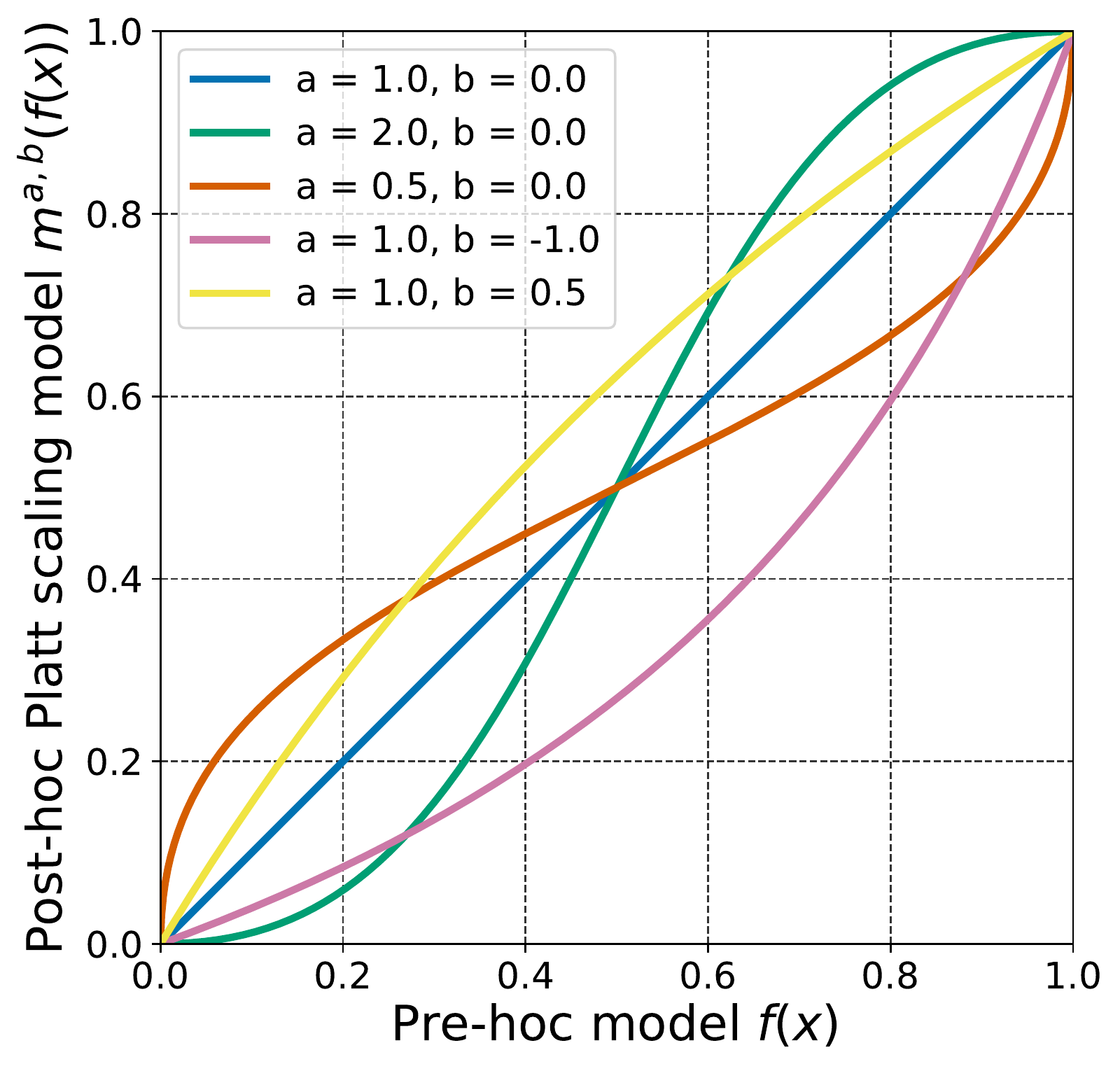}
    \caption{Platt scaling}
    \label{fig:intro-platt-figure}
    \end{subfigure}
    \hspace{0.3cm}
    \begin{subfigure}[b]{0.28\textwidth}
\fbox{%
    \parbox{\textwidth}{
    \leftskip 5pt
        \scriptsize
         Initialize weights $\w_1 \in \Real^d = \Xcal$\\
        At time $t = 1, 2, \ldots, T$
        \begin{itemize}[itemsep=3pt,topsep=3pt]
            \item Observe features $\x_t \in \Real^d$
            \item Predict $p_t = (1+ e^{-\w_t^\intercal \x_t})^{-1}$
            \item Observe  $y_t \in \{0,1\}$
            \item Compute updated weight $\w_{t+1} \in \Real^d$  
        \end{itemize}
        \vspace{0.1cm}
        \textbf{Goal: }minimize regret $\sum_{t=1}^T l(y_t, p_t)$, where \vspace{-0.1cm}
        \begin{equation*}
            l(y,p) = -y \log p - (1-y)\log(1-p). 
        \end{equation*} %
    \vspace{-0.4cm}
    }}
    \vspace{0.4cm}
    \caption{Online logistic regression}
    \label{fig:intro-olr}
    \end{subfigure}
    \hspace{0.3cm}
    \begin{subfigure}[b]{0.42\textwidth}
    \centering
    \includegraphics[width=\linewidth, trim=5cm 0 0 2cm, clip]{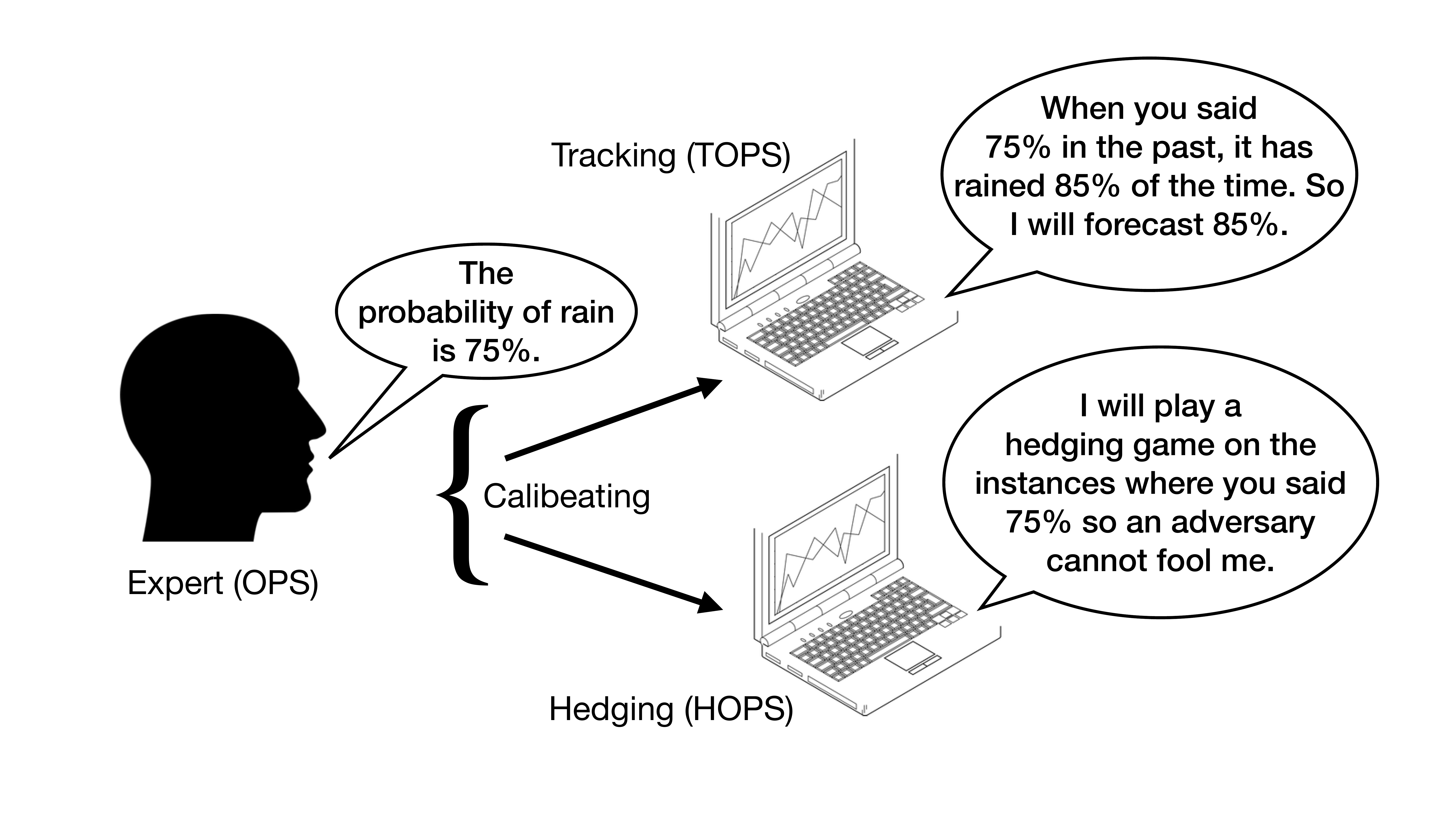}
    \caption{Calibeating applied to online Platt scaling}
    \label{fig:intro-calibeating}
    \end{subfigure}
    \caption{The combination of Platt scaling and online logistic regression yields Online Platt Scaling (OPS). Calibeating is applied on top of OPS to achieve further empirical improvements and theoretical validity.}
    \vspace{-0.3cm}
\end{figure*}

The focus of this paper is on a recalibration method proposed by \citet{platt1999probabilistic}, commonly known as Platt scaling (PS). The PS mapping $m$ is a sigmoid transform over $f$ parameterized by two scalars $(a, b) \in \Real^2$: %
\begin{equation}
    m^{a, b}(f(\x)) := \sigmoid(a\cdot\logit(f(\x)) + b).%
    \label{eq:ps-model-class}
\end{equation}
Here $\logit(z) = \log(\nicefrac{z}{1-z})$ and $\sigmoid(z) = \nicefrac{1}{1+e^{-z}}$ are inverses. Thus $m^{1, 0}$ is the identity mapping. 
\revision{Figure~\ref{fig:intro-platt-figure} has additional illustrative $m^{a,b}$ plots; these are easily interpreted---if $f$ is overconfident, that is if $f(x)$ values are skewed towards $0$ or $1$, we can pick $a \in (0,1)$ to improve calibration; if $f$ is underconfident, we can pick $a > 1$; if $f$ is systematically biased towards 0 (or 1), we can pick $b > 0$ (or $b < 0$). The counter-intuitive choice $a < 0$ can also make sense if $f$'s predictions oppose reality (perhaps due to a distribution shift).}
The parameters $(a, b)$ are typically set as those that minimize log-loss over calibration data or equivalently maximize log-likelihood under the model $y_i \overset{\text{iid}}{\sim} \text{Bernoulli}(m^{a, b}(f(\x_i)))$, on the held-out data.

\revision{Although a myriad of recalibration methods now exist, PS remains an empirically strong baseline. In particular, PS is effective when few samples are available for recalibration \citep{niculescu2005predicting, gupta2021distribution}. Scaling before subsequent binning has emerged as a useful methodology \citep{kumar2019calibration, zhang2020mix}. Multiclass adaptations of PS, called temperature, vector, and matrix scaling have become popular \citep{guo2017calibration}. Being a parametric method, however, PS comprises a limited family of post-hoc corrections---for instance, since $m^{a,b}$ is always a monotonic transform, PS must fail even for i.i.d.\ data for some data-generating distributions (see \citet{gupta2020distribution} for a formal proof). %
Furthermore, we are interested in going beyond i.i.d.\ data to data with drifting/shifting distribution. This brings us to our first question,}\vspace{-0.1cm}
\begin{gather*}\vspace{-0.4cm}
    \revision{\text{(Q1) Can Platt Scaling (PS) be extended to handle}}\\[-0.1cm]
    \revision{\text{shifting or drifting data distributions?}}
\end{gather*}
    \revision{A separate view of calibration that pre-dates the ML post-hoc calibration literature is the online adversarial calibration framework \citep{degroot1981assessing,foster1998asymptotic}. Through the latter work, we know that calibration can be achieved for arbitrary $y_t$ sequences without relying on a pretrained model $f$ or doing any other modeling over available features. This is achieved by hedging or randomizing over multiple probabilities, so that ``the past track record can essentially only improve, no matter the future outcome" (paraphrased from \citet{foster2021forecast}). 
    For interesting classification problems, however, the $y_t$ sequence is far from adversarial and informative covariates $\x_t$ are available. In such settings, covariate-agnostic algorithms achieve calibration by predicting something akin to an average $\sum_{s=1}^{t}y_s/t$ at time $t+1$ (see Appendix~\ref{appsec:climatology-experiment}). Such a prediction, while calibrated, is arguably not useful.
A natural question is: \vspace{-0.1cm}}
\begin{gather*}\vspace{-0.3cm}
    \revision{\text{(Q2) Can informative covariates (features) be used}}\\[-0.1cm]
    \revision{\text{ to make online adversarial calibration practical?}}
\end{gather*}
\posticml{
We answer (Q1) by developing an online version of Platt scaling, and (Q2) by leveraging the recently developed framework of calibeating \citep{foster2022calibeating}. The method of calibeating, illustrated in Figure~\ref{fig:intro-calibeating}, is to perform certain \textit{corrections} on top of pre-existing \textit{expert} forecasts to improve their calibration. A key calibeating idea that we use was already discovered by \citet{kuleshov2017estimating} to resolve (Q2) in a manner similar to ours. Namely, they first proposed the idea of binning and hedging on top of an expert, as we do in HOPS (Section~\ref{sec:hops}). We return to a more detailed comparison between our work and Kuleshov and Ermon's in Section~\ref{sec:hops}. To reiterate, while we repeatedly use the term ``calibeating" coined by Foster and Hart, the main idea in resolving (Q2) can equally be credited to Kuleshov and Ermon. }

\posticml{Unlike previous papers, the online expert is not a black-box but a centerpiece of our work. In the forthcoming proposal, we describe an end-to-end pipeline, where first, a covariate-based and time-adaptive expert is constructed using post-hoc calibration (OPS), and then it is calibeaten to achieve adversarial calibration (TOPS, HOPS).}

\begin{figure}[!h]
    \centering
    \vspace{-0.2cm}
    \includegraphics[trim=6cm 16cm 12cm 10.5cm, clip, width=\columnwidth]{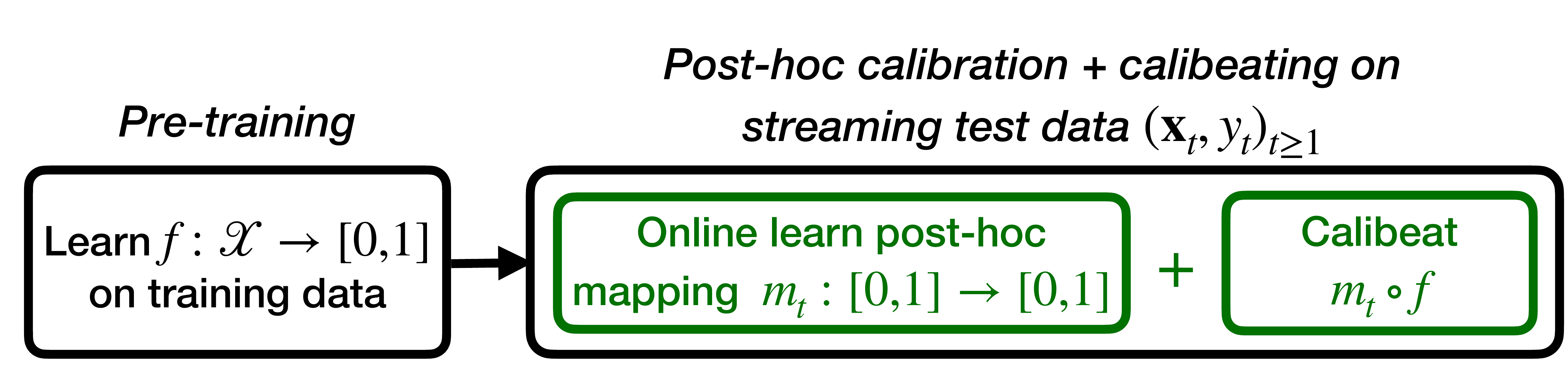}
    \caption{Online adversarial post-hoc calibration. }
    \label{fig:methology-summary}
    \vspace{-0.4cm}
\end{figure}

\begin{figure*}[htp]
    \centering
    \begin{subfigure}[t]{0.58\textwidth}
        \centering
        \vskip 0pt
        \includegraphics[trim=1cm 0 0 2cm, clip, width=1.1\textwidth]{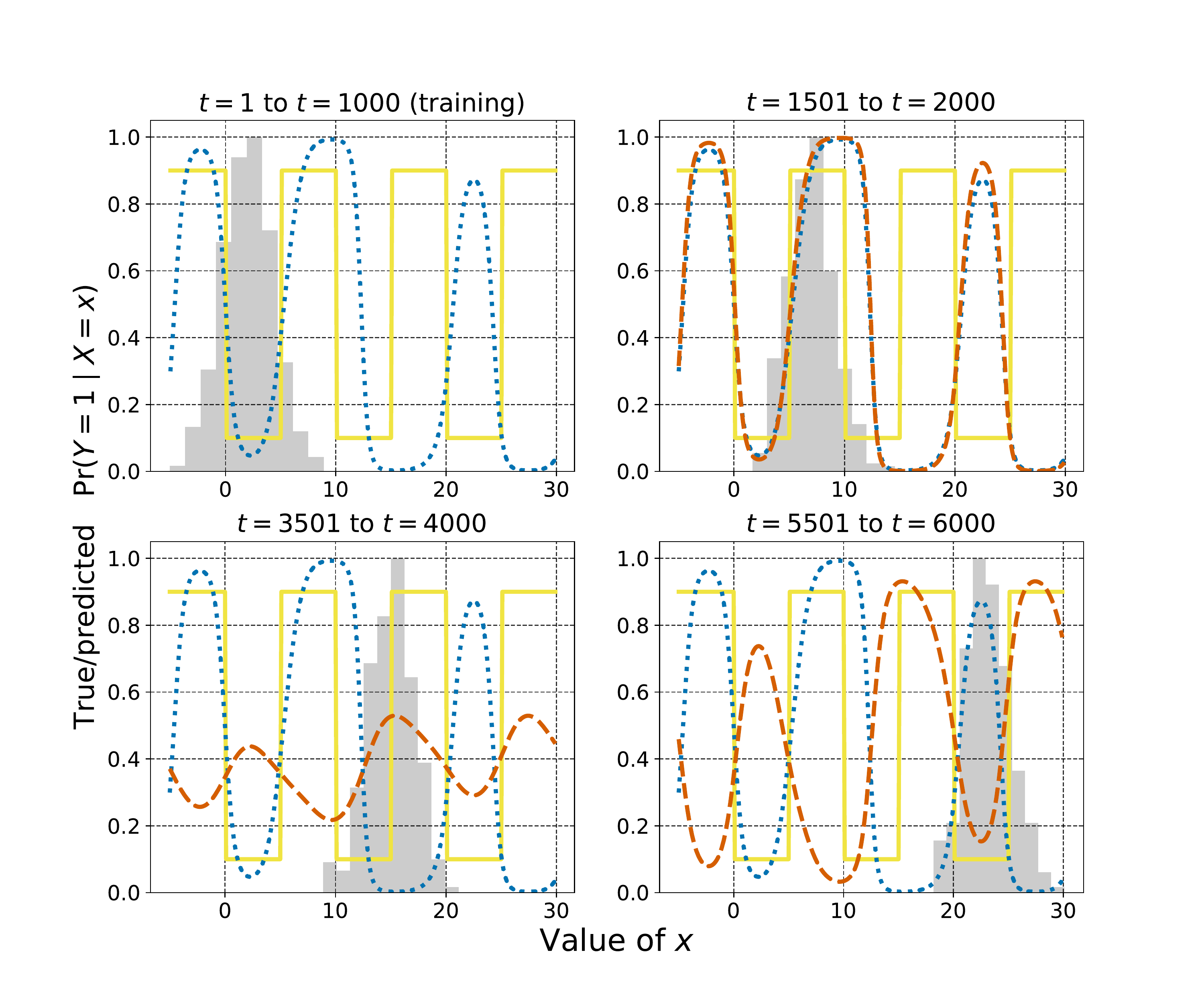}
    \end{subfigure}
    \begin{subfigure}[t]{0.41\textwidth}
        \vskip 0.3cm
        \centering\includegraphics[width=0.85\textwidth, trim=0 -2cm 0 -1cm, clip]{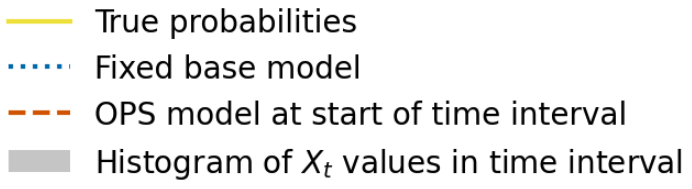}
        \begin{tabular}{cccc}
        \toprule
        $t$ & Model & Acc $\uparrow$ & CE $\downarrow$ \\\midrule
        1---1000 & Base & 88.00\% & 0.1  \\\midrule
         \multirow{2}{*}{\parbox{1cm}{1501---2000}} & Base & 81.20\% & 0.18 \\
          & OPS & 81.68\% & 0.19 \\
         \midrule
         \multirow{2}{*}{\parbox{1cm}{3501---4000}} & Base & 43.12\% & 0.55\\
          & OPS &  62.16\% & 0.36\\
         \midrule
         \multirow{2}{*}{\parbox{1cm}{5501---6000}} & Base & 23.60\% & 0.64  \\
          & OPS & 87.76\% & 0.13 \\\bottomrule
        \end{tabular}
        \caption{Accuracy and calibration error (CE) values of the base model and OPS for indicated values of $t$.}
        \label{fig:cov-shift-illustration-table}
    \end{subfigure}
    \vspace{-0.5cm}
    \caption{The adaptive behavior of Online Platt scaling (OPS) for the covariate drift dataset described in Section~\ref{sec:cov-shift-illustrative}. The title of each panel indicates the time-window that panel corresponds to. The histogram of $X_t$ values in the corresponding time window is plotted with maximum height normalized to $1$. Also plotted is the true curve for $\text{Pr}(Y = 1 \mid X = x)$ and two predictive curves: a base model trained on $t=1$ to $t=1000$, and OPS-calibrated models with parameter values fixed at the start of the time window. %
    The base model is accurate for the training data which is mostly in $[-5 , 10]$, but becomes inaccurate and miscalibrated with the covariate-shifted values for larger $t$ (bottom two subplots). OPS adapts well, agreeing with the base model in the top-right subplot, but flipping the base model predictions in the bottom-right subplot.}%
    \vspace{-0.3cm}
    \label{fig:cov-shift-illustration}
\end{figure*}

\subsection{Online adversarial post-hoc calibration}
\revision{
The proposal, summarized in Figure~\ref{fig:methology-summary}, is as follows. First, train any probabilistic classifier $f$ on some part of the data. Then, perform \textit{online post-hoc calibration} on top of $f$ to get online adaptivity. In effect, this amounts to viewing $f(\x_t)$ as a scalar ``summary" of $\x_t$, and the post-hoc mapping $(m_t: [0,1]\to [0,1])_{t\geq 1}$ becomes the time-series model over the scalar feature $f(\x_t)$. Finally, apply calibeating on the post-hoc predictions $m_t(f(\x_t))$ to obtain adversarial validity. Figure~\ref{fig:methology-summary} highlights our choice to do both post-hoc calibration and calibeating simultaneously on the streaming test data $(\x_t, y_t)_{t\geq 1}$.}%

Such an online version of post-hoc calibration has not been previously studied to the best of our knowledge. We show how one would make PS online, to obtain Online Platt Scaling (OPS). OPS relies on a simple but crucial observation: PS is a two-dimensional logistic regression problem over ``pseudo-features" $\logit(f(\x_t))$. Thus the problem of learning OPS parameters is the problem of online logistic regression (OLR, see Figure~\ref{fig:intro-olr} for a brief description). %
Several regret minimization algorithms have been developed for OLR \citep{hazan2007logarithmic, foster2018logistic, jezequel2020efficient}. We consider these and find an algorithm with optimal regret guarantees that runs in linear time. %
These regret guarantees imply that OPS is guaranteed to perform as well as the best fixed PS model in hindsight for an arbitrarily distributed online stream $(\x_t, y_t)_{t\geq 1}$, which includes the entire range of distribution drifts---i.i.d.\ data, data with covariate/label drift, and adversarial data.  We next present illustrative experiments where this theory bears out impressively in practice.

Then, Section~\ref{sec:ops} presents OPS, Section~\ref{sec:calibeating} discusses calibeating, Section~\ref{sec:experiments} presents baseline experiments on synthetic and real-world datasets. Section~\ref{sec:beta-scaling} discusses the extension of all OPS ideas to a post-hoc technique called beta scaling.

\begin{figure*}[t]
    \centering
    \begin{subfigure}{\linewidth}
    \centering
    \includegraphics[trim=0 0 0 0, clip, width=0.6\linewidth]{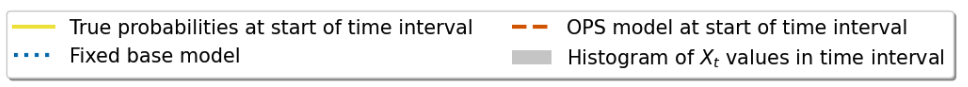}
    \end{subfigure}
    \begin{subfigure}[t]{0.49\textwidth}
        \centering
        \includegraphics[width=1\textwidth, trim=0 30 0 1.5cm, clip]{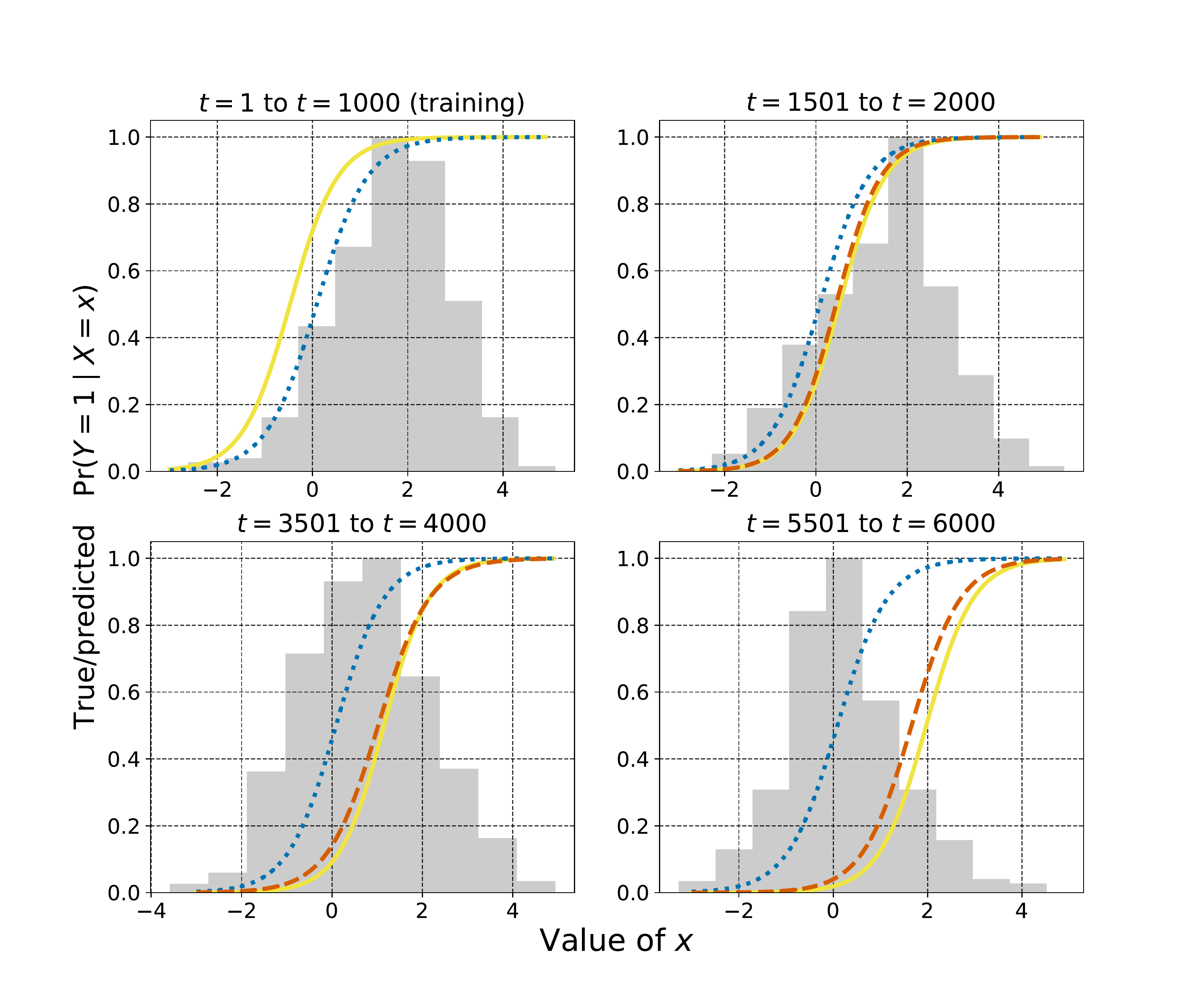}
        \vspace{-0.5cm}
        \caption{OPS with label drift.}
        \label{fig:label-shift-small}
    \end{subfigure}
    \begin{subfigure}[t]{0.49\textwidth}
        \centering
        \includegraphics[width=1\textwidth, trim=0 30 0 1.5cm, clip]{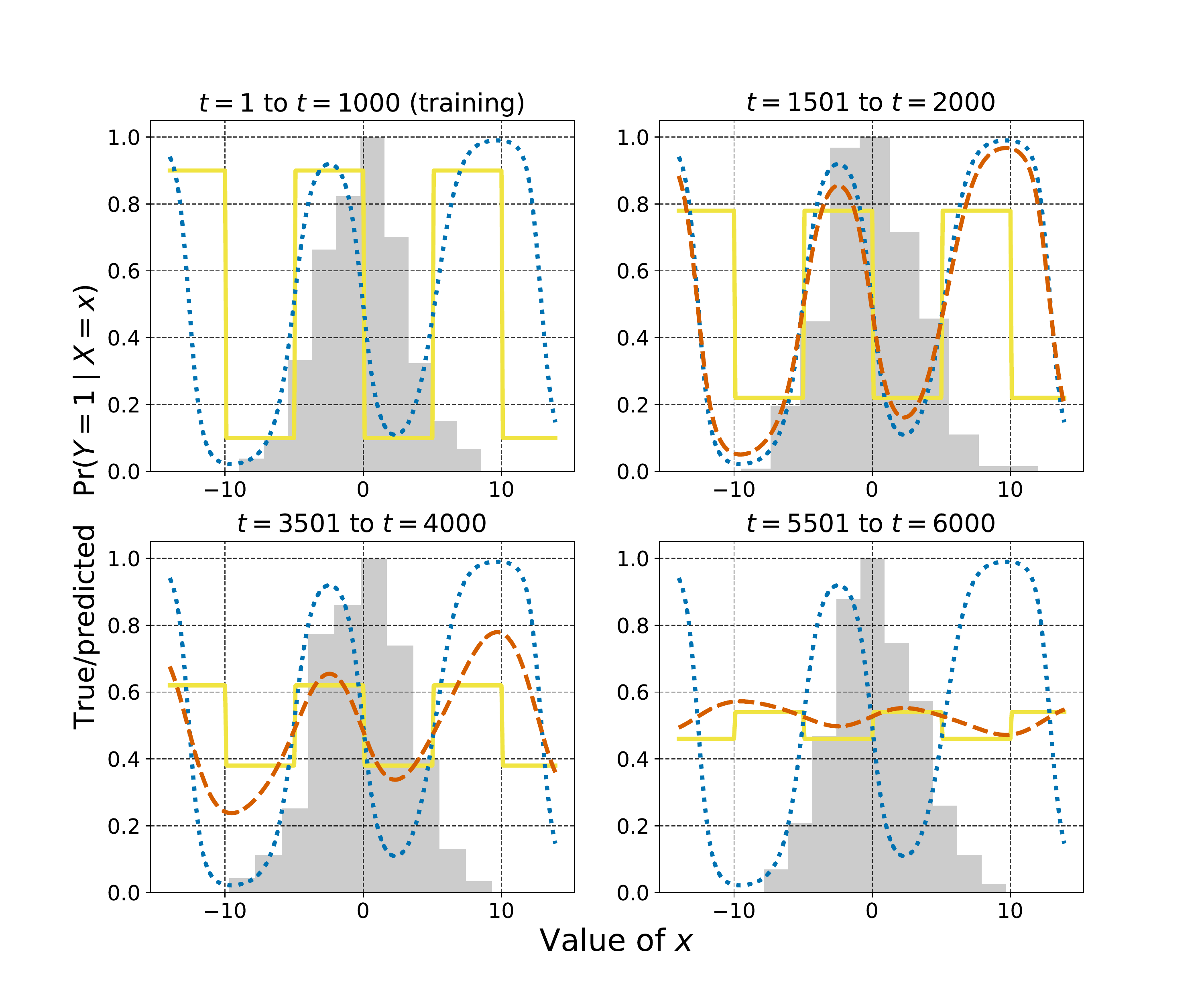}
        \vspace{-0.5cm}
        \caption{OPS with regression-function drift.}
        \label{fig:regression-function-shift-small}
    \end{subfigure}
    \caption{The adaptive behavior of OPS for the simulated label shift and regression-function drift datasets described in Section~\ref{sec:cov-shift-illustrative}. For more details on the contents of the figure, please refer to Figure~\ref{fig:cov-shift-illustration}. %
    The improvement in calibration and accuracy of OPS over the base model is visually apparent, but for completeness, \{Acc, CE\} values are reported in the Appendix as part of Figures~\ref{fig:label-shift-illustration} and \ref{fig:regression-shift-illustration}.}%
    \label{fig:regression-shift-cov-shift-joint-illustration}
    \vspace{-0.2cm}
\end{figure*}

\subsection{Illustrative experiments with distribution drift}
\label{sec:cov-shift-illustrative}
\textbf{Covariate drift. }%
We generated data as follows. For $t = 1, 2, \ldots, 6000$,
\begin{equation} \label{eq:cov-shift-1d-sim}\begin{split}
    X_t &\sim \Ncal((t-1)/250, 4);\\
    Y_t | X_t &\sim \left\{\begin{array}{ll}\text{Ber}(0.1) & \text{if } \text{mod}{(\floor{X_t/5}, 2)} = 0,\\
    \text{Ber}(0.9) & \text{if }  \text{mod}{(\floor{X_t/5}, 2)} = 1.
    \end{array}\right.
\end{split}
\end{equation}
Thus the distribution of $Y_t$ given $X_t$ is a fixed periodic function, but the distribution of $X_t$ drifts over time. The solid yellow line in Figure~\ref{fig:cov-shift-illustration} plots $\smash{\text{Pr}(Y = 1 \mid X = x)}$ against $x$. We featurized $x$ as a 48-dimensional vector with the components 
$\sin\roundbrack{\frac{x}{\text{freq}} + \text{translation}}$, where
$\smash{\text{freq} \in \{1, 2, 3, 4, 5, 6\}}$ and $\smash{\text{translation} \in \{0, \pi/4, \pi/2, \ldots 7\pi/4\}}$. 

A logistic regression base model $f$ is trained over this 48-dimensional representation using the points $(X_t, Y_t)_{t=1}^{1000}$, randomly permuted and treated as a single batch of exchangeable points, which we will call \textit{training points}. The points $(X_t, Y_t)_{t=1001}^{6000}$ form a supervised non-exchangeable test stream: we use this stream to evaluate $f$, recalibrate $f$ using OPS, and evaluate the OPS-calibrated model. 

Figure~\ref{fig:cov-shift-illustration} displays $f$ and the recalibrated OPS models at four ranges of $t$ (one per plot). The training data has most $x_t$-values in the range $[-5, 10]$ as shown by the (height-normalized) histogram in the top-left plot. In this regime, $f$ is visually accurate and calibrated---the dotted light blue line is close to the solid yellow truth. We now make some observations at three test-time regimes of $t$: 
\begin{enumerate}[label=(\alph*), leftmargin=10pt, itemsep=0pt, topsep=0pt]
    \item $t = 1501$ to $t = 2000$ (the histogram shows the distribution of $(x_t)_{t=1501}^{2000}$). For these values of $t$, the test data is only slightly shifted from the training data, and $f$ continues to perform well. The OPS model recognizes the good performance of $f$ and does not modify it much.
    \item $t = 3500$ to $t = 4000$. Here, $f$ is ``out-of-phase" with the true distribution, and Platt scaling is insufficient to improve $f$ by a lot. OPS recognizes this, and it offers slightly better calibration and accuracy by making less confident predictions between $0.2$ and $0.4$. 
    \item $t = 5500$ to $t = 6000$. In this regime, $f$ makes predictions opposing reality. Here, the OPS model flips the prediction, achieving high accuracy and calibration. 
\end{enumerate}
These observations are quantitatively supported by the accuracy and $\ell_1$-calibration error (CE) values reported by the table in Figure~\ref{fig:cov-shift-illustration-table}. Accuracy and CE values are estimated using the known true distribution of $Y_t \mid X_t$ and the observed $X_t$ values, making them unbiased and avoiding some well-known issues with CE estimation. More details are provided in Appendix~\ref{appsec:intro-experiments-additional}.

\textbf{Label drift.} For $t = 1, 2, \ldots, 6000$, data is generated as: 
\vspace{-0.1cm}
\begin{equation} \label{eq:label-shift-1d-sim}\begin{split}
    Y_t &\sim \text{Bernoulli}(0.95(1-\alpha_t) + 0.05 \alpha_t),\\
    & \text{ where } \alpha_t = (t-1)/6000);\\
    X_t|Y_t &\sim \indicator{Y_t = 0}\Ncal(0, 1) + \indicator{Y_t = 1}\Ncal(2, 1).
\end{split}
\vspace{-0.1cm}
\end{equation}
Thus, $X_t \mid Y_t$ is fixed while the label distribution drifts. 
We follow the same training and test splits described in the covariate drift experiment, but without sinusoidal featurization of $X_t$; the base logistic regression model is trained directly on the scalar $X_t$'s. The gap between $f$ and the true model increases over time but OPS adapts well (Figure~\ref{fig:label-shift-small}). %

\textbf{Regression-function drift.} For $t=1, 2, \ldots, 6000$, the data is generated as follows: $\alpha_t = (t-1)/5000$,
\begin{align} \label{eq:reg-shift-1d-sim}
    X_t &\sim \Ncal(0, 10) \text{ and } Y_t|X_t \sim \text{Bernoulli}(p_t), \text{ where }\\
    p_t &= \left\{\begin{array}{ll} 0.1(1-\alpha_t) + 0.5\alpha_t & \text{if } \text{mod}{(\floor{X_t/5}, 2)} = 0,\\
    0.9(1-\alpha_t) + 0.5\alpha_t & \text{if } \text{mod}{(\floor{X_t/5}, 2)} = 1.
    \end{array}\right. \nonumber %
\end{align}
Thus the distribution of $X_t$ is fixed, but the regression function $\text{Pr}(Y_t = 1 \mid X_t)$ drifts over time. 
We follow the same training and test splits described in the covariate drift experiment, as well as the 48-dimensional featurization and logistic regression modeling. The performance of the base model worsens over time, while OPS adapts (Figure~\ref{fig:regression-function-shift-small}). %

\vspace{-0.1cm}
\section{Online Platt scaling (OPS)}
\label{sec:ops}
In a batch post-hoc setting, the Platt scaling parameters are set to those that minimize log-loss over the calibration data. If we view the first $t$ instances in our stream as the calibration data, the fixed-batch Platt scaling parameters are,
\vspace{-0.3cm}
\begin{equation}
% \vspace{-0.1cm}
    (\widehat{a}_t, \widehat{b}_t) = \argmin_{(a, b) \in \Real^2} \sum_{s = 1}^t  l(m^{a,b}(f(\x_s)), y_s),  \label{eq:ops-optimal-t}
\end{equation}
where $l(p, y) = -y\log p-(1-y)\log(1-p)$ and $m^{a,b}$ is defined in \eqref{eq:ps-model-class}. Observe that this is exactly logistic regression over the dataset $(\logit(f(\x_s)), y_s)_{s=1}^t$.

\begin{table*}[htp]
    \centering
    \begin{tabular}{ccc} \toprule
     Algorithm & Regret & Running time 
    \\\midrule
    Online Gradient Descent (OGD) \citep{zinkevich2003online}  & $B\sqrt{T}$ & $T$
    \\
    Online Newton Step (ONS) \citep{hazan2007logarithmic} & $e^B\log T$ & $T$
    \\
    AIOLI \citep{jezequel2020efficient} & $B \log (BT)$ & $T\log T$
    \\
    Aggregating Algorithm (AA) \citep{vovk1990aggregating, foster2018logistic} & $\log (BT)$ & $ B^{18}T^{24}$
    \\
    \bottomrule
    \end{tabular}
    \caption{Asymptotic regret and running times of online logistic regression (OLR) algorithms for OPS as functions of the radius of reference class $B$ and time-horizon $T$. For general OLR, regret and running times also depend on the dimension of $\Xcal$. However, OPS effectively reduces the dimensionality of $\Xcal$ to $2$, so that a second-order method like ONS runs almost as fast as a first-order method like OGD. Also note that $B = \sqrt{a^2 + b^2}$ is small if the base model $f$ is not highly miscalibrated. ONS with fixed hyperparameters was chosen for all OPS experiments; see Section~\ref{sec:ops-ons-final} for implementation details.}
    \label{tab:lr-regret-runningtime}
    \vspace{-0.4cm}
\end{table*}

The thesis of OPS is that as more data is observed over time, we should use it to update the Platt scaling parameters. Define  $p_t^\ops := m^{a_t, b_t}(f(\x_t))$, where $(a_t, b_t)$ depends on $\{(f(\x_1), y_1), \ldots, (f(\x_{t-1}), y_{t-1})\}$.\footnote{A variant of this setup allows $(a_t, b_t)$ to depend on $f(\x_t)$ \citep{foster2018logistic}.} One way to compare methods in this online setting is to consider \textit{regret} $R_T$ with respect to a reference $\ell_2$-ball of radius $B$, $\Bcal := \{(a, b) \in \Real^2: a^2 + b^2 \leq B^2\}$:
\vspace{-0.1cm}
\begin{equation}
    R_T = \sum_{t=1}^T l(p_t^\ops, y_t) - \min_{(a, b) \in \Bcal}\sum_{t=1}^T l(m^{a,b}(f(\x_t)), y_t).\label{eq:regret-definition}
\end{equation} 
$R_T$ is the difference between the total loss incurred when playing $(a_t, b_t)$ at times $t\leq T$ and the total loss incurred when playing the single optimal $(a, b) \in \Bcal$ for all $t \leq T$. Typically, we are interested in algorithms that have low $R_T$ irrespective of how $(\x_t, y_t)$ is generated.
\vspace{-0.1cm}
\subsection{Logarithmic worst-case regret bound for OPS}
OPS regret minimization is exactly online logistic regression (OLR) regret minimization  over ``pseudo-features" $\logit(f(\x_t))$. Thus our OPS problem is immediately solved using OLR methods. A number of OLR methods have been proposed, and we consider their regret guarantees and running times for the OPS problem. These bounds typically depend on $T$ and two problem-dependent parameters: the dimension (say $d$) and  $B$, the radius of $\Bcal$. 
\begin{enumerate}[itemsep=-1pt, topsep=0pt]
    \item In our case, $d = 2$ since there is one feature $\logit(f(\x))$ and a bias term. Thus $d$ is a constant. %
    \item $B$ could technically be large, but in practice, if $f$ is not highly miscalibrated, we expect small values of $a$ and $b$ which would in turn lead to small $B$. This was true in all our experiments.%
\end{enumerate}
Regret bounds and running times for candidate OPS methods are presented in Table~\ref{tab:lr-regret-runningtime}, which is an adaptation of Table 1 of \citet{jezequel2020efficient} with all $\text{poly}(d)$ terms removed. Based on this table, we identify AIOLI and Online Newton Step (ONS) as the best candidates for implementing OPS, since they both have $O(\log T)$ regret and $\widetilde{O}(T)$ running time. In the following theorem, we collect explicit regret guarantees for OPS based on ONS and AIOLI. Since the log-loss can be unbounded if the predicted probability equals $0$ or $1$, we require some restriction on $f(\x_t)$.%
\begin{theorem}
\label{thm:ops-regret}
    Suppose $\forall t, f(\x_t) \in [0.01, 0.99]$, $B \geq 1$, and $T \geq 10$. Then, for any sequence $(\x_t, y_t)_{t=1}^T$, OPS based on ONS  satisfies\vspace{-0.2cm}
    \begin{equation}
        \vspace{-0.1cm}R_T(\text{ONS}) \leq 2(e^B + 10B) \log T + 1,  \label{eq:ons-regret}
    \end{equation}
    and OPS based on AIOLI  satisfies\vspace{-0.1cm}
    \begin{equation}\vspace{-0.2cm}
        R_T(\text{AIOLI}) \leq 22 B \log(BT).\label{eq:aioli-regret}
    \end{equation}
\end{theorem}\vspace{-0.05cm}
The proof is in Appendix~\ref{appsec:proofs}. Since log-loss is a proper loss \citep{gneiting2007strictly}, minimizing it has implications for calibration \citep{brocker2009reliability}. However, no ``absolute" calibration bounds can be shown for OPS, as discussed shortly in Section~\ref{sec:limitations-regret}. %

\vspace{-0.1cm}
\subsection{Hyperparameter-free ONS implementation}
\label{sec:ops-ons-final}
In our experiments, we found ONS to be significantly faster than AIOLI while also giving better calibration. Further, ONS worked without any hyperparameter tuning after an initial investigation was done to select a single set of hyperparameters. Thus we used ONS for experiments based on a \textit{verbatim implementation} of Algorithm 12 in \citet{hazan2016introduction}, with $\gamma = 0.1$, $\rho = 100$, and $\mathcal{K} = \{(a, b): \|(a, b)\|_2 \leq 100\}$. Algorithm~\ref{alg:ops-ons} in the Appendix contains pseudocode for our final OPS implementation.

\vspace{-0.1cm}
\subsection{Follow-The-Leader as a baseline for OPS}
\label{sec:ftl-wps}
The Follow-The-Leader (FTL) algorithm sets $(a_t, b_t) = (\widehat{a}_{t-1}, \widehat{b}_{t-1})$ (defined in \eqref{eq:ops-optimal-t}) for $t \geq 1$. This corresponds to solving a logistic regression optimization problem at every time step, making the overall complexity of FTL $\Omega(T^2)$. 
Further, FTL has $\Omega (T)$ worst-case regret. %
Since full FTL is intractably slow to implement even for an experimental comparison, we propose to use a computationally cheaper variant, called Windowed Platt Scaling (WPS). In WPS the optimal parameters given all current data, $(\widehat{a}_t, \widehat{b}_t)$, are computed and updated every $O(100)$ steps instead of at every time step. We call this a \textit{window} and the exact size of the window can be data-dependent. The optimal parameters computed at the start of the window are used to make predictions until the end of that window, then they are updated for the next window. This heuristic version of FTL performs well in practice (Section~\ref{sec:experiments}).%

\vspace{-0.1cm}
\subsection{Limitations of regret analysis}
\label{sec:limitations-regret}
Regret bounds are relative to the best in class, so Theorem~\ref{thm:ops-regret} implies that OPS will do no worse than the best Platt scaling model in hindsight. However, even for i.i.d. data, the best Platt scaling model is itself miscalibrated on some distributions \citep[Theorem 3]{gupta2020distribution}. This latter result shows that some form of binning must be deployed to be calibrated for arbitrarily distributed i.i.d. data. Further, if the data is adversarial, any deterministic predictor %
can be rendered highly miscalibrated \citep{oakes1985self, dawid1985comment}; a simple strategy is to set $y_t = \indicator{p_t \leq 0.5}$.  %
In a surprising seminal result, \citet{foster1998asymptotic} showed that adversarial calibration is possible by randomizing/hedging between different bins. %
The following section shows how one can perform such binning and hedging on top of OPS, based on a technique called calibeating. \vspace{-0.1cm}

\section{Calibeating the OPS forecaster}
\label{sec:calibeating}
Calibeating \citep{foster2022calibeating} is a technique to improve or ``beat" an expert forecaster. The idea is to first use the expert's forecasts to allocate data to representative \textit{bins}. Then, the bins are treated \emph{nominally}: they are just names or tags for ``groups of data-points that the expert suggests are similar". The final forecasts in the bins are computed using only the outcome ($y_t$) values of the points in the bin (seen so far), with no more dependence on the expert's original forecast. The intuition is that forecasting inside each bin can be done in a theoretically valid sense, irrespective of the theoretical properties of the expert. 

We will use the following ``$\epsilon$-bins" to perform calibeating:
\vspace{-0.1cm}
\begin{equation}\vspace{-0.1cm}
    B_1 = [0, \epsilon), B_2 = [\epsilon, 2\epsilon), \ldots, B_m = [1-\epsilon, 1].
 \label{eq:B-bins}
\end{equation}
Here $\epsilon > 0$ is the width of the bins, and for simplicity we assume that $m = 1/\epsilon$ is an integer. For instance, one could set $\epsilon = 0.1$ or the number of bins $m=10$, as we do in the experiments in Section~\ref{sec:experiments}. Two types of calibeating---tracking and hedging---are described in the following subsections. We suggest recalling our illustration of calibeating in the introduction (Figure~\ref{fig:intro-calibeating}). %

\vspace{-0.2cm}
\subsection{Calibeating via tracking past outcomes in bins}
\label{sec:tops}
Say at some $t$, the expert forecasts $p_t \in [0.7, 0.8)$. We look at the instances $s < t$ when $p_s  \in [0.7, 0.8)$ and compute\vspace{-0.1cm}
\[\vspace{-0.1cm}
\Bar{y}^b_{t-1} = \text{Average}\{y_s : s < t, p_s  \in [0.7, 0.8)\}.
\]
Suppose we find that $\Bar{y}^b_{t-1} = 0.85$. That is, when the expert forecasted bin $[0.7 , 0.8)$ in the past, the average outcome was $0.85$. A natural idea now is to forecast $0.85$ instead of $0.75$. We call this process ``Tracking", and it is the form of calibeating discussed in Section 4 of \citet{foster2022calibeating}. 
In our case, we treat OPS as the expert and call the tracking version of OPS as TOPS. If $p_t^\ops \in B_b$%
, then\vspace{-0.1cm}
\begin{equation}\vspace{-0.1cm}
     p_t^\tops := \text{Average}\{y_s : s < t,  p^{\ops}_s \in B_b\}.\label{eq:tops-update}
\end{equation}
The average is defined as the mid-point of $B_b$ if the set above is empty.

\citet{foster2022calibeating} showed that the Brier-score of the TOPS forecasts $p_t^\tops$, defined as $\frac{1}{T}\sum_{t=1}^T (y_t - p_t^\tops)^2$, is better than the corresponding Brier-score of the OPS forecasts $p_t^\ops$, by roughly the squared calibration error of $p_t^\ops$ (minus a $\log T$ term). In the forthcoming Theorem~\ref{thm:tops-sharpness-guarantee}, we derive a result for a different object that is often of interest in post-hoc calibration, called sharpness.%
\vspace{-0.1cm}
\subsection{Segue: defining sharpness of forecasters}\vspace{-0.1cm}
\label{subsec:sharpness-defn}
Recall the $\epsilon$-bins introduced earlier \eqref{eq:B-bins}. Define $N_b = \abs{\{t \leq T : p_t \in B_b\}}$ and $\widehat{y}_b = \frac{1}{N_b}\sum_{t\leq T, p_t \in B_b} y_t$ if $N_b > 0$, else $\widehat{y}_b = 0$. Sharpness is defined as,\vspace{-0.3cm}
\begin{equation}
\vspace{-0.1cm}
    \shp(p_{1:T}) := \frac{1}{T}\sum_{b=1}^m N_b \cdot \widehat{y}_b^2.\footnote{The original definition of sharpness \citep{murphy1973new} was (essentially): $-T^{-1}\sum_{b=1}^m N_b \widehat{y}_b(1-\widehat{y}_b)$, which equals $\shp(p_{1:T}) - \Bar{y}_T$. We add the forecast-independent term $\Bar{y}_T$ on both sides and define the (now non-negative) quantity as~\shp.}\label{eq:sharpnes-definition}
\end{equation}
If the forecaster is perfectly knowledgeable and forecasts $p_t = y_t$, its \shp~equals $\sum_{t=1}^T y_t/T =: \Bar{y}_T$. On the other hand, if the forecaster puts all points into a single bin $b$, its \shp~equals $(\sum_{t=1}^T y_t/T)^2 = \Bar{y}_T^2$. The former forecaster is precise or \textit{sharp}, while the latter is not, and \shp~captures this---it can be shown that %
$\Bar{y}_T^2 \leq \shp(p_{1:T}) \leq \Bar{y}_T$.
We point the reader to \citet{brocker2009reliability} for further background. One of the goals of effective forecasting is to ensure high sharpness~\citep{gneiting2007probabilistic}. OPS achieves this goal by relying on the log-loss, a proper loss. The following theorem shows that TOPS suffers a small loss in \shp~compared to \ops. %
\begin{theorem}
The sharpness of TOPS forecasts satisfies%
\vspace{-0.1cm}
\begin{equation}\vspace{-0.1cm}
    \shp(p_{1:T}^\tops) \geq \shp(p_{1:T}^{\ops}) - \epsilon - \frac{\epsilon^2}{4} - \frac{\log T+1}{\epsilon T}.
\end{equation}
\label{thm:tops-sharpness-guarantee}\vspace{-0.5cm}
\end{theorem}
The proof (in Appendix~\ref{appsec:proofs}) uses Theorem 3 of \citet{foster2022calibeating} and relationships between sharpness, Brier-score, and a quantity called refinement. If $T$ is fixed and known, setting $\epsilon \approx \sqrt{\log T / T}$ (including constant factors), or equivalently, the number of bins $B \approx \sqrt{T/\log T}$ gives a rate of $\widetilde{O}(\sqrt{1/T})$ for the \shp~difference term. While we do not show a calibration guarantee, TOPS had the best calibration performance in most experiments (Section~\ref{sec:experiments})

\vspace{-0.1cm}
\subsection{Calibeating via hedging or randomized prediction}
\label{sec:hops}
All forecasters introduced so far---the base model $f$, OPS, and TOPS---make forecasts $p_t$ that are deterministic given the past data until time $t-1$. If the $y_t$ sequence is being generated by an adversary that acts after seeing $p_t$, then the adversary can ensure that each of these forecasters is miscalibrated by setting $y_t = \indicator{p_t \leq 0.5}$.

Suppose instead that the forecaster is allowed to hedge---randomize and draw the forecast from a distribution instead of fixing it to a single value---and the adversary only has access to the distribution and not the actual $p_t$. Then there exist hedging strategies that allow the forecaster to be arbitrarily well-calibrated \citep{foster1998asymptotic}. In fact, \citet[henceforth F99]{foster1999proof} showed that this can be done while hedging between two arbitrarily close points in $[0,1]$.

In practice, outcomes are not adversarial, and covariates are available. %
A hedging algorithm that does not use covariates cannot be expected to give informative predictions. We verify this intuition through an experiment in Appendix on historical rain data~\ref{appsec:climatology-experiment}---F99's hedging algorithm simply predicts the average $y_t$ value in the long run. 

A best-of-both-worlds result can be achieved by using the expert forecaster to bin data based on $\x_t$ values, just like we did in Section~\ref{sec:tops}. Then, inside every bin, a separate hedging algorithm is instantiated. For the OPS predictor, this leads to HOPS (OPS + hedging). Specifically, in our experiments and the upcoming calibration error guarantee, we used F99:\vspace{-0.1cm}
\begin{equation}
     p_t^\hops := \text{F99}(y_s : s < t,  p_s \in B_b).\label{eq:hops-update}
\end{equation}
A standalone description of F99 is included in Appendix~\ref{appsec:f99}. F99 hedges between consecutive mid-points of the $\epsilon$-bins defined earlier \eqref{eq:B-bins}. The only hyperparameter for F99 is $\epsilon$. In the experiments in the main paper, we set %
$\epsilon=0.1$. 
To be clear, $p_t$ is binned on the $\epsilon$-bins, and the hedging inside each bin is again over the $\epsilon$-bins. 

The upcoming theorem shows a \shp~lower bound on HOPS. In addition, we show an assumption-free upper bound on the ($\ell_1$-)calibration error, defined as \vspace{-0.1cm}
\begin{equation}\vspace{-0.2cm}
    \ce(p_{1:T}) := \frac{1}{T}\sum_{b=1}^m N_b \cdot \abs{\widehat{p}_b - \widehat{y}_b},\label{eq:ce-definition}
\end{equation}
where $N_b, \widehat{y}_b$ were defined in Section~\ref{subsec:sharpness-defn}, and $\widehat{p}_b = \frac{1}{N_b}\sum_{t\leq T, p_t \in B_b} p_t$, if $N_b > 0$, else $\widehat{p}_b = \text{mid-point}(B_b)$. 
Achieving small CE is one formalization of \eqref{eq:cal-informal}. 
The following result is conditional on the $y_{1:T}$, $p_{1:T}^\ops$ sequences. The expectation is over the randomization in F99. 
\begin{theorem}
For adversarially generated data, the expected sharpness of HOPS forecasts using the forecast hedging algorithm of \citet{foster1999proof} is lower bounded as \vspace{-0.1cm}%
\begin{equation}%
    \Exp{}{\shp(p_{1:T}^\hops)} \geq \shp(p_{1:T}^\ops) - \roundbrack{\epsilon%
    + \frac{\log T+1}{\epsilon^2T}},\label{eq:cops-sharpness-guarantee}
\end{equation}
and the expected calibration error of HOPS satisfies,
\begin{equation}\vspace{-0.1cm}
    \Exp{}{\ce(p_{1:T}^\hops)} \leq \epsilon/2 + 2\sqrt{1/\epsilon^2 T}. \label{eq:cops-calibration-guarantee}
\end{equation}
\label{thm:cops-calibration-guarantee}\vspace{-0.5cm}
\end{theorem}
\vspace{-0.2cm}
The proof  in Appendix~\ref{appsec:proofs} is based on Theorem 5 of \citet{foster2022calibeating}   and a \ce~bound for F99 based on Blackwell approachability \citep{blackwell1956analog}. With $\epsilon = \widetilde{\Theta}(T^{-1/3})$, the difference term in the \shp~bound is $\widetilde{O}(T^{-1/3})$ and with $\epsilon = \widetilde{\Theta}(T^{-1/4})$, the \ce~bound is  $\widetilde{O}(T^{-1/4})$. Compare \eqref{eq:cops-calibration-guarantee} to the usual (without calibeating) calibration bound of $O(\epsilon + 1/\sqrt{\epsilon T})$ which leads to $O(T^{-1/3})$ \citep{foster1998asymptotic}. High-probability versions of \eqref{eq:cops-calibration-guarantee} can be derived using probabilistic Blackwell approachability lemmas, such as those in \citet{perchet2014approachability}.  %

\posticml{The ``Online Recalibration" method of \citet[Algorithm 1]{kuleshov2017estimating} amounts to performing the same binning and hedging that we have described, but on top of a black-box expert. We used a specific expert, OPS, and experimentally demonstrate its benefits on multiple datasets (Section~\ref{sec:cov-shift-illustrative} and \ref{sec:experiments}). Theoretically, our calibration bound \eqref{eq:cops-calibration-guarantee} is identical to their Lemma 3 (if Lemma 3 is instantiated with F99). Their Lemma 2 shows a bound on the expected increase of any bounded proper loss on performing the calibeating step. For the case of Brier-loss their bound is $O(\epsilon + 1/\epsilon^2\sqrt{T})$. Our proof of \eqref{eq:cops-sharpness-guarantee} can be used to show an improved bound of $O(\epsilon + \log T/\epsilon^2 T)$, as stated formally in Appendix~\ref{appsec:proofs} (Theorem~\ref{thm:cops-brier-score-guarantee}). 
}

\vspace{-0.1cm}
\section{Experiments}
\label{sec:experiments}
We perform experiments with synthetic and real-data, in i.i.d. and distribution drift setting. Code to reproduce the experiments can be found at \url{https://github.com/aigen/df-posthoc-calibration} (see Appendix~\ref{appsec:implementation} for more details).  All baseline and proposed methods are described in Collection 1 on the following page.  In each experiment, the \textbf{base model} $f$ was a random forest (\texttt{sklearn}'s implementation). All default parameters were used, except \texttt{n\_estimators} was set to 1000. No hyperparameter tuning on individual datasets was performed for any of the recalibration methods.

\begin{algorithm}[t]
\begin{algorithmic}
	\STATE {\bfseries Input:} $f : \Xcal \to [0,1]$, any pre-learnt model
	\STATE {\bfseries Input:} $(\x_1, y_1), (\x_2, y_2), \ldots, (\x_T, y_T)  \in \Xcal \times \{0,1\}$
    \STATE {\bfseries Input:} calibration-set-size $T_\text{cal} < T$, window-size $W$
	\STATE Fixed Platt scaling: $(a^{\text{FPS}}, b^{\text{FPS}}) \gets (\widehat{a}_{T_\text{cal}}, \widehat{b}_{T_\text{cal}})$ (eq. \ref{eq:ops-optimal-t})\;
    \STATE Windowed Platt scaling: $(a^{\text{WPS}}, b^{\text{WPS}}) \gets (a^{\text{FPS}}, b^{\text{FPS}})$\;
    \STATE Online Platt scaling: $(a^{\text{OPS}}_1, b^{\text{OPS}}_1) \gets (1, 0)$\;	
    \FOR{$t=2$ {\bfseries to} $T$}
    \STATE $(a_t^\ops, b_t^\ops) \gets \text{ONS}((\x_1, y_1),\ldots,(\x_{t-1}, y_{t-1}))$
     \STATE (ONS is Algorithm \ref{alg:ops-ons} in the Appendix)
    \ENDFOR
    \FOR{$t=T_\text{cal} + 1$ {\bfseries to} $T$}
    \STATE $\textcolor{blue}{p_t^{\text{BM}}} \gets f(\x_t)$\;
    \STATE $\textcolor{blue}{p_t^{\text{FPS}}} \gets \sigmoid(a^\text{FPS}\cdot  \logit(f(\x_t)) + b^\text{FPS})$\;
    \STATE $\textcolor{blue}{p_t^{\text{WPS}}} \gets \sigmoid(a^\text{WPS}\cdot  \logit(f(\x_t)) + b^\text{WPS})$\;
    \STATE $\textcolor{blue}{p_t^{\text{OPS}}} \gets \sigmoid(a^\text{OPS}_{t}\cdot  \logit(f(\x_t)) + b^\text{OPS}_{t})$\;
    \STATE $\textcolor{blue}{p_t^{\text{TOPS}}}$ is set using past $(y_s, p_s^\ops)$ values as in \eqref{eq:tops-update}\;
    \STATE $\textcolor{blue}{p_t^{\text{HOPS}}}$ is set using past $(y_s, p_s^\ops)$ values as in \eqref{eq:hops-update}\;
    \STATE If$\smash{\mod(t - T_\text{cal}},W) = 0$, $(a^{\text{WPS}}, b^{\text{WPS}}) \gets (\widehat{a}_{t}, \widehat{b}_{t})$\;
    \ENDFOR
    \end{algorithmic}
 	\label{alg:online-post-hoc}
	\caption*{\textbf{Collection 1. }Proposed and baseline methods for online post-hoc calibration. Final forecasts are identified in \textcolor{blue}{blue}.} 
\end{algorithm}

\begin{figure*}[htp]
    \centering
    \begin{subfigure}{\linewidth}
    \centering
    \includegraphics[trim=0 10cm 0 0, clip, width=0.7\linewidth]{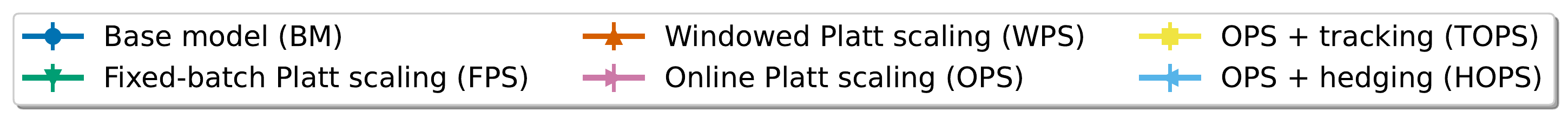}
    \end{subfigure}
        \centering
        \includegraphics[trim=0 0 0 0, clip, width=0.22\linewidth]{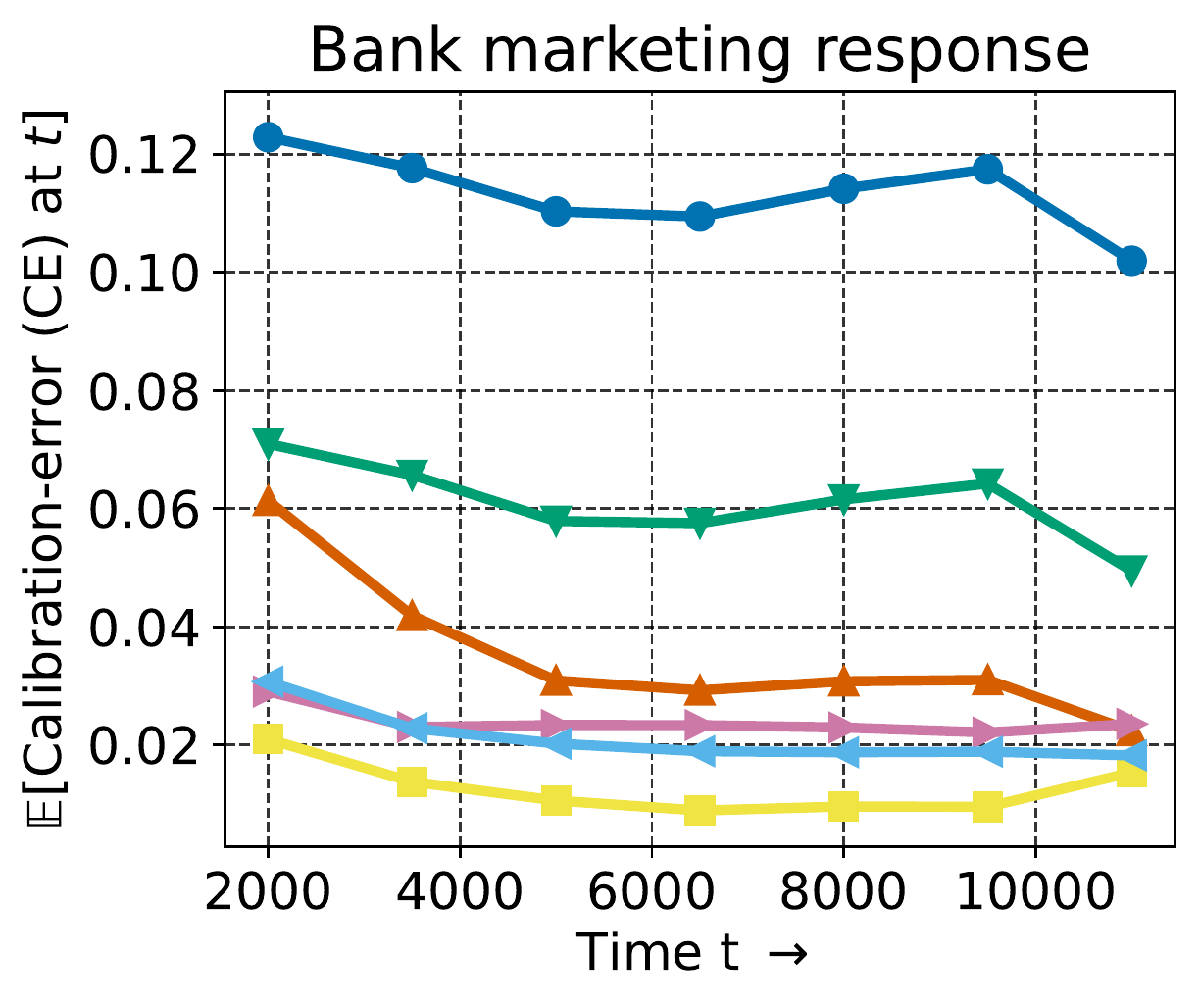}
        \includegraphics[trim=0 0 0 0, clip, width=0.22\linewidth]{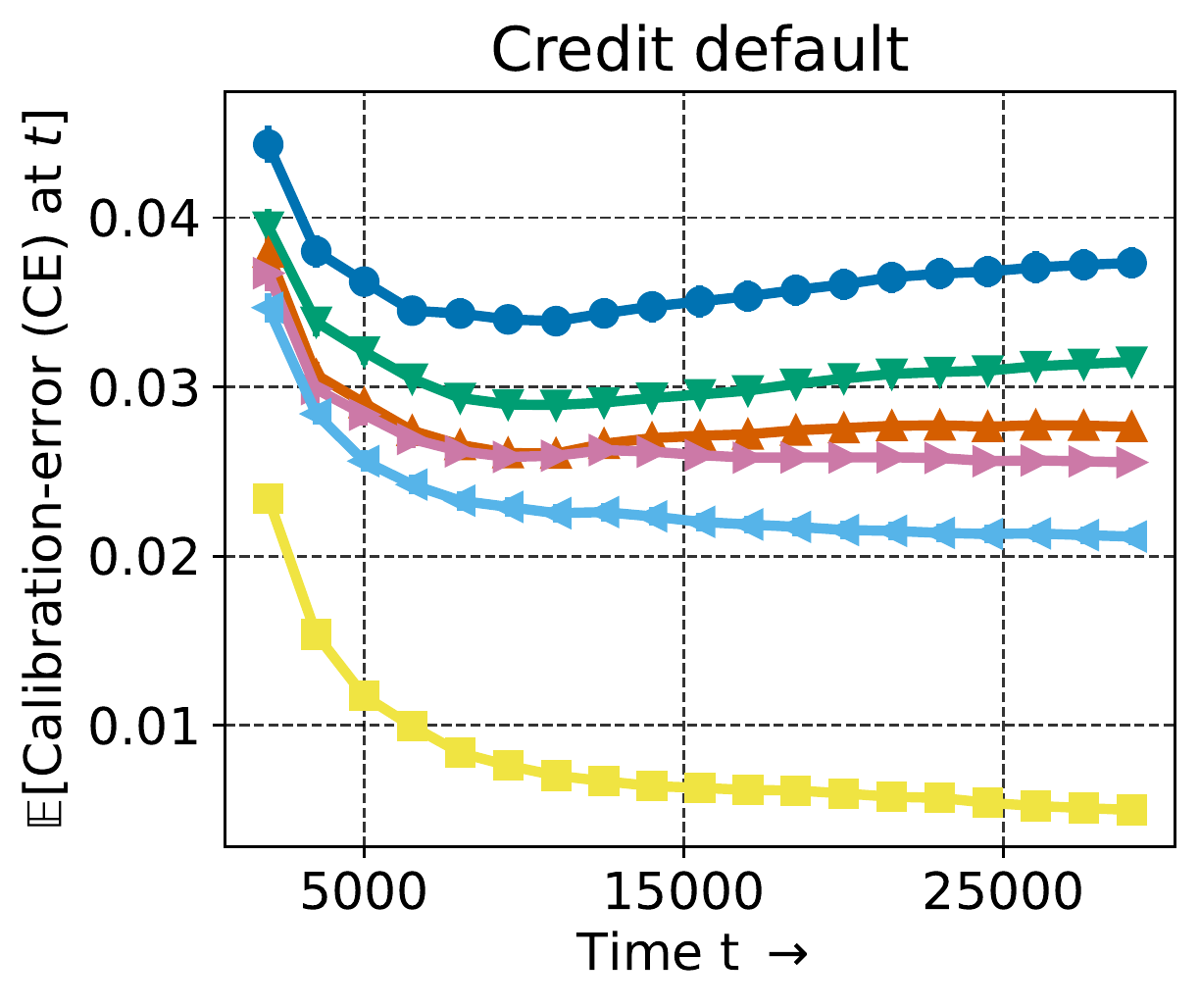}
        \includegraphics[trim=0 0 0 0, clip, width=0.22\linewidth]{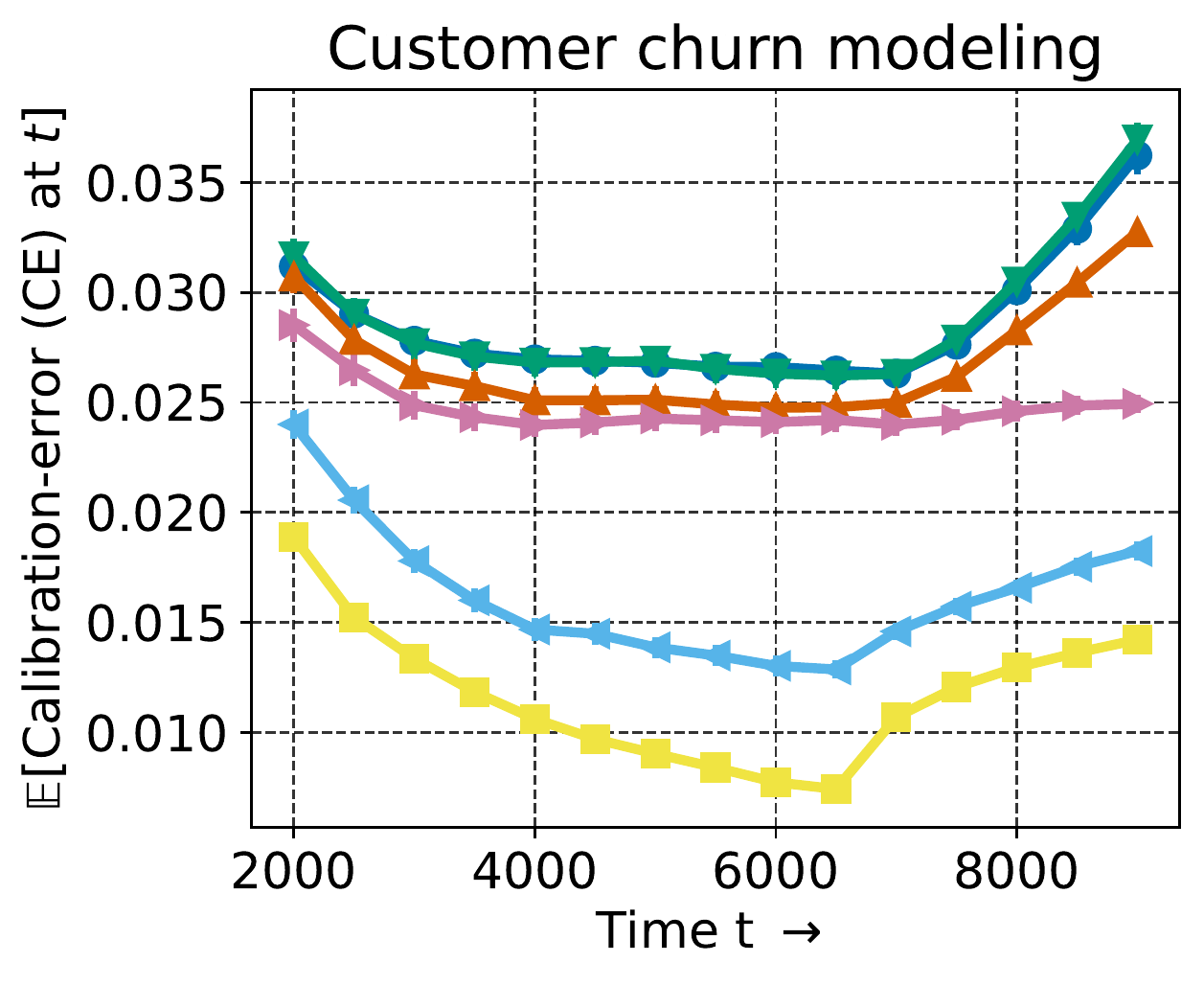}
        \includegraphics[trim=0 0 0 0, clip, width=0.22\linewidth]{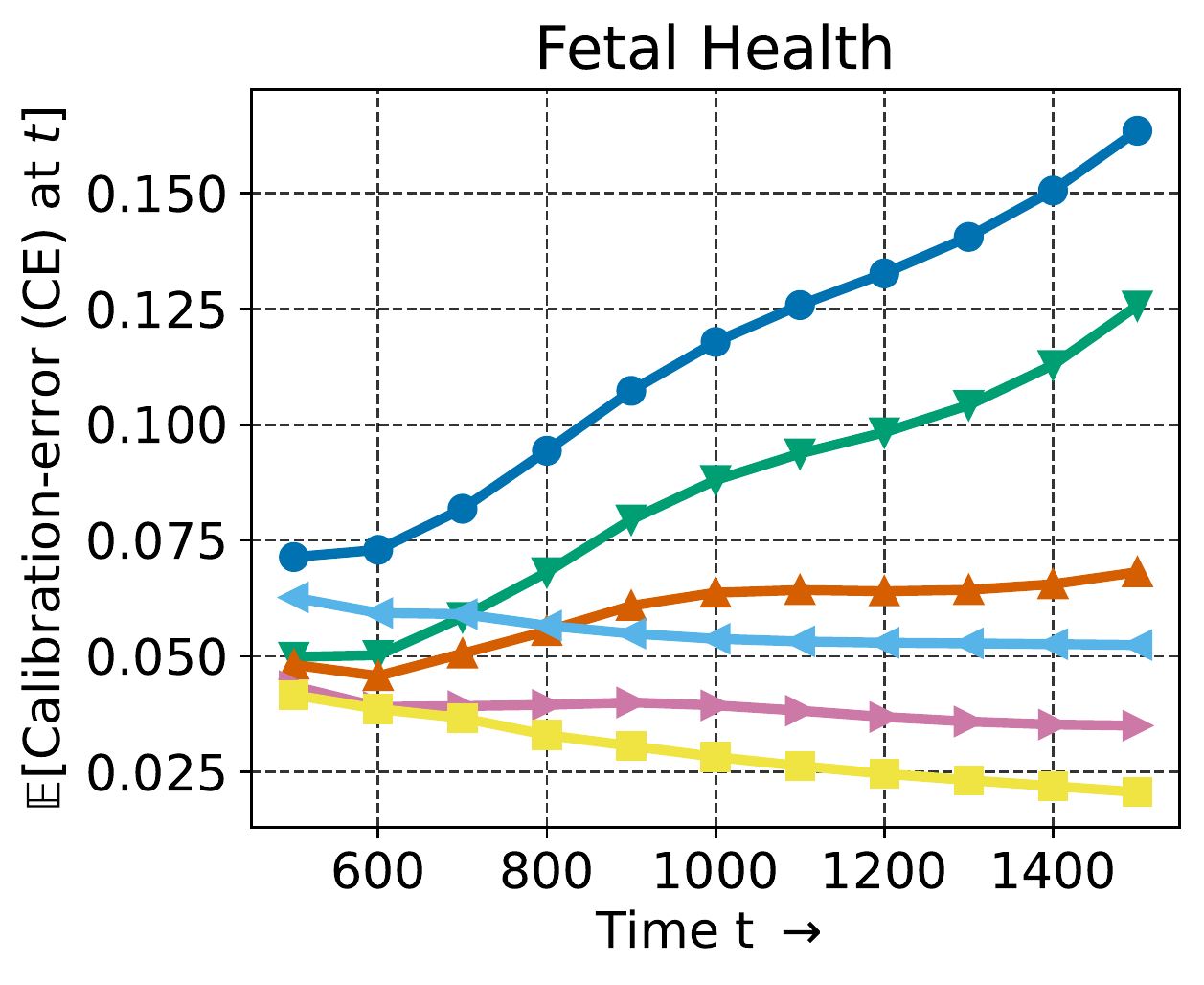}
        \vspace{-0.1cm}
        \vspace{-0.1cm}
        \caption{\textbf{Drifting data.} %
        \ce~(calibration error) values over time of considered models on four datasets with synthetically induced drifts. The plots have invisible error bars since variation across the 100 runs was small. OPS consistently performs better than BM, FPS, and WPS, while TOPS is the best-performing among all methods across datasets and time. All methods had roughly the same \shp~values at a given time-step, so the \shp~plots are delayed to Appendix~\ref{appsec:all-additional-collection} (Figure~\ref{fig:real-data-shp}). 
    \vspace{-0.3cm}}
    \label{fig:real-data-ce}
\end{figure*}

\textbf{Metrics. }We measured the \shp~and \ce~metrics defined in \eqref{eq:sharpnes-definition} and \eqref{eq:ce-definition} respectively. Although estimating population versions of \shp~and \ce~in statistical (i.i.d.) settings is fraught with several issues (\citet{kumar2019calibration, roelofs2022mitigating} and several other works), our definitions target actual observed quantities which are directly interpretable without reference to population quantities. %

\textbf{Reading the plots.} The plots we report show \ce~values at certain time-stamps starting from $T_\text{cal} + 2W$ and ending at $T$ %
(see third line of Collection 1). $T_\text{cal}$ and $W$ are fixed separately for each dataset (Table~\ref{tab:metadata} in Appendix). We also generated \shp~plots, but these are not reported since the drop in \shp~was always very small. 

\vspace{-0.1cm}
\subsection{Experiments on real datasets}
\label{sec:experiments-real}
We worked with four public datasets in two settings. Links to the datasets are in Appendix~\ref{appsec:additional-exp}.%

\textbf{Distribution drift.} We introduced synthetic drifts in the data based on covariate values, so this is an instance of covariate drift. For example, in the bank marketing dataset (leftmost plot in Figure~\ref{fig:real-data-ce}), the problem is to predict which clients are likely to subscribe to a term deposit if they are targeted for marketing, using covariates like \texttt{age}, \texttt{education}, and \texttt{bank-balance}. We ordered the available 12000 rows roughly by \texttt{age} by adding a random number uniformly from $\{-1, 0, 1\}$ to \texttt{age} and sorting all the data. Training is done on the first 1000 points, $T_\text{cal} = 1000$, and $W = 500$. Similar drifts are induced for the other datasets, and $T_\text{cal}, W$ values are set depending on the total number of points; further details %
are in Appendix~\ref{appsec:additional-exp}. 

All simulations were performed 100 times and the average \ce~and \shp~values with $\pm$ std-deviation errorbars were evaluated at certain time-steps. Thus, our lines correspond to estimates of the expected values of \ce~and \shp, as indicated by the Y-axis labels. We find that across datasets, OPS has the least \ce~among non-calibeating methods, and both forms of calibeating typically improve OPS further (Figure~\ref{fig:real-data-ce}). Specifically, TOPS performs the best by a margin compared to other methods. We also computed \shp~values, which are reported in Appendix \ref{appsec:all-additional-collection} (Figure~\ref{fig:real-data-shp}). The drop in \shp~is insignificant in each case (around $0.005$).

\begin{figure*}[htp]
    \centering
    \begin{subfigure}{\linewidth}
    \centering
    \includegraphics[trim=0 10cm 0 0, clip, width=0.8\linewidth]{figs_comparisons_legend.pdf}
    \end{subfigure}
        \centering
        \includegraphics[trim=0 0 0 0, clip, width=0.22\linewidth]{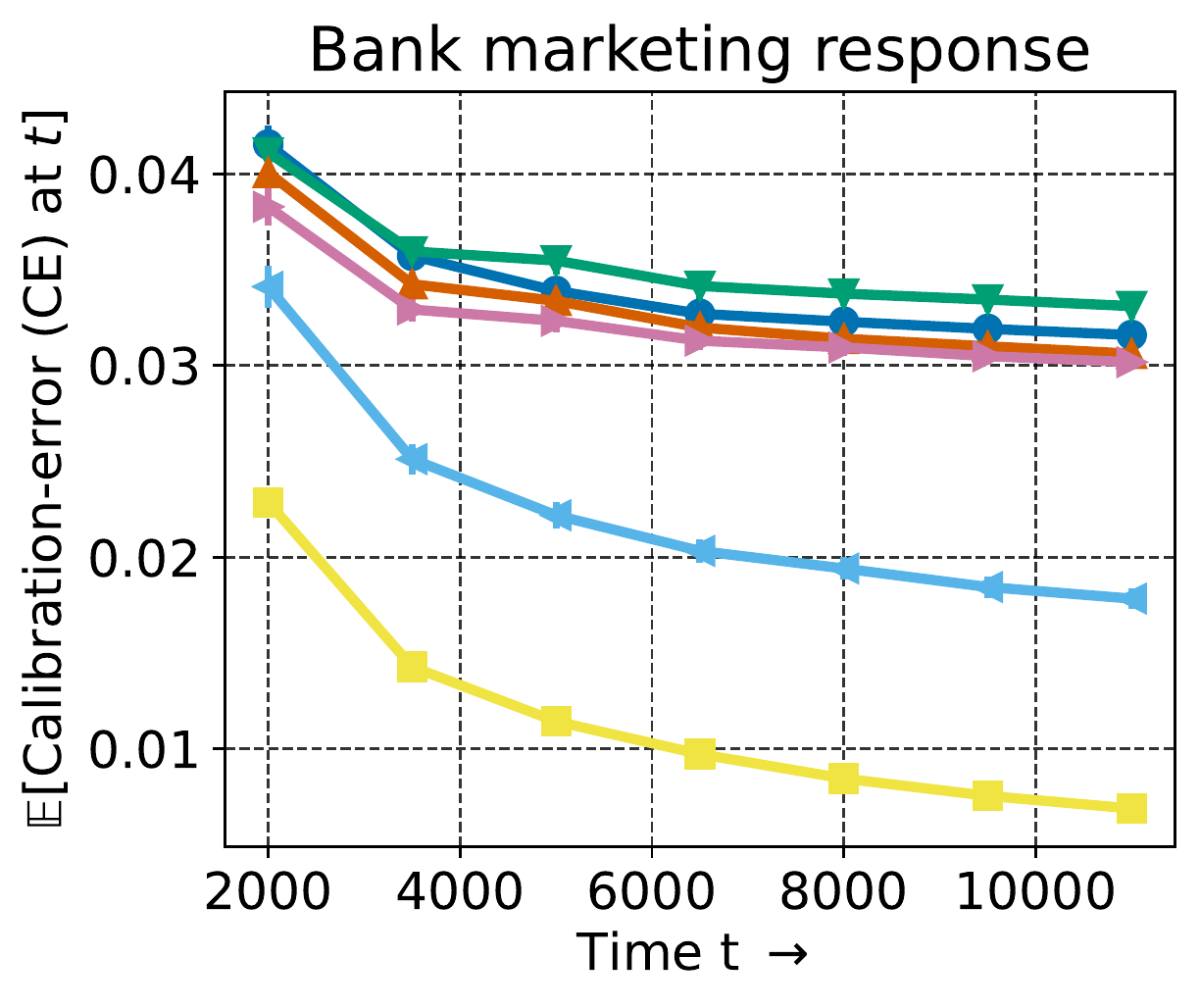}
        \includegraphics[trim=0 0 0 0, clip, width=0.22\linewidth]{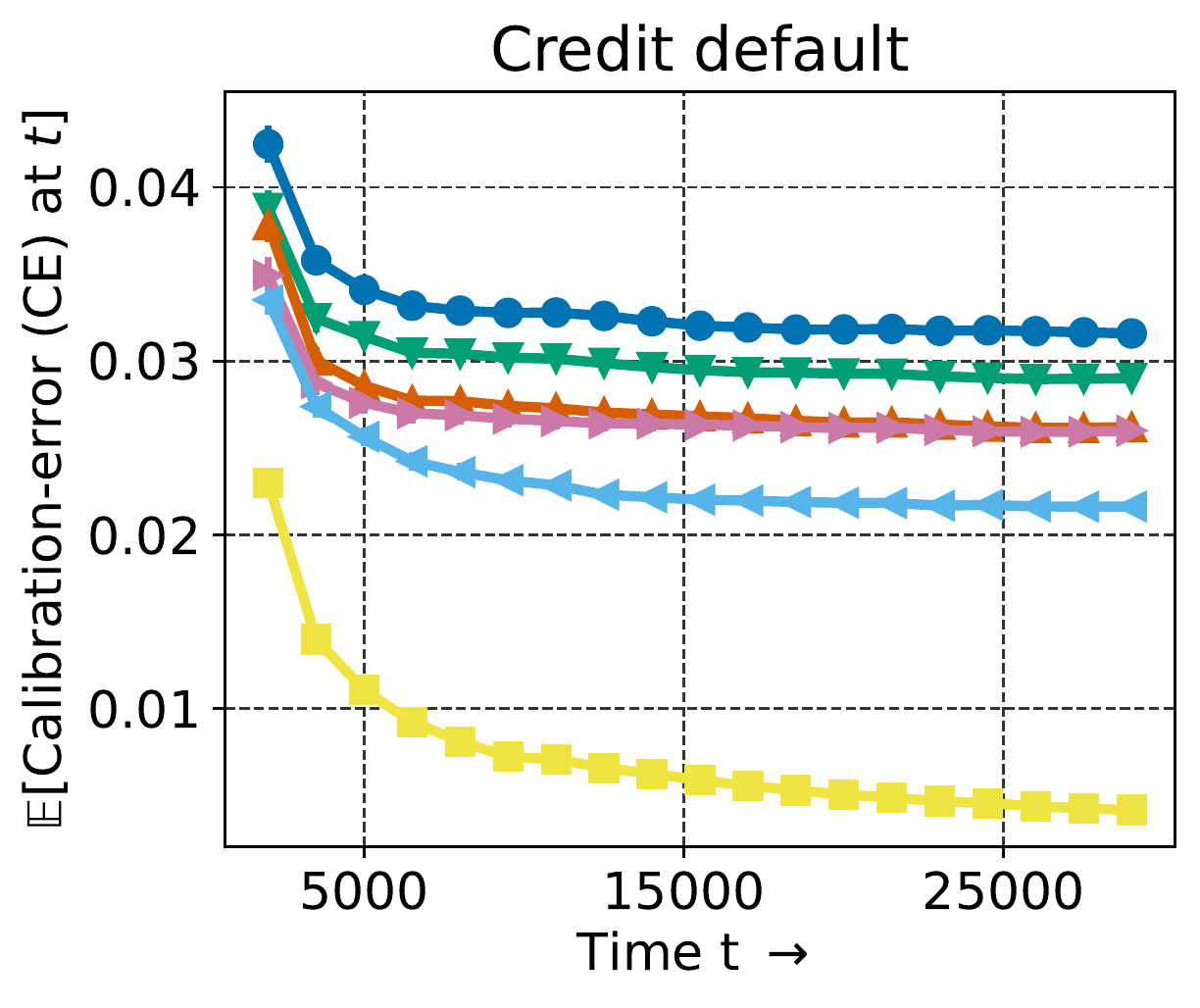}
        \includegraphics[trim=0 0 0 0, clip, width=0.22\linewidth]{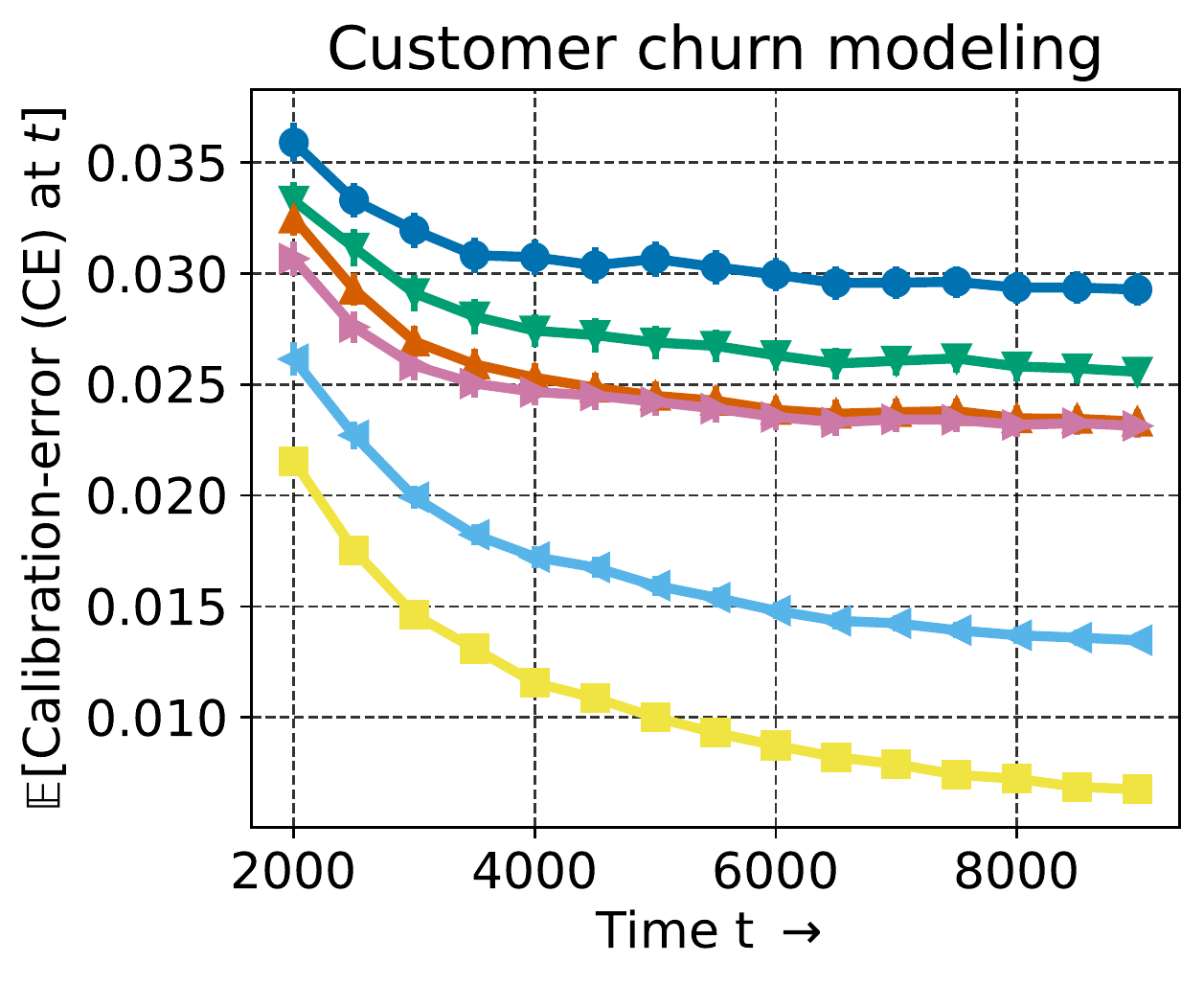}
        \includegraphics[trim=0 0 0 0, clip, width=0.22\linewidth]{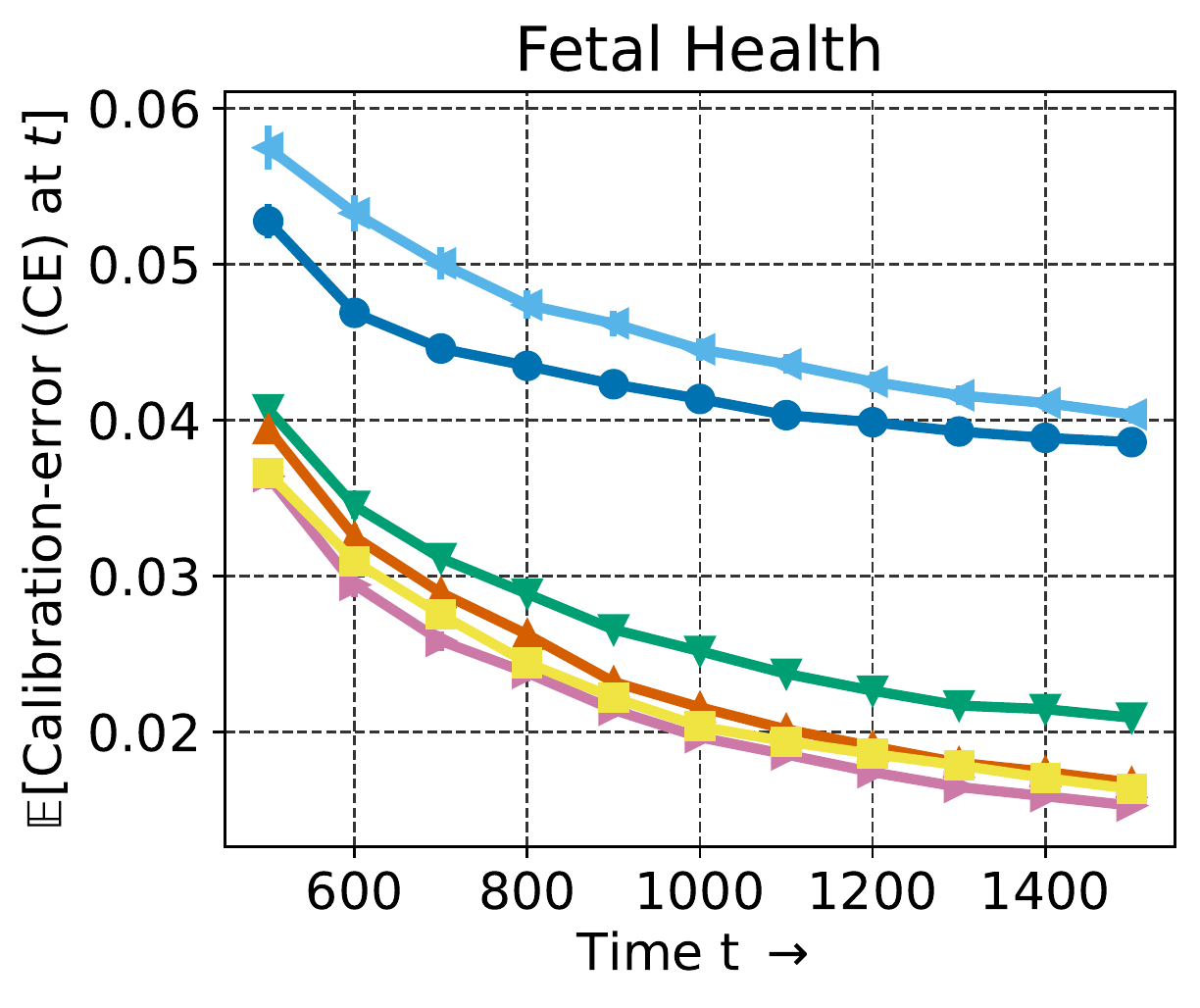}
        \vspace{-0.2cm}\caption{\textbf{IID data.} %
        CE values over time of considered models with four randomly shuffled (ie, nearly i.i.d.) datasets. The plots have invisible error bars since variation across runs was small. TOPS achieves the smallest values of CE throughout.}%
    \label{fig:real-data-iid-ce}\vspace{-0.3cm}
\end{figure*}

\textbf{IID data. }This is the usual batch setting formed by shuffling all available data. Part of the data is used for training and the rest forms the test-stream. We used the same values of $T_\text{cal}$ and $W$ as those used in the data drift experiments (see Appendix~\ref{appsec:additional-exp}). In our experiments, we find that the gap in \ce~between BM, FPS, OPS, and WPS is smaller (Figure~\ref{fig:real-data-iid-ce}). However, TOPS   performs the best in all scenarios, typically by a margin. Here too, the change in \shp~was small, so those plots were delayed to Appendix \ref{appsec:all-additional-collection} (Figure~\ref{fig:real-data-iid-shp}).

\vspace{-0.1cm}
\subsection{Synthetic experiments}
In all experiments with real data, WPS performs almost as good as OPS. In this subsection, we consider some synthetic data drift experiments where OPS and TOPS continue performing well, but WPS performs much worse. 

\vspace{-0.1cm}
\textbf{Covariate drift. }%
Once for the entire process, we draw random orthonormal vectors $\vvec_1, \vvec_2 \in \Real^{10}$ ($\enorm{\vvec_1} = \enorm{\vvec_2} = 1$, $\vvec_1^\intercal \vvec_2 = 0$), a random weight vector $\w \in \{-1, 1\}^{10 + \binom{10}{2}}$ with each component set to $1$ or $-1$ independently with probability $0.5$, and set a drift parameter $\delta \geq 0$. The data is generated as follows:\vspace{-0.1cm}
\begin{align*}
    \uvec_t &= \vvec_1\cos(\delta t) + \vvec_2 \sin(\delta t),X_t \sim \Ncal(\mathbf{0}, \mathcal{I}_{10} + 10\uvec_t\uvec_t^\intercal), \\
    Y_t&|X_t \sim \text{Bernoulli}(\sigmoid(\w^\intercal\widetilde{X}_t)), \text{ where }\\ \widetilde{X}_t &= [\x_1,\ldots,\ \x_{10},\ \x_1\x_2,\ \x_1 \x_3,\ldots,\ \x_9\x_{10}] \in \Real^{10 + \binom{10}{2}}.
\end{align*}
Thus the distribution of $Y_t$ given $X_t$ is fixed as a logistic model over the expanded representation $\widetilde{X}_t$ that includes all cross-terms (this is unknown to the forecaster who only sees $X_t$). The features $X_t$ themselves are normally distributed with mean $\mathbf{0}$ and a time-varying covariance matrix. The principal component (PC) of the covariance matrix is a vector $\uvec_t$ that is rotating on the two-dimensional plane containing the orthonormal vectors $\vvec_1$ and $\vvec_2$. The first 1000 points are used as training data, the remaining $T=5000$ form a test-stream, and $W = 500$. We report results in two settings: one is i.i.d., that is $\delta = 0$, and the other is where the $\uvec$ for the first and last point are at a $180^{\circ}$ angle (Figure~\ref{fig:cs-d}).

\textbf{Label drift. }Given some $\delta > 0$, $(X_t, Y_t)$ is generated as:\vspace{-0.1cm}
\begin{align*}
    Y_t &\sim \text{Bernoulli}(0.5 + \delta t),\\
    X_t|Y_t &\sim \indicator{Y_t = 0}\Ncal(\mathbf{0}, \Real^{10}) + \indicator{Y_t = 1}\Ncal(\mathbf{e_1}, \Real^{10}).
\end{align*}
Thus $P(Y_1 = 1) = 0.5 + \delta$ and for the last test point, $P(Y_{6000} = 1) = 0.5 + 6000\delta$. This final value can be set to control the extent of label drift; we show results with no drift (i.e., $\delta=0$, Figure~\ref{fig:ls-d} left) and $\delta$ set so that final bias $0.5 + 6000\delta= 0.9$ (Figure~\ref{fig:ls-d} right). The number of training points is 1000, $T=5000$, and $W=500$.

\begin{figure}[t]
    \centering%
    \begin{subfigure}{0.85\linewidth}
    \includegraphics[width=0.49\textwidth]{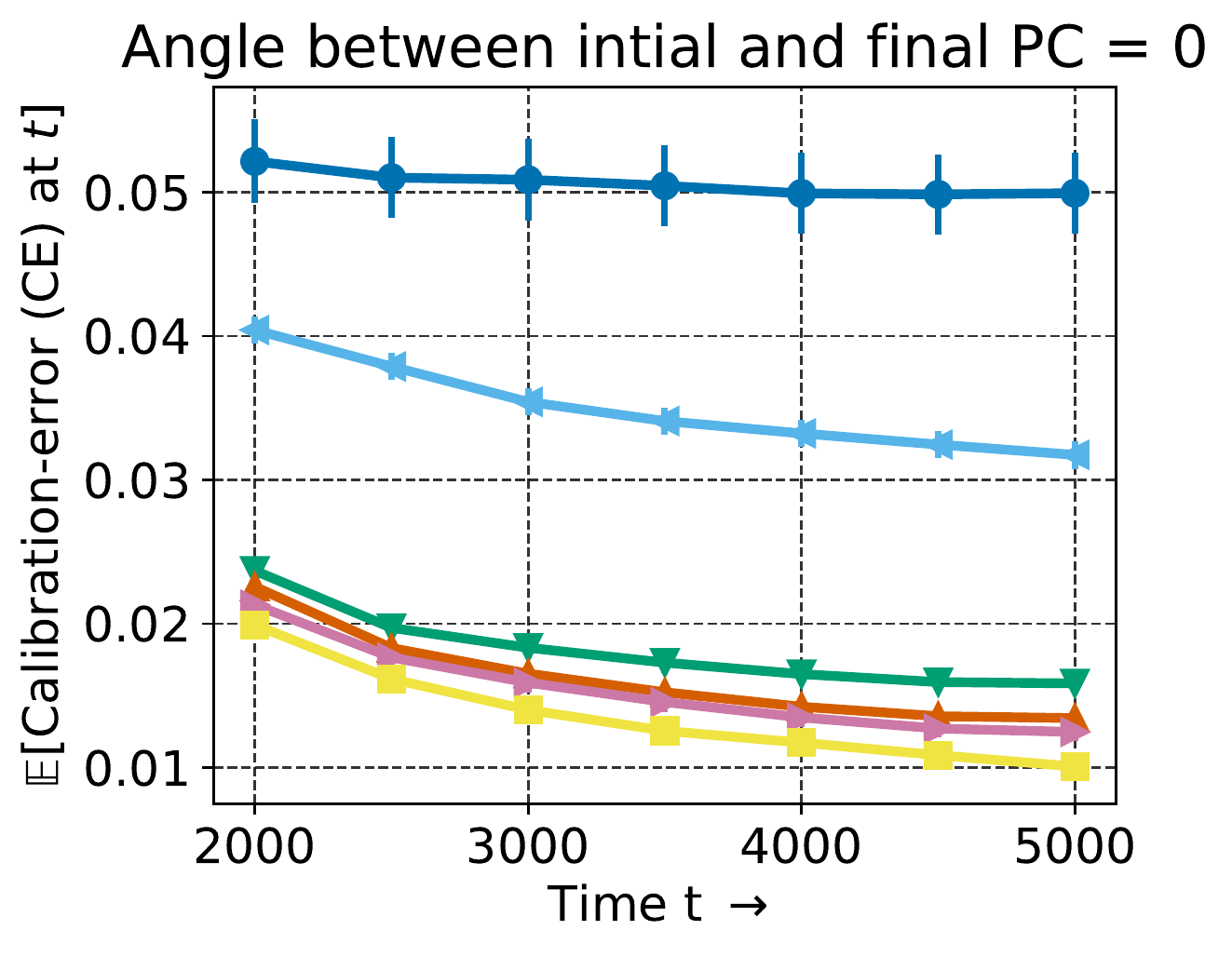}
    \includegraphics[width=0.49\textwidth]{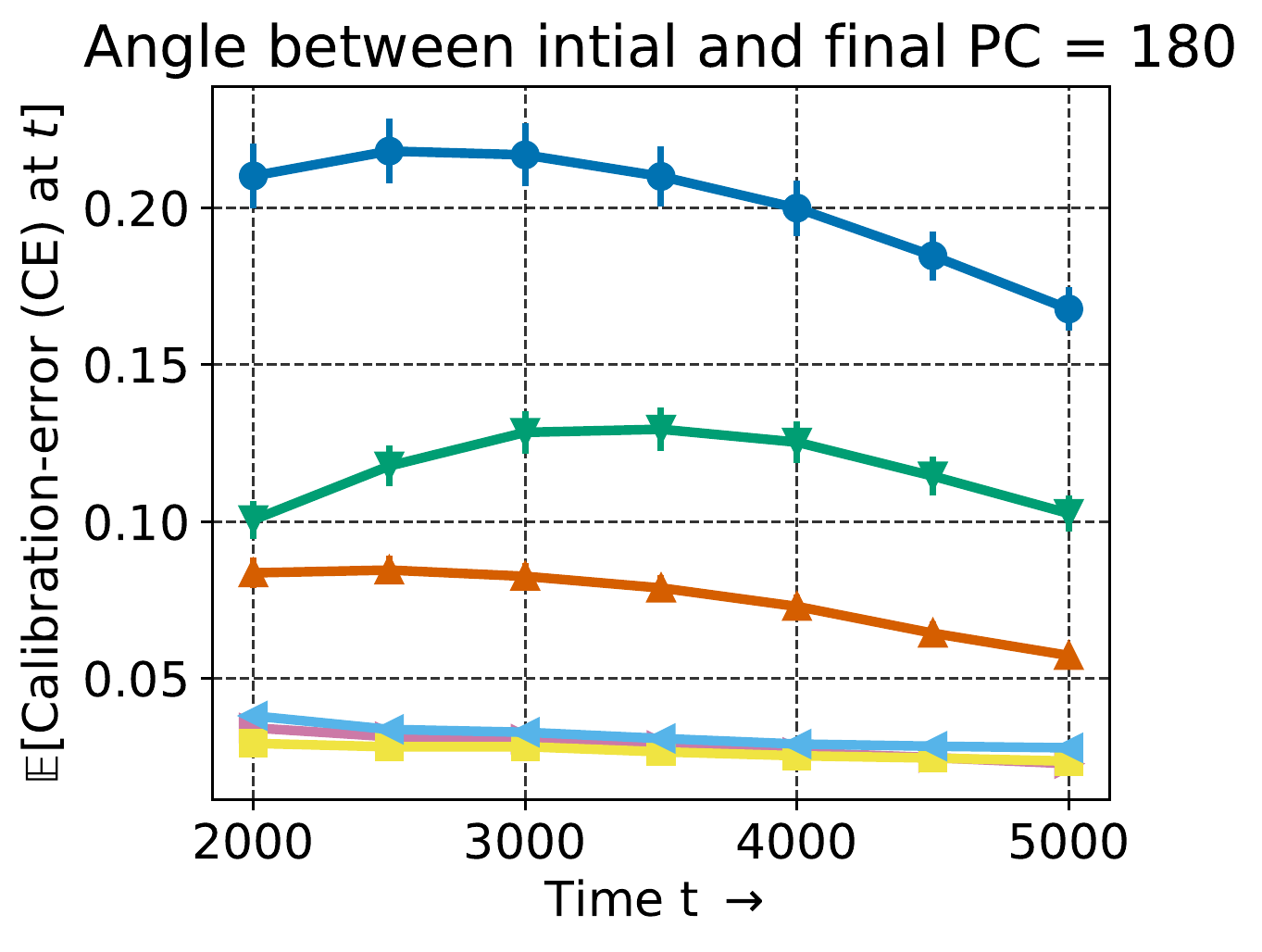}
    \vspace{-0.2cm}
    \caption{Left plot: i.i.d.\ data, right plot: covariate drift.}
    \label{fig:cs-d}
    \end{subfigure}
    \begin{subfigure}{0.85\linewidth}
    \vspace{0.2cm}
    \includegraphics[width=0.49\textwidth]{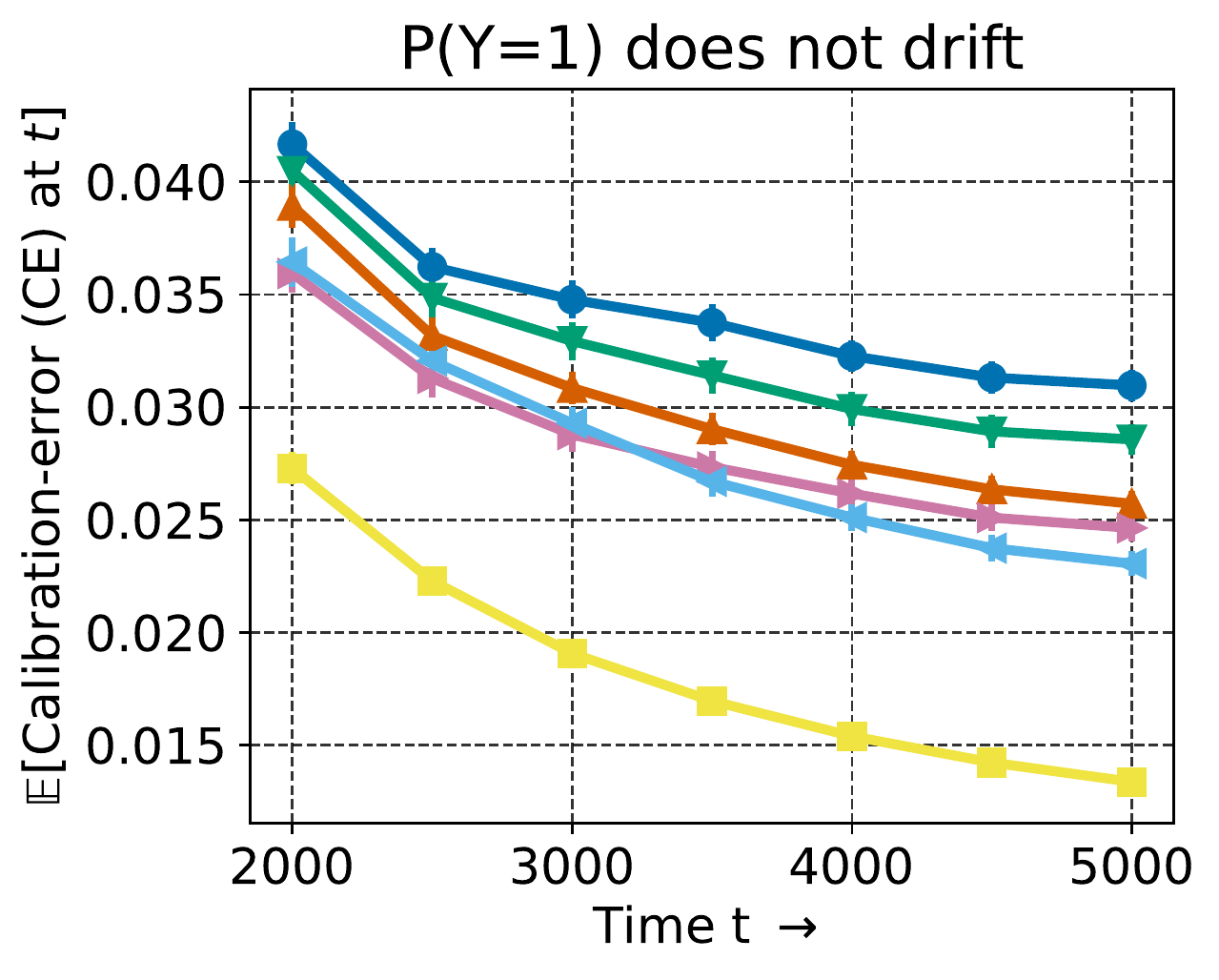}
    \includegraphics[width=0.49\textwidth]{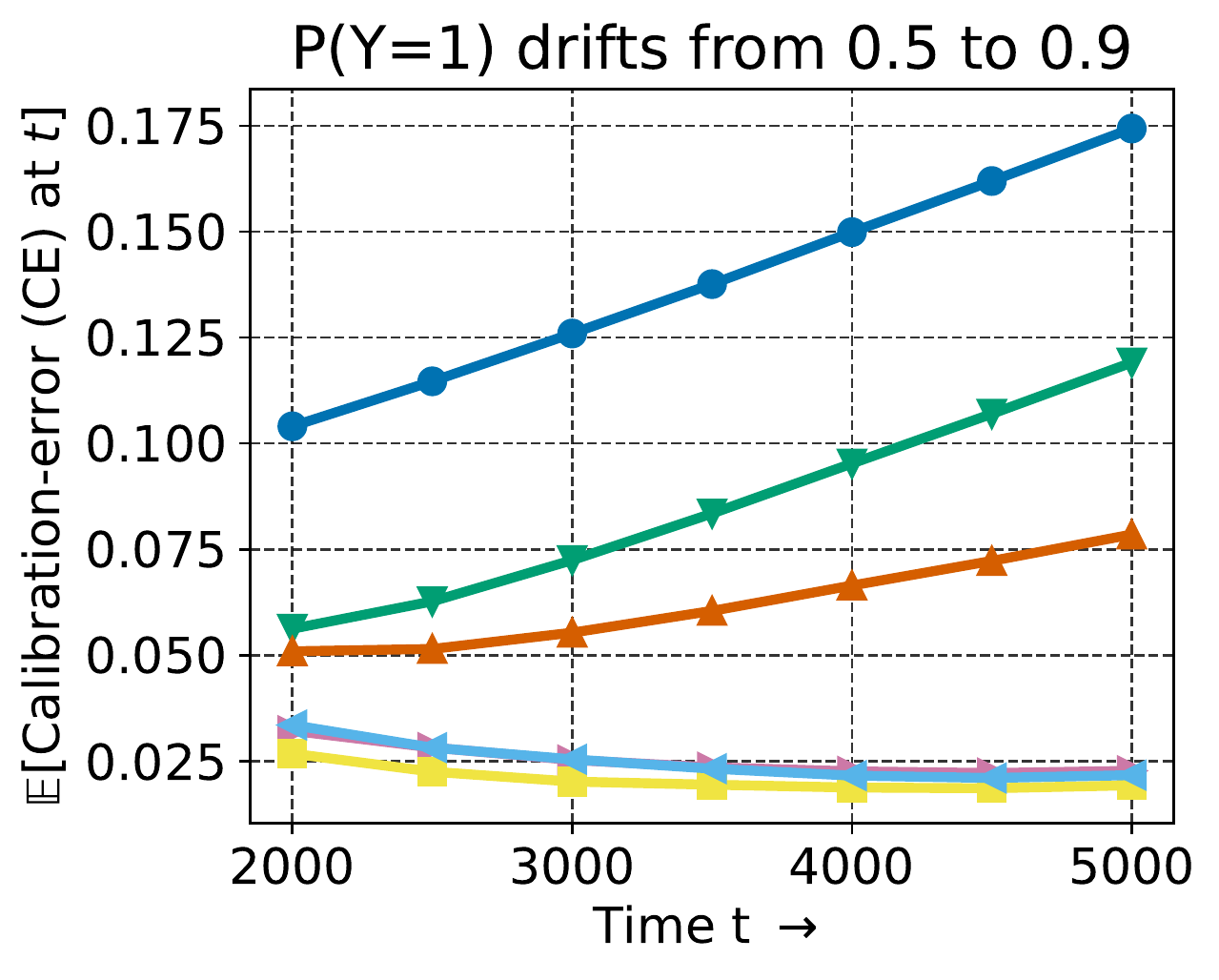}
    \vspace{-0.2cm}
    \caption{Left plot: i.i.d.\ data, right plot: label drift.}
    \label{fig:ls-d}
    \end{subfigure}
    \vspace{-0.1cm}
    \caption{Experiments with synthetic data. In all cases, TOPS has the lowest \ce~across time.}
    \vspace{-0.3cm}
    \label{fig:synthetic-exps}
\end{figure}
\vspace{-0.2cm}

\subsection{Changing $\epsilon$, histogram binning, beta scaling}
\revision{In Appendix~\ref{appsec:additional-exps}, we report versions of Figures~\ref{fig:real-data-ce}, \ref{fig:real-data-iid-ce} with $\epsilon = 0.05, 0.2$ (instead of $\epsilon = 0.1$) with similar conclusions (Figures~\ref{fig:real-data-eps-0.05}, \ref{fig:real-data-eps-0.2}). We also perform comparisons with a windowed version of the popular histogram binning method \citep{zadrozny2001obtaining} and online versions of the beta scaling method, as discussed in the forthcoming Section~\ref{sec:beta-scaling}.  %
}

\section{Online beta scaling with calibeating}
\label{sec:beta-scaling}
\revision{A recalibration method closely related to Platt scaling is beta scaling \citep{kull2017beyond}. The beta scaling mapping $m$ has three parameters $(a, b, c) \in \Real^3$,
\begin{gather*}
{\scalebox{0.9}{
    $m^{a, b, c}_{\text{beta}}(f(\x)) := \sigmoid(a\cdot\log f(\x) + b\cdot\log(1-f(\x)) + c).$}%
}
\end{gather*}
Observe that enforcing $b = -a$ recovers Platt scaling since $\logit(z) = \log(z) - \log(1-z)$. The beta scaling parameters can be learnt following identical protocols as Platt scaling: (i) the traditional method of fixed batch post-hoc calibration akin to FPS, (ii) a natural benchmark of windowed updates akin to WPS, and (iii) regret minimization based method akin to OPS. This leads to the methods FBS, WBS, and OBS, replacing the ``P" of Platt with the ``B" of beta. Tracking + OBS (TOBS) and Hedging + OBS (HOBS) can be similarly derived. Further details on all beta scaling methods are in Appendix~\ref{appsec:beta-scaling}, where we also report plots similar to Figures~\ref{fig:real-data-ce}, \ref{fig:real-data-iid-ce} for beta scaling (Figure~\ref{fig:real-data-tops}). In a comparison between histogram binning, beta scaling, Platt scaling, and their tracking versions, TOPS and TOBS are the best-performing methods across experiments (Figure~\ref{fig:real-data-bs}).} 
\section{Summary}
We provided a way to bridge the gap between the online (typically covariate-agnostic) calibration literature, where data is assumed to be adversarial, and the (typically i.i.d.) post-hoc calibration literature, where the joint covariate-outcome distribution takes centerstage. First, we adapted the post-hoc method of Platt scaling to the online setting, based on a reduction to logistic regression, to give our OPS algorithm. Second, we showed how calibeating can be applied on top of OPS to give further improvements. 

The TOPS method we proposed has the lowest calibration error in all experimental scenarios we considered. On the other hand, the HOPS method which is based on online adversarial calibration provably controls miscalibration at any pre-defined level and could be a desirable choice in sensitive applications. The good performance of OPS+calibeating lends further empirical backing to the thesis that scaling+binning methods perform well in practice, as has also been noted in prior works \citep{zhang2020mix, kumar2019calibration}. Our theoretical results formalize this empirical observation. 

We note a few directions for future work.
First, online algorithms that control regret on the most recent data have been proposed \citep{hazan2009efficient, zhang2018dynamic}. These approaches could give further improvements on ONS, particularly for drifting data. 
Second, while this paper entirely discusses calibration for binary classification, all binary routines can be lifted to achieve multiclass notions such as top-label or class-wise calibration \citep{gupta2022top}. Alternatively, multiclass versions of Platt scaling \citep{guo2017calibration} such as temperature and vector scaling can also be targeted directly using online multiclass logistic regression \citep{jezequel2021mixability}.

\section*{Acknowledgements}
We thank Youngseog Chung and Dhruv Malik for fruitful discussions, and the ICML reviewers for valuable feedback. We thank 
Volodymyr Kuleshov for reaching out and discussing the relationship to their work. CG was supported by the Bloomberg Data Science Ph.D. Fellowship. For computation, we used allocation CIS220171 from the Advanced Cyberinfrastructure Coordination Ecosystem: Services \& Support (ACCESS) program, supported by NSF grants 2138259, 2138286, 2138307, 2137603, and 2138296. Specifically, we used the Bridges-2 system \citep{towns2014xsede}, supported by NSF grant 1928147, at the Pittsburgh Supercomputing Center (PSC).

\bibliographystyle{plainnat}
\bibliography{references}

\appendix 
\onecolumn
\begin{table}[t]
    \centering
    \begin{tabular}{cccccc}
    \toprule
   Name & $T_\text{train}$ & $T_{\text{cal}}$ & W & Sort-by & Link to dataset \\
    \midrule\addlinespace[0.3cm]
    Bank marketing & 1000 & 1000 & 500 & Age & \parbox{9cm}{\url{https://www.kaggle.com/datasets/kukuroo3/bank-marketing-response-predict}} \\\addlinespace[0.3cm]
   Credit default & 1000 & 1000 & 500 & Sex & \parbox{9cm}{\url{https://www.kaggle.com/datasets/uciml/default-of-credit-card-clients-dataset}} \\\addlinespace[0.3cm]
   Customer churn & 1000 & 1000 & 500 & Location &\parbox{9cm}{\url{https://www.kaggle.com/datasets/shrutimechlearn/churn-modelling}}  \\\addlinespace[0.3cm]
   Fetal health & 626 & 300 & 100 & Acceleration & \parbox{9cm}{\url{https://www.kaggle.com/datasets/andrewmvd/fetal-health-classification}}  \\\addlinespace[0.3cm]
    \bottomrule
  \end{tabular}
    \caption{Metadata for datasets used in Section~\ref{sec:experiments-real}. The sort-by column indicates which covariate was used to order data points. All datasets are under the Creative Commons CC0 license.}
    \label{tab:metadata}
\end{table}

\begin{figure*}[htp]
    \centering
    \begin{subfigure}{\linewidth}
    \centering
    \includegraphics[trim=0 10cm 0 0, clip, width=0.8\linewidth]{figs_comparisons_legend.pdf}
    \end{subfigure}
        \centering
        \includegraphics[trim=0 0 0 0, clip, width=0.24\linewidth]{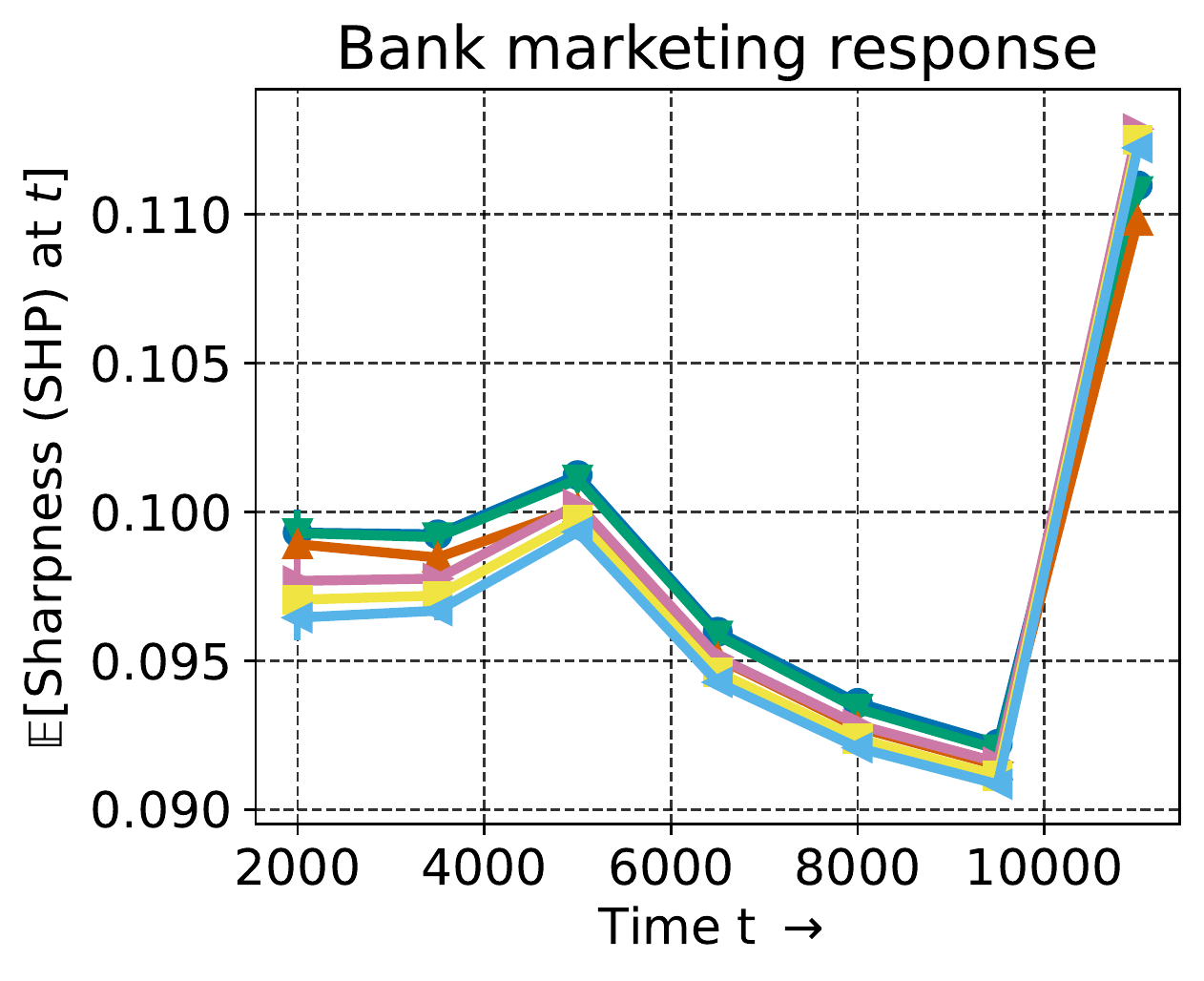}
        \includegraphics[trim=0 0 0 0, clip, width=0.24\linewidth]{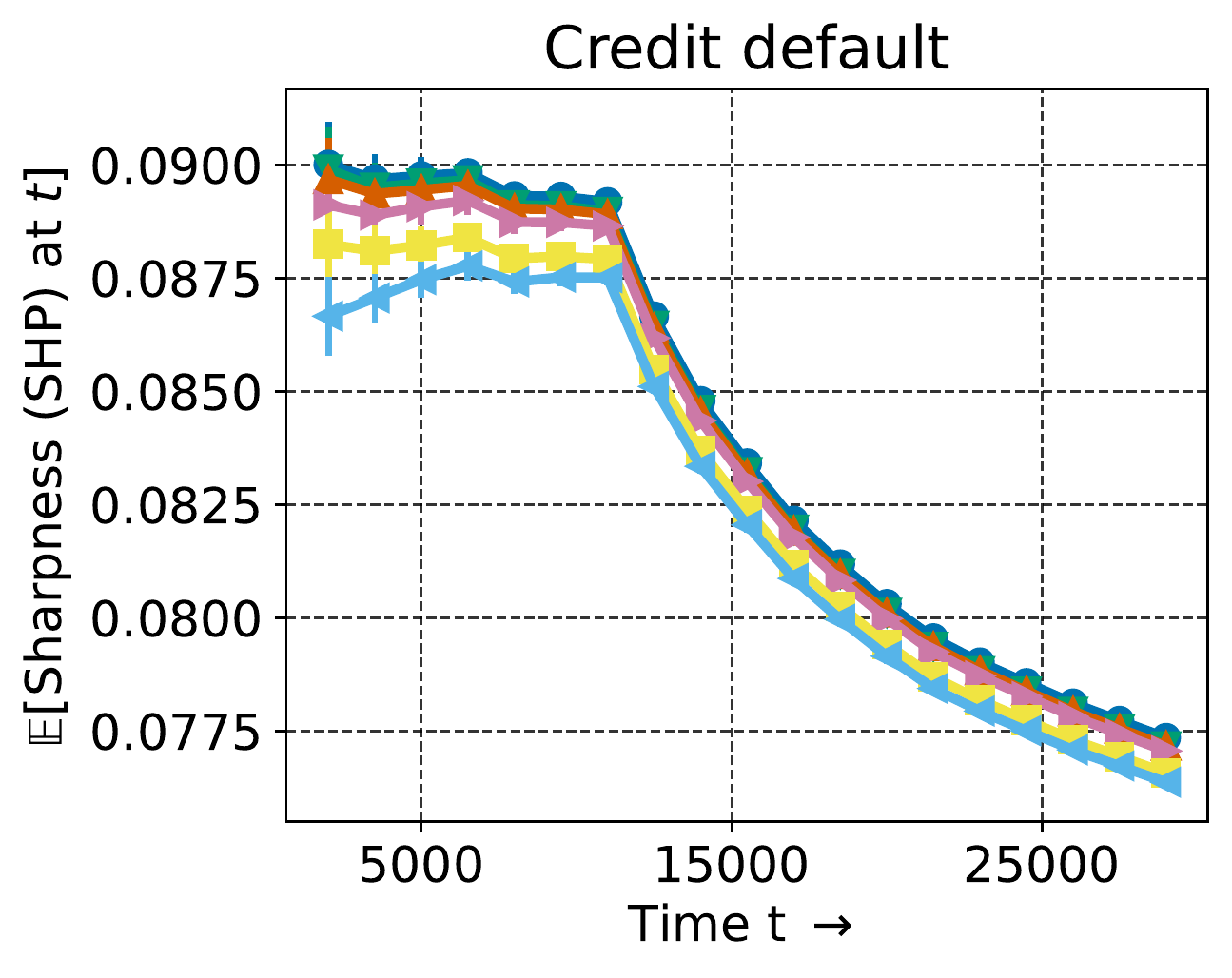}
        \includegraphics[trim=0 0 0 0, clip, width=0.24\linewidth]{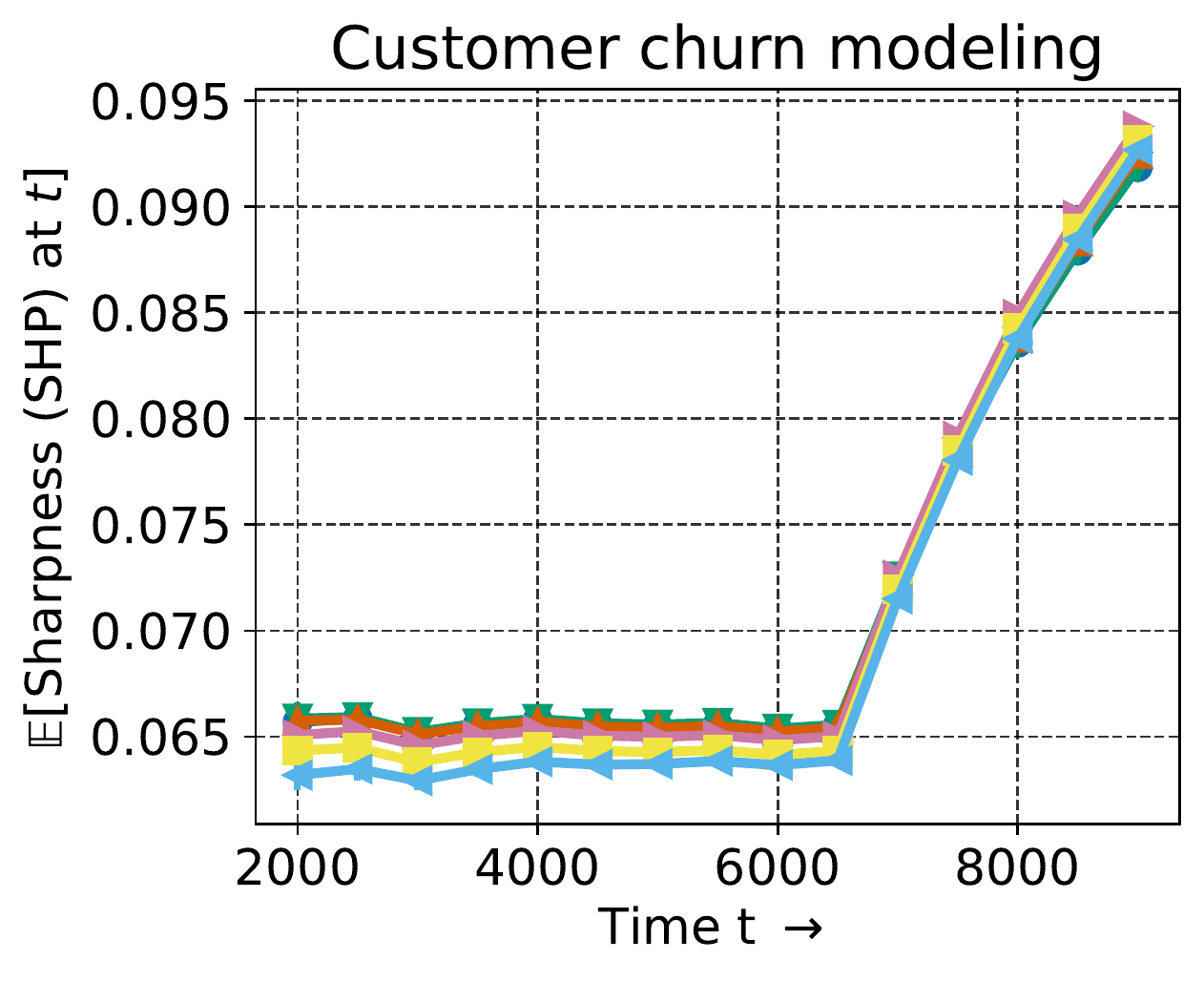}
        \includegraphics[trim=0 0 0 0, clip, width=0.24\linewidth]{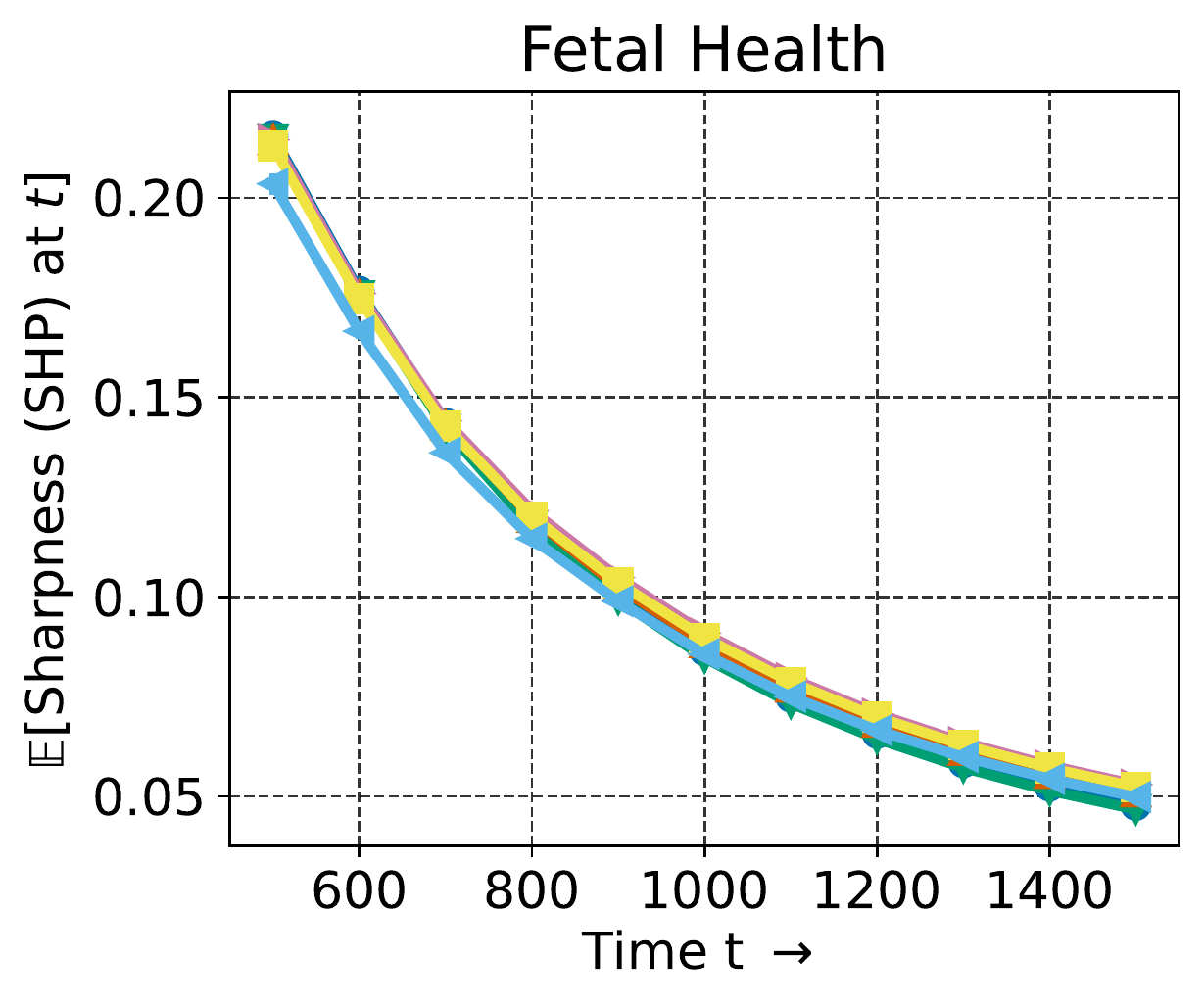}
        \vspace{-0.1cm}
        \caption{\textbf{Sharpness results with drifting data.} %
        \shp~values over time of considered models on four datasets with synthetically induced drifts (Section~\ref{sec:experiments-real}). The plots have invisible error bars since variation across the 100 runs was small. The drop in expected sharpness is below $0.005$ at all times except on the Fetal Health Dataset. 
    \vspace{-0.3cm}}
    \label{fig:real-data-shp}
\end{figure*}

\begin{figure*}[htp]
    \centering
    \begin{subfigure}{\linewidth}
    \centering
    \includegraphics[trim=0 10cm 0 0, clip, width=0.7\linewidth]{figs_comparisons_legend.pdf}
    \end{subfigure}
        \centering
        \includegraphics[trim=0 0 0 0, clip, width=0.24\linewidth]{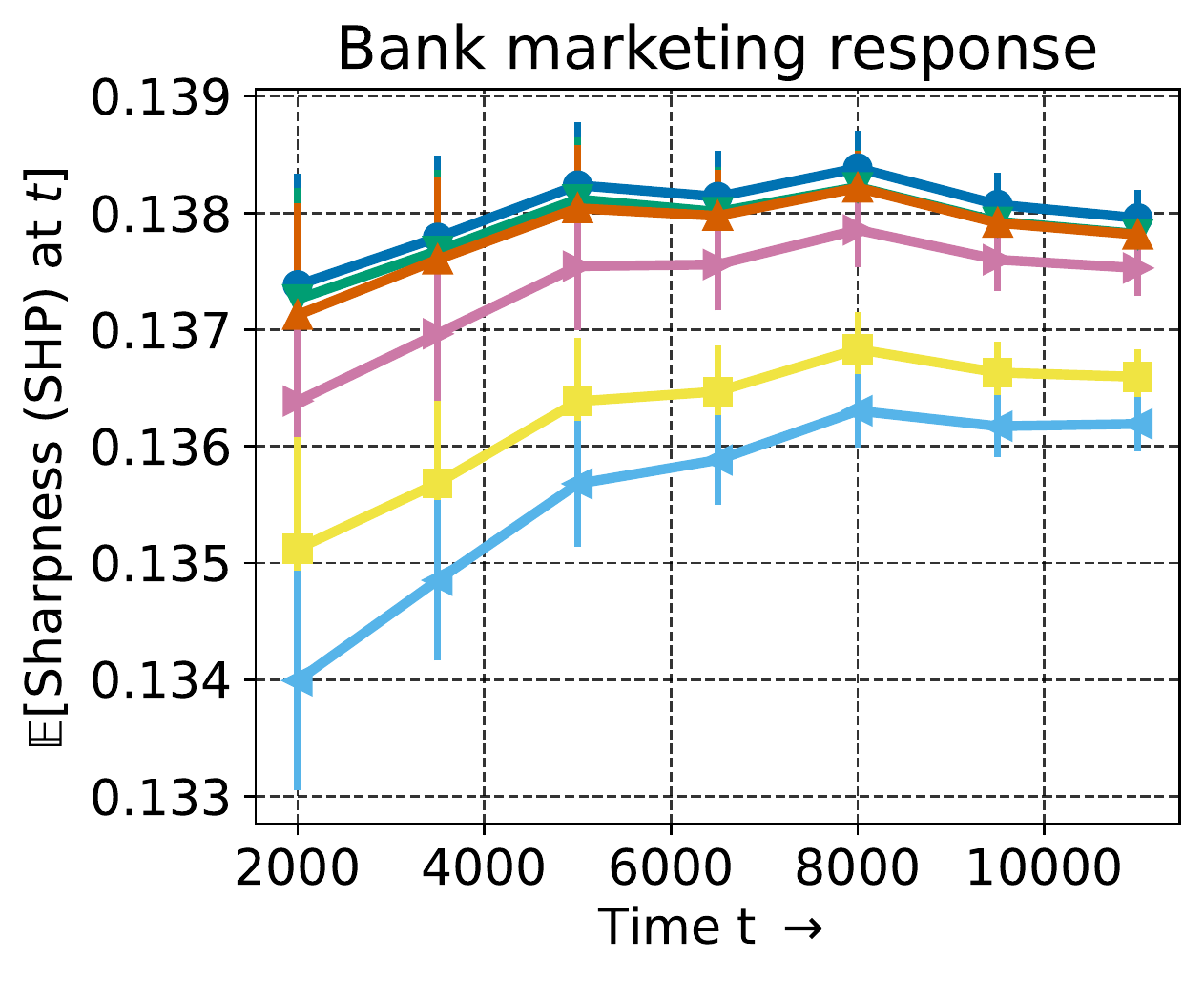}
        \includegraphics[trim=0 0 0 0, clip, width=0.24\linewidth]{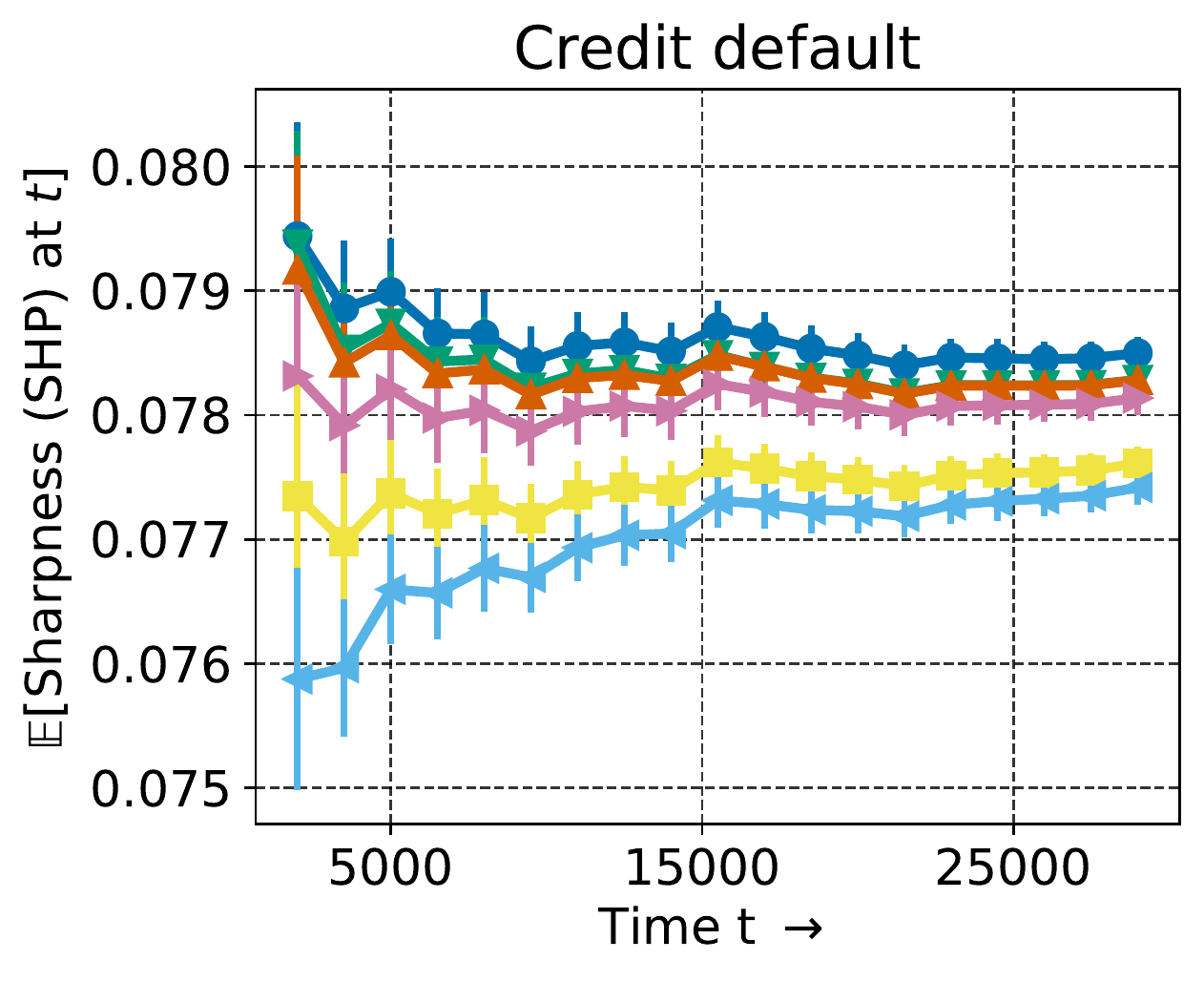}
        \includegraphics[trim=0 0 0 0, clip, width=0.24\linewidth]{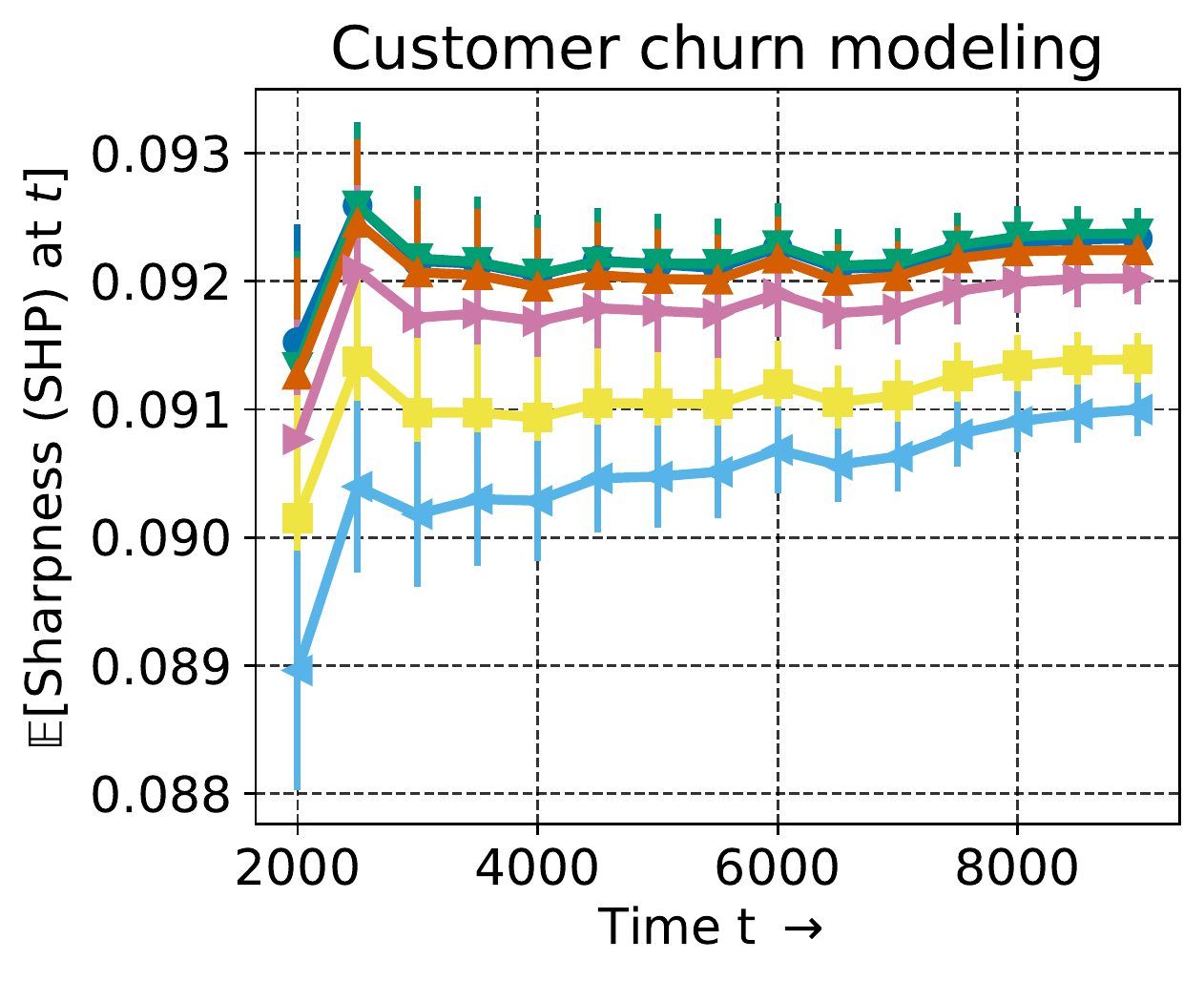}
        \includegraphics[trim=0 0 0 0, clip, width=0.24\linewidth]{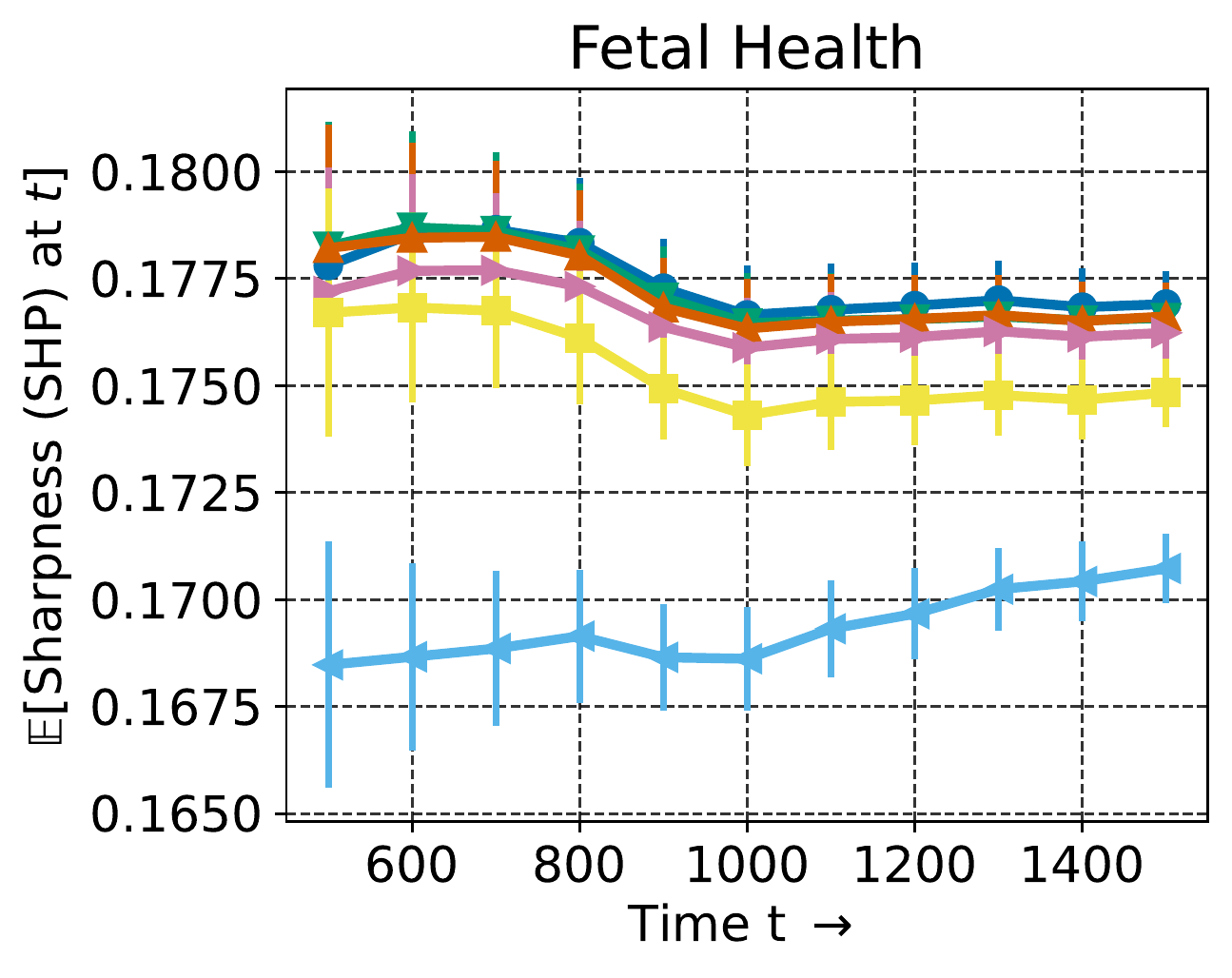}
        \caption{\textbf{Sharpness results with i.i.d.\ data.} %
        \shp~values over time of considered models on four shuffled (ie, nearly i.i.d.) datasets (Section~\ref{sec:experiments-real}). The drop in expected sharpness is less than $0.005$ in all cases except for the HOPS forecaster on the Fetal Health dataset, where it is $0.01$. }%
    \label{fig:real-data-iid-shp}
\end{figure*}

\section{Experimental details and additional results}
\label{appsec:all-additional-collection}
Some implementation details, metadata, information on metrics, and additional results and figures are collected here. 

\subsection{Metadata for datasets used in Section~\ref{sec:experiments-real}}%
\label{appsec:additional-exp}

\begin{figure*}[htp]
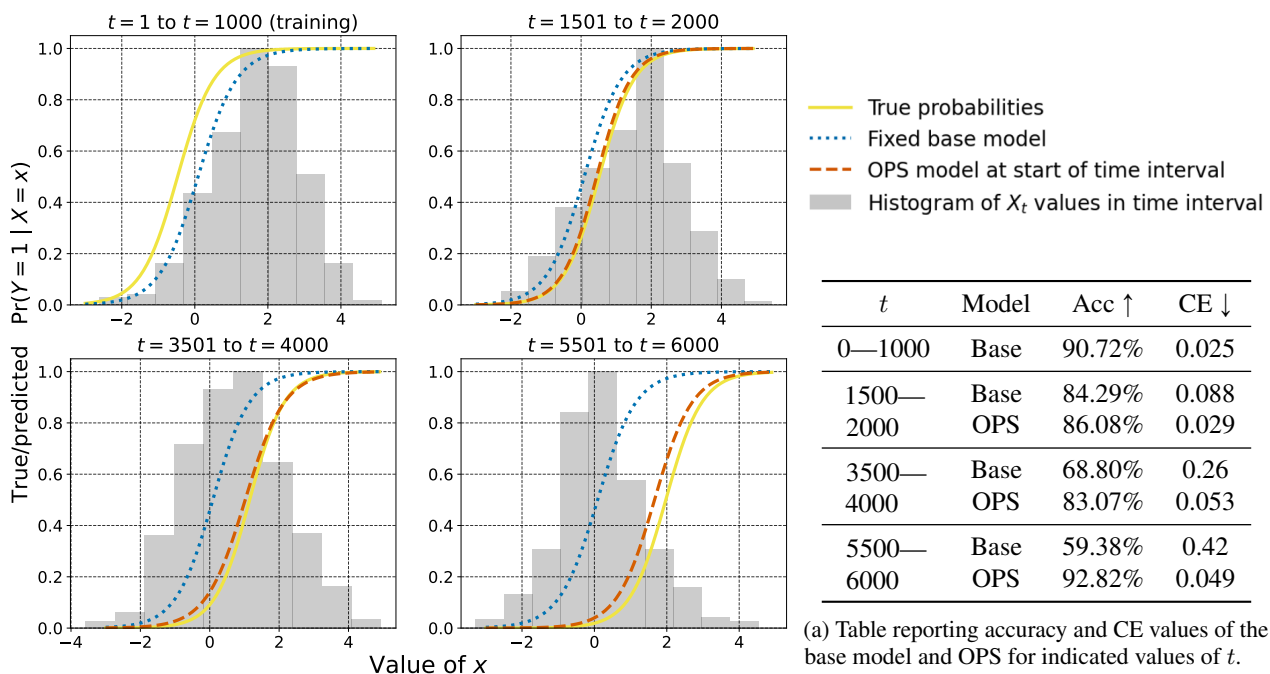

    \centering
    \begin{subfigure}[t]{0.63\textwidth}
        \centering
        \vskip 0pt
        \includegraphics[trim=1cm 0 0 0, clip, width=1.1\textwidth]{figs_label_shift_1d.pdf}
    \end{subfigure}
    \begin{subfigure}[t]{0.36\textwidth}
        \vskip 2cm
        \centering\includegraphics[width=\textwidth, trim=0 -3cm 0 0, clip]{figs_legend_updated.png}
        \begin{tabular}{cccc}\toprule
        $t$ & Model & Acc $\uparrow$ & CE $\downarrow$ \\\midrule
         0---1000 & Base & 90.72\% & 0.025\\
         \midrule
         \multirow{2}{*}{\parbox{1cm}{1500---2000}} & Base & 84.29\% & 0.088\\
          & OPS &  86.08\% & 0.029\\
         \midrule
         \multirow{2}{*}{\parbox{1cm}{3500---4000}} & Base & 68.80\% & 0.26\\
          & OPS & 83.07\% & 0.053 \\
         \midrule
         \multirow{2}{*}{\parbox{1cm}{5500---6000}} & Base & 59.38\% & 0.42 \\
          & OPS & 92.82\% & 0.049 \\\bottomrule
        \end{tabular}
        \caption{Table reporting accuracy and CE values of the base model and OPS for indicated values of $t$.}
        \label{fig:label-shift-illustration-table}
    \end{subfigure}
    \caption{The adaptive behavior of OPS for the simulated label drift scenario described in Section~\ref{sec:cov-shift-illustrative}.}
    \label{fig:label-shift-illustration}
\end{figure*}

\begin{figure*}[htp]
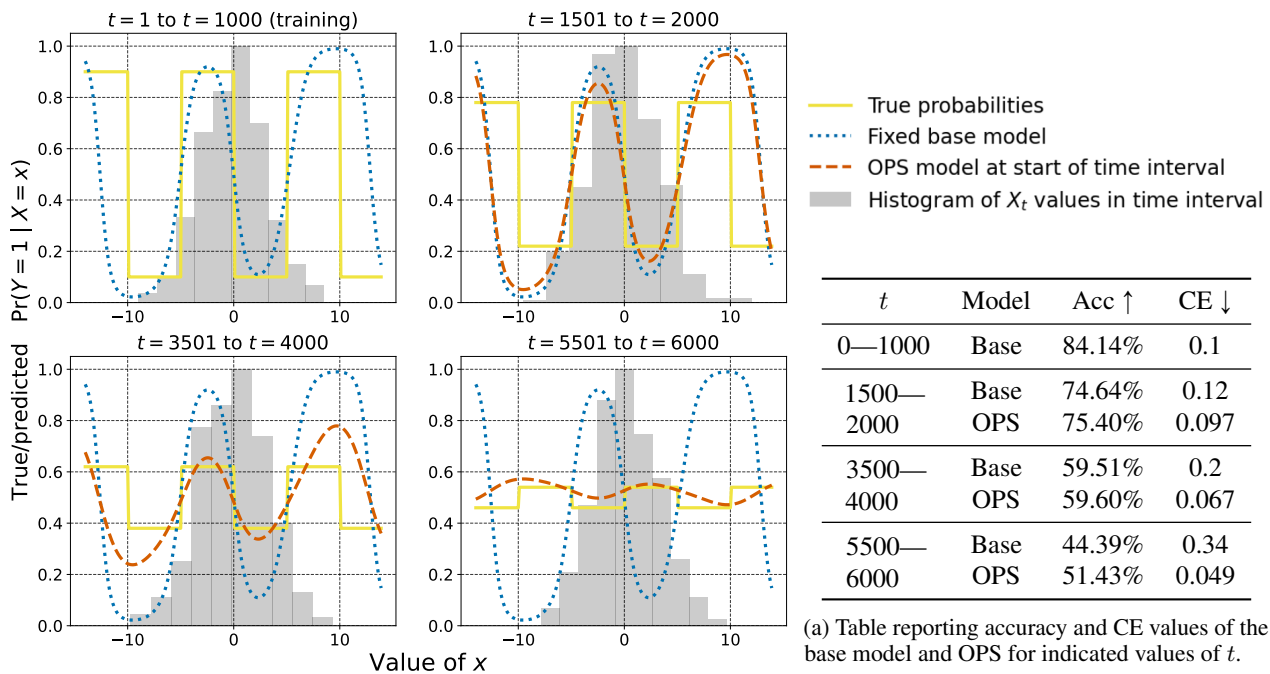

    \centering
    \begin{subfigure}[t]{0.63\textwidth}
        \centering
        \vskip 0pt
        \includegraphics[trim=1cm 0 0 0, clip, width=1.1\textwidth]{figs_regression_shift_1d.pdf}
    \end{subfigure}
    \begin{subfigure}[t]{0.36\textwidth}
        \vskip 2cm
        \centering\includegraphics[width=\textwidth, trim=0 -3cm 0 0, clip]{figs_legend_updated.png}
        \begin{tabular}{cccc}
        \toprule
        $t$ & Model & Acc $\uparrow$ & CE $\downarrow$ \\\midrule
         0---1000 & Base & 84.14\% & 0.1\\
         \midrule
         \multirow{2}{*}{\parbox{1cm}{1500---2000}} & Base & 74.64\% & 0.12\\
          & OPS &  75.40\% & 0.097\\
         \midrule
         \multirow{2}{*}{\parbox{1cm}{3500---4000}} & Base & 59.51\% & 0.2\\
          & OPS & 59.60\% & 0.067 \\
         \midrule
         \multirow{2}{*}{\parbox{1cm}{5500---6000}} & Base & 44.39\% & 0.34 \\
          & OPS & 51.43\% & 0.049 \\\bottomrule
        \end{tabular}
        \caption{Table reporting accuracy and CE values of the base model and OPS for indicated values of $t$.}
        \label{fig:reg-shift-illustration-table}
    \end{subfigure}
    \caption{The adaptive behavior of OPS for the simulated regression-function drift scenario described in Section~\ref{sec:cov-shift-illustrative}.}
    \label{fig:regression-shift-illustration}
\end{figure*}

Table~\ref{tab:metadata} contains metadata for the datasets we used in Section~\ref{sec:experiments-real}. $T_\text{train}$ refers to the number of training examples. The ``sort-by" column indicates which covariate was used to order data points. In each case some noise was added to the covariate in order to create variation for the experiments. The exact form of drift can be found in the python file \texttt{sec\_4\_experiments\_core.py} in the repository \url{https://github.com/AIgen/df-posthoc-calibration/tree/main/Online\%20Platt\%20Scaling\%20with\%20Calibeating}.

\subsection{Additional plots and details for label drift and regression-function drift experiments from Section~\ref{sec:intro}}
\label{appsec:intro-experiments-additional}
Figures~\ref{fig:cov-shift-illustration},~\ref{fig:label-shift-illustration}, and \ref{fig:regression-shift-illustration} report accuracy (Acc) and calibration error (CE) values for the base model and the OPS model in the three dataset drift settings we considered. The Acc values are straightforward averages and can be computed without issues. However, estimation of CE on real datasets is tricky and requires sophisticated techniques such as adaptive binning, debiasing, heuristics for selecting numbers of bins, or kernel estimators \citep{kumar2019calibration, roelofs2022mitigating, widmann2019calibration}. The issue typically boils down to the fact that $\text{Pr}(Y = 1 \mid X = x)$ cannot be estimated for every $x \in \Xcal$ without making smoothness assumptions or performing some kind of binning. However, in the synthetic experiments of Section~\ref{sec:intro}, $\text{Pr}(Y = 1 \mid X)$ is known exactly, so such techniques are not required. For some subset of forecasts $p_s, p_2, \ldots, p_t$,  %
we compute 
\begin{equation*}
    \text{CE} = \frac{1}{t-s+1}\sum_{i=s}^t \abs{p_i - \text{Pr}(Y_i = 1 \mid X_i = \x_i)},
\end{equation*}
on the instantiated values of $X_s, X_{s+1}, \ldots, X_t$. Thus, what we report is the true \ce~given covariate values. 

\subsection{Additional results with windowed histogram binning and changing bin width}
\label{appsec:additional-exps}

\textbf{Comparison to histogram binning (HB)}. HB is a recalibration method that has been shown to have excellent empirical performance as well as theoretical guarantees \citep{zadrozny2001obtaining, gupta2021distribution}. There are no online versions of HB that we are aware of, so we use the same windowed approach as windowed Platt and beta scaling for benchmarking (see Section~\ref{sec:ftl-wps} and the second bullet in Section~\ref{appsec:beta-scaling}). This leads to windowed histogram binning (WHB), the fixed-batch HB recalibrator that is updated every $O(100)$ time-steps. We compare WHB to OPS and OBS (see Section~\ref{sec:beta-scaling}). Since tracking improves both OPS and OBS, we also consider tracking WHB. Results are presented in Figure~\ref{fig:real-data-bs}.

We find that WHB often performs better than OPS and OBS in the i.i.d. case, and results are mixed in the drifting case. However, since WHB is a binning method, it inherently produces something akin to a running average, and so tracking does not improve it further. The best methods (TOPS, TOBS) are the ones that combine one of our proposed parametric online calibrators (OPS, OBS) with tracking.

\textbf{Changing the bin width $\epsilon$}. In the main paper, we used $\epsilon = 0.1$ and defined corresponding bins as in \eqref{eq:B-bins}. This binning reflects in three ways on the experiments we performed. First, $\epsilon$-binning is used to divide forecasts into representative bins before calibeating (equations~\eqref{eq:tops-update}, \eqref{eq:hops-update}). Second, $\epsilon$-binning is used to define the sharpness and calibration error metrics. Third, the hedging procedure F99 requires specifying a binning scheme, and we used the same $\epsilon$-bins. 

Here, we show that the empirical results reported in the main paper are not sensitive to the chosen representative value of $\epsilon=0.1$. We run the same experiment used to produce Figures~\ref{fig:real-data-ce} and \ref{fig:real-data-iid-ce} but with $\epsilon=0.05$ (Figure~\ref{fig:real-data-eps-0.05}) and $\epsilon = 0.2$ (Figure~\ref{fig:real-data-eps-0.2}). The qualitative results remain identical, with TOPS still the best performer and hardly affected by the changing epsilon. In fact, the plots for all methods except HOPS are indistinguishable from their $\epsilon=0.1$ counterparts at first glance. HOPS is slightly sensitive to $\epsilon$: the performance improves slightly with $\epsilon = 0.05$, and worsens slightly with $\epsilon=0.2$. 

\subsection{Reproducibility}
\label{appsec:implementation}
All results in this paper can be reproduced exactly, including the randomization, using the IPython notebooks that can be found at \url{https://github.com/aigen/df-posthoc-calibration} in the folder \texttt{Online Platt scaling with Calibeating}. The README page in the folder contains a table describing which notebook to run to reproduce individual figures from this paper. %

\begin{figure*}[htp]
    \centering
    \begin{subfigure}{\linewidth}
    \centering
    \includegraphics[trim=0 10cm 0 0, clip, width=0.8\linewidth]{figs_comparisons_legend.pdf}
    \end{subfigure}
    \begin{subfigure}{\linewidth}
        \centering
        \includegraphics[trim=0 0 0 0, clip, width=0.24\linewidth]{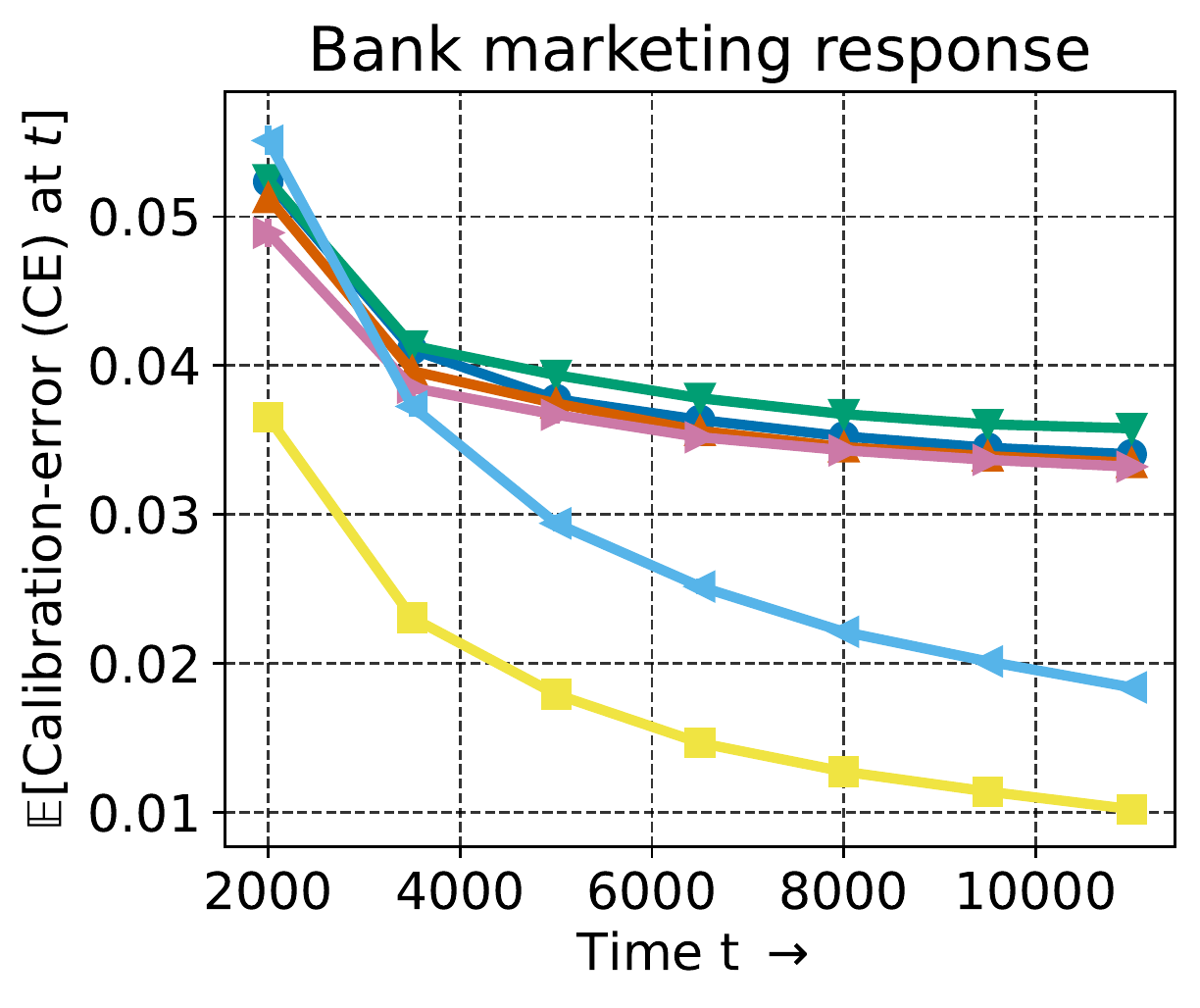}
        \includegraphics[trim=0 0 0 0, clip, width=0.24\linewidth]{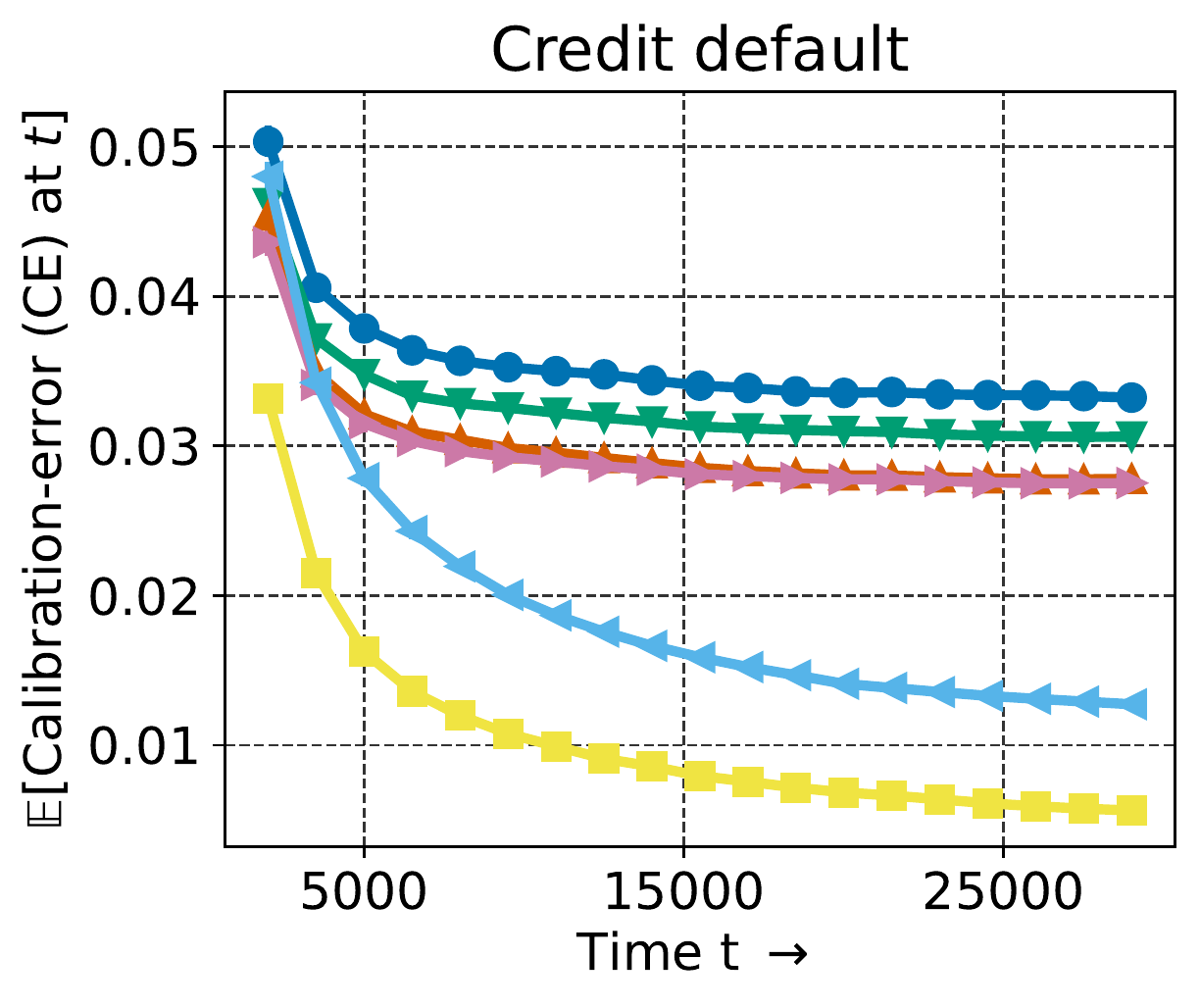}
        \includegraphics[trim=0 0 0 0, clip, width=0.24\linewidth]{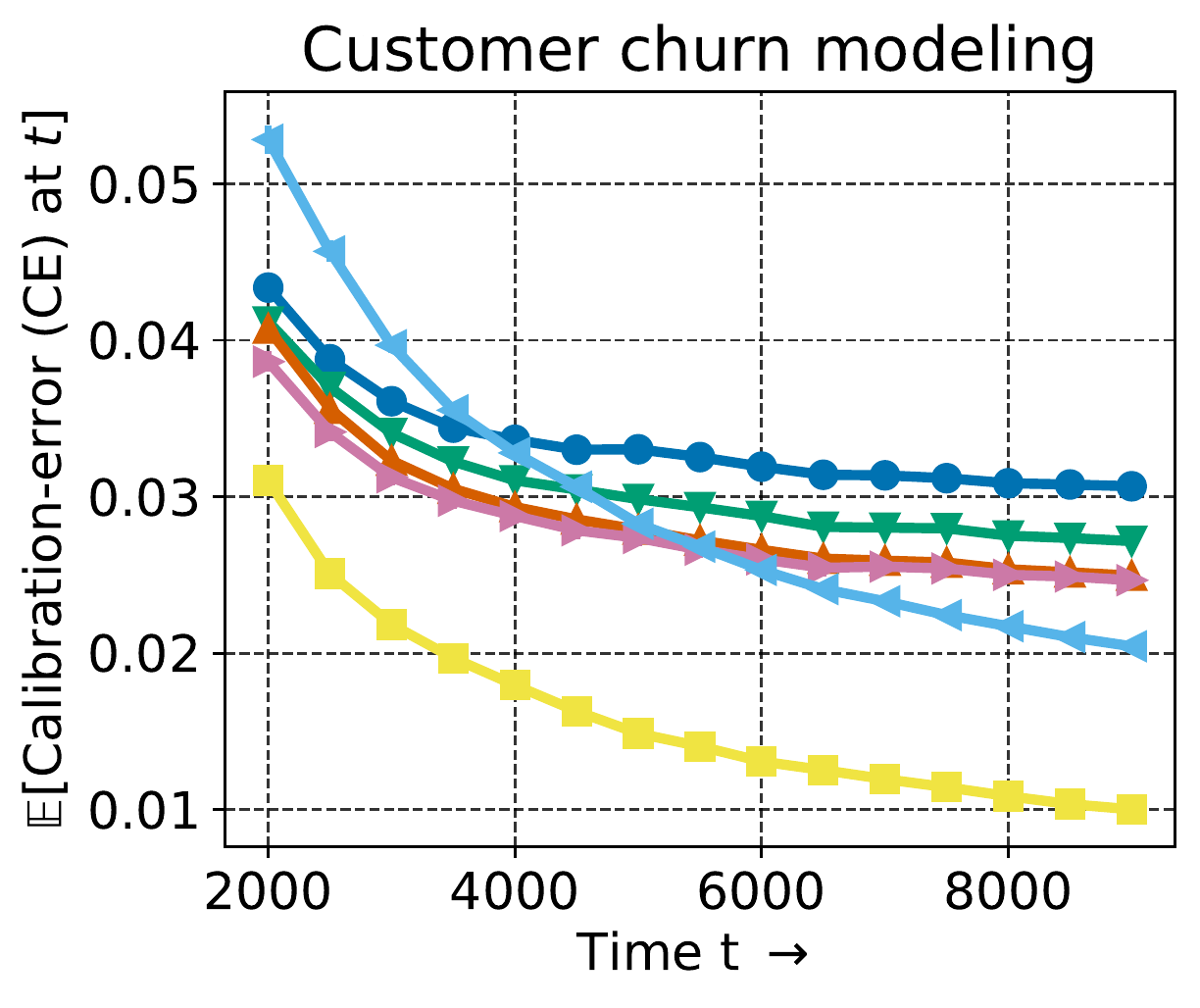}
        \includegraphics[trim=0 0 0 0, clip, width=0.24\linewidth]{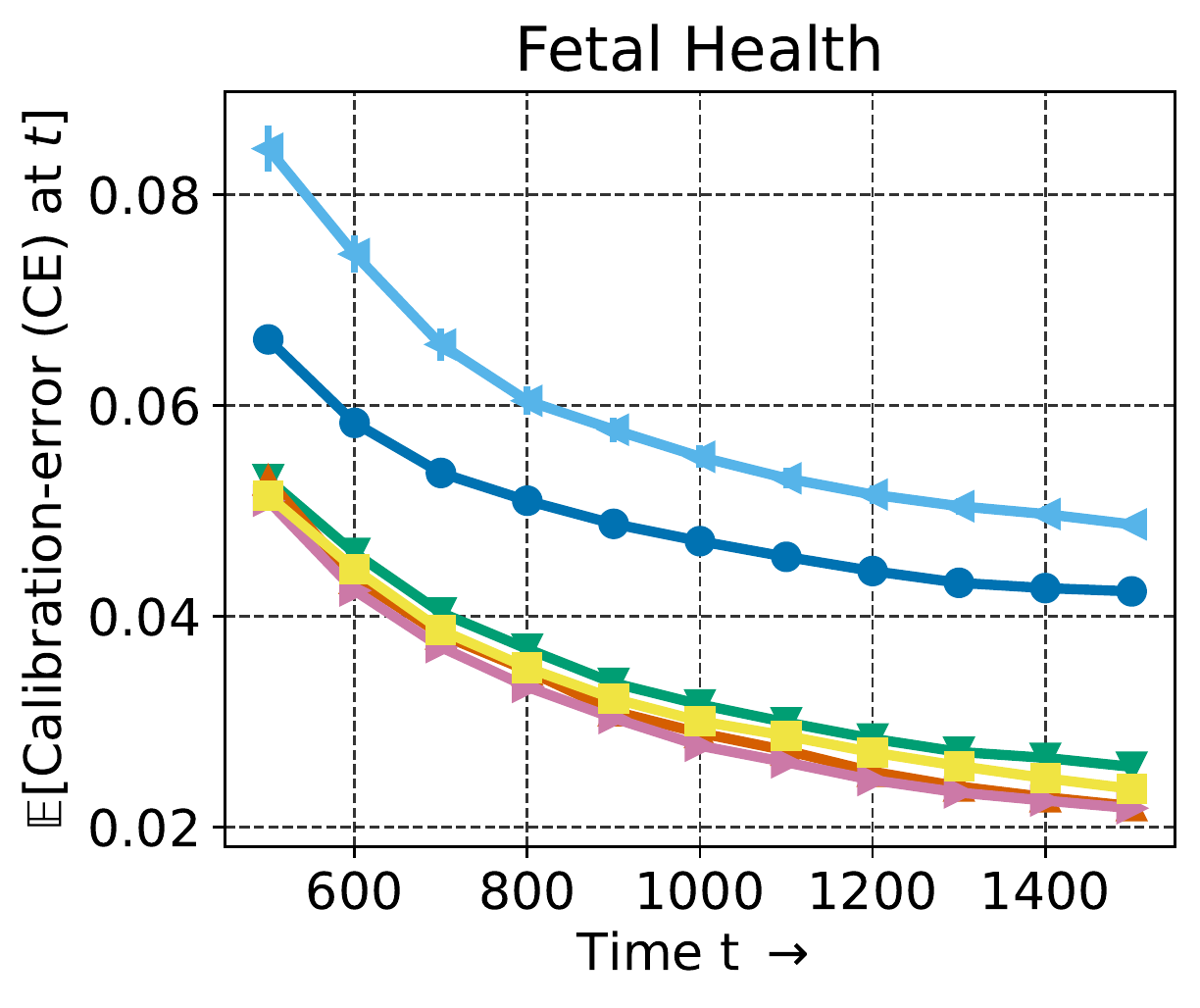}
        \caption{Calibration error for i.i.d.\ data streams.}
    \end{subfigure}
    \vskip .1cm
    \begin{subfigure}{\linewidth}
        \centering
        \includegraphics[trim=0 0 0 0, clip, width=0.24\linewidth]{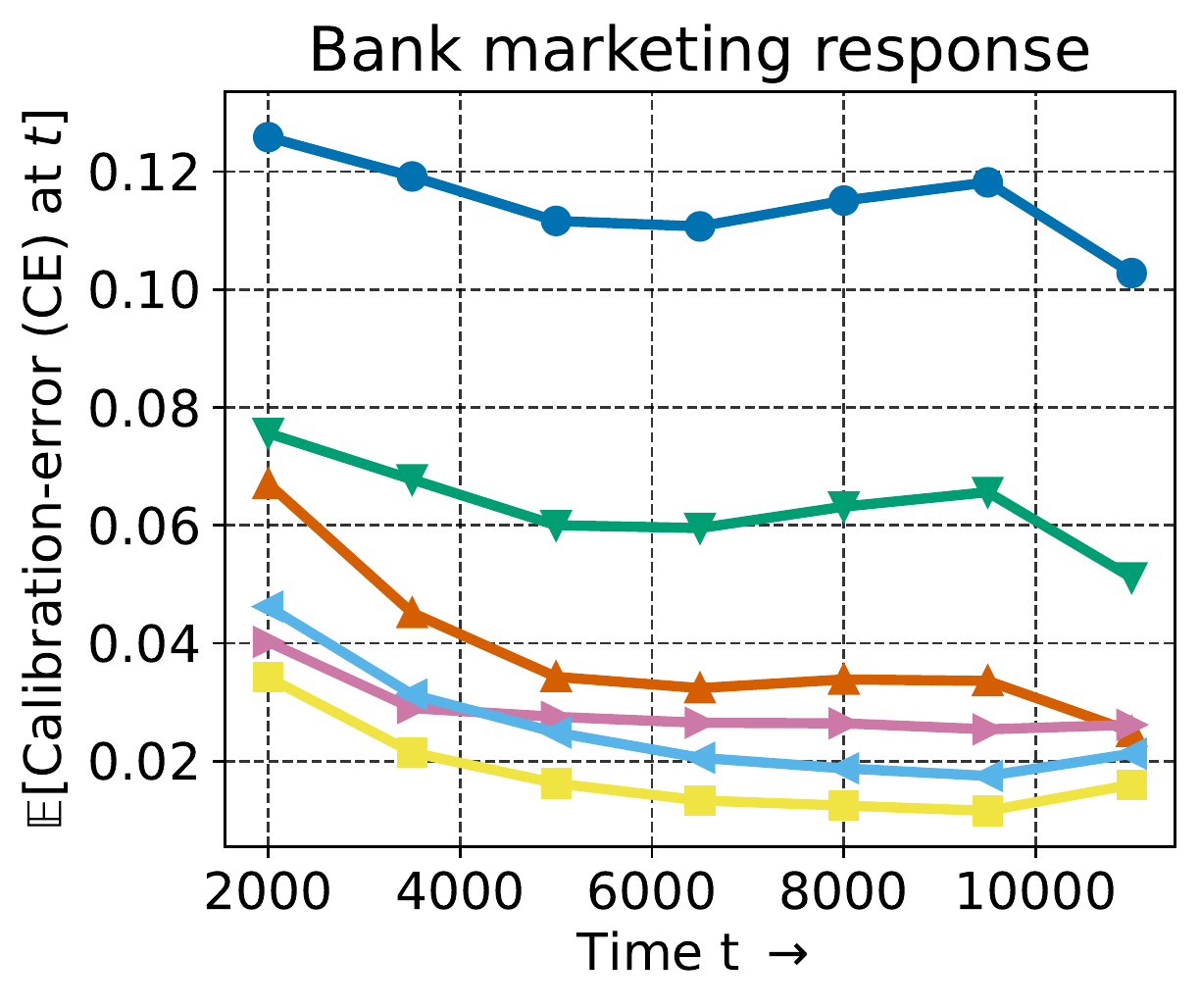}
        \includegraphics[trim=0 0 0 0, clip, width=0.24\linewidth]{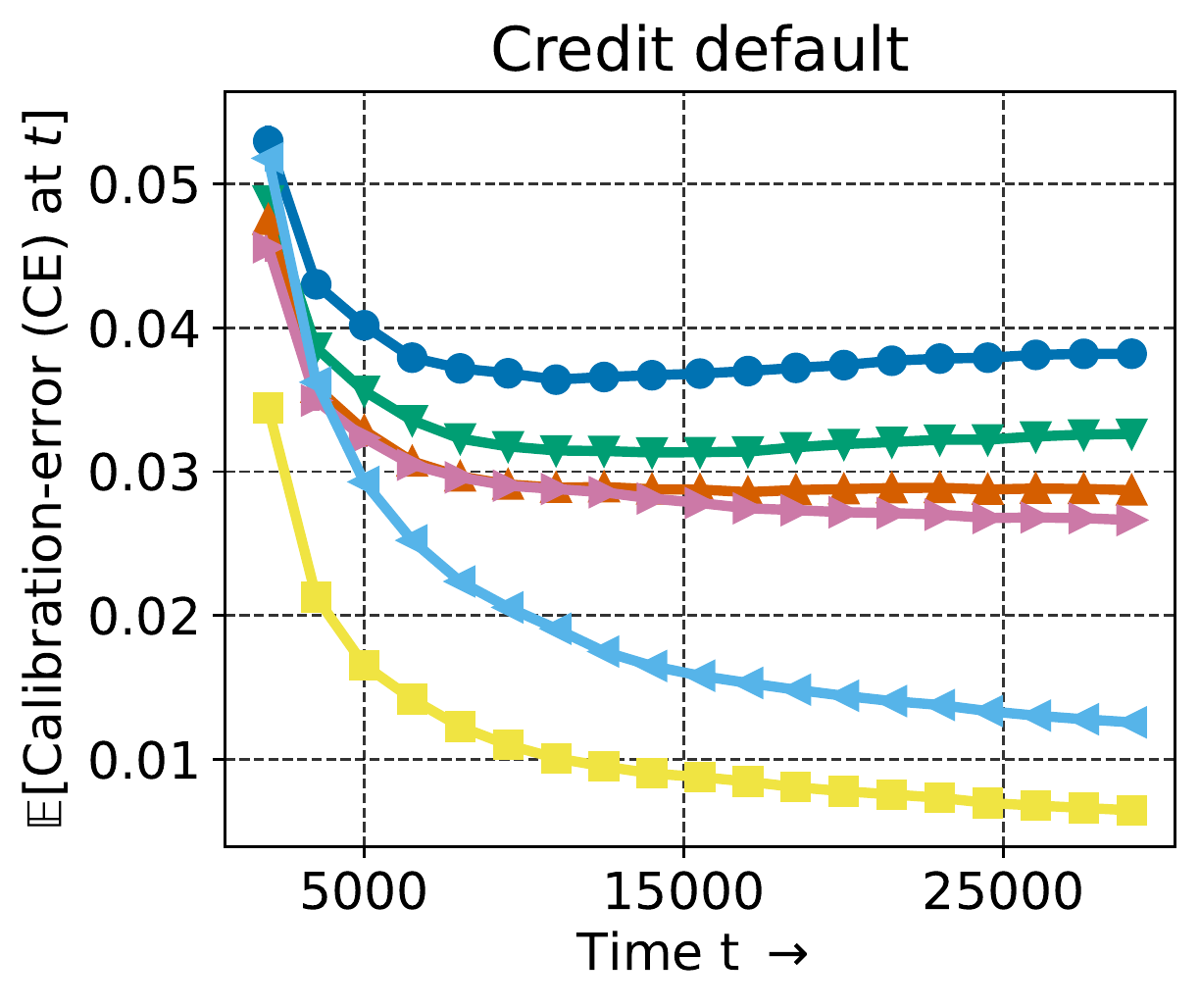}
        \includegraphics[trim=0 0 0 0, clip, width=0.24\linewidth]{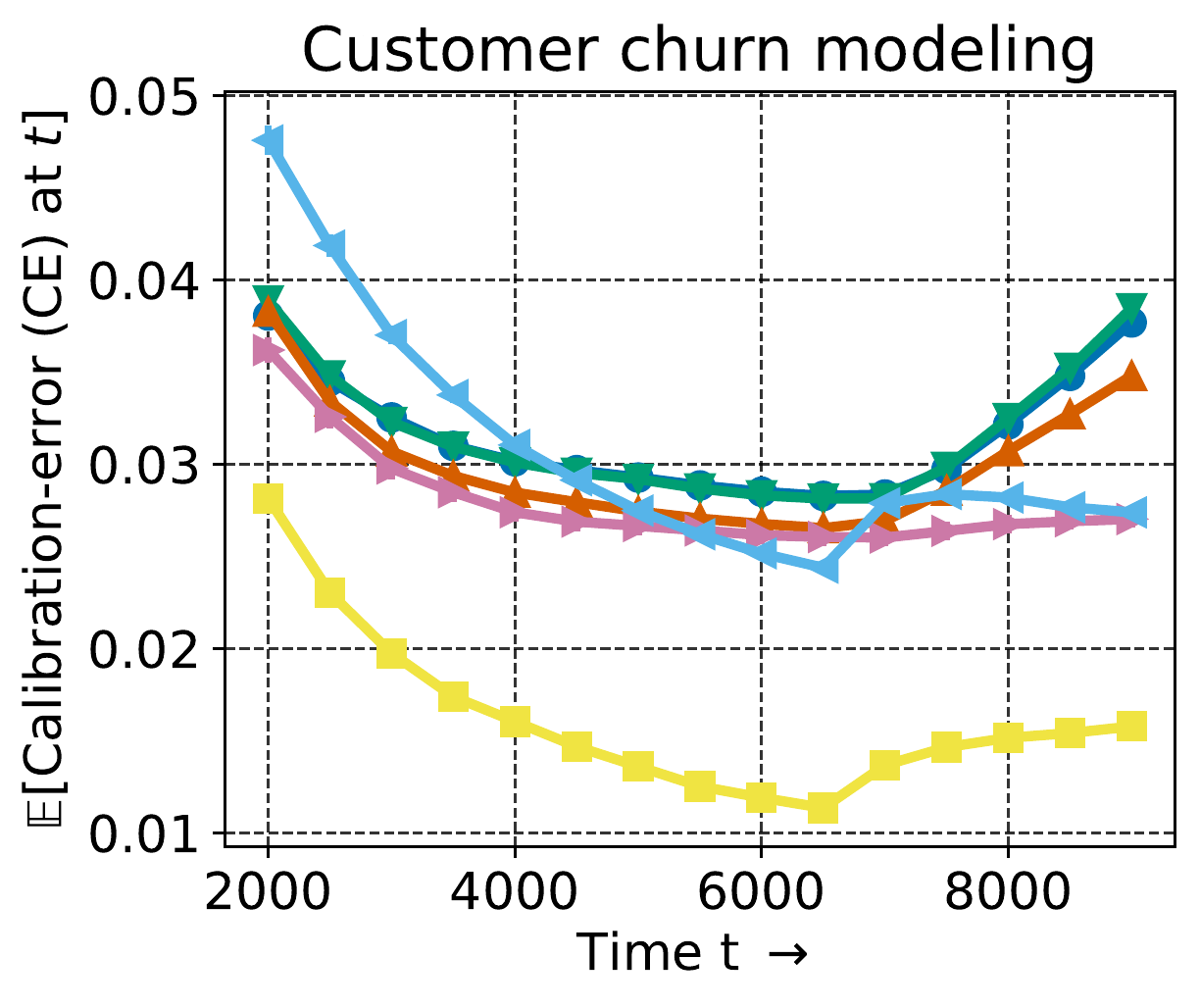}
        \includegraphics[trim=0 0 0 0, clip, width=0.24\linewidth]{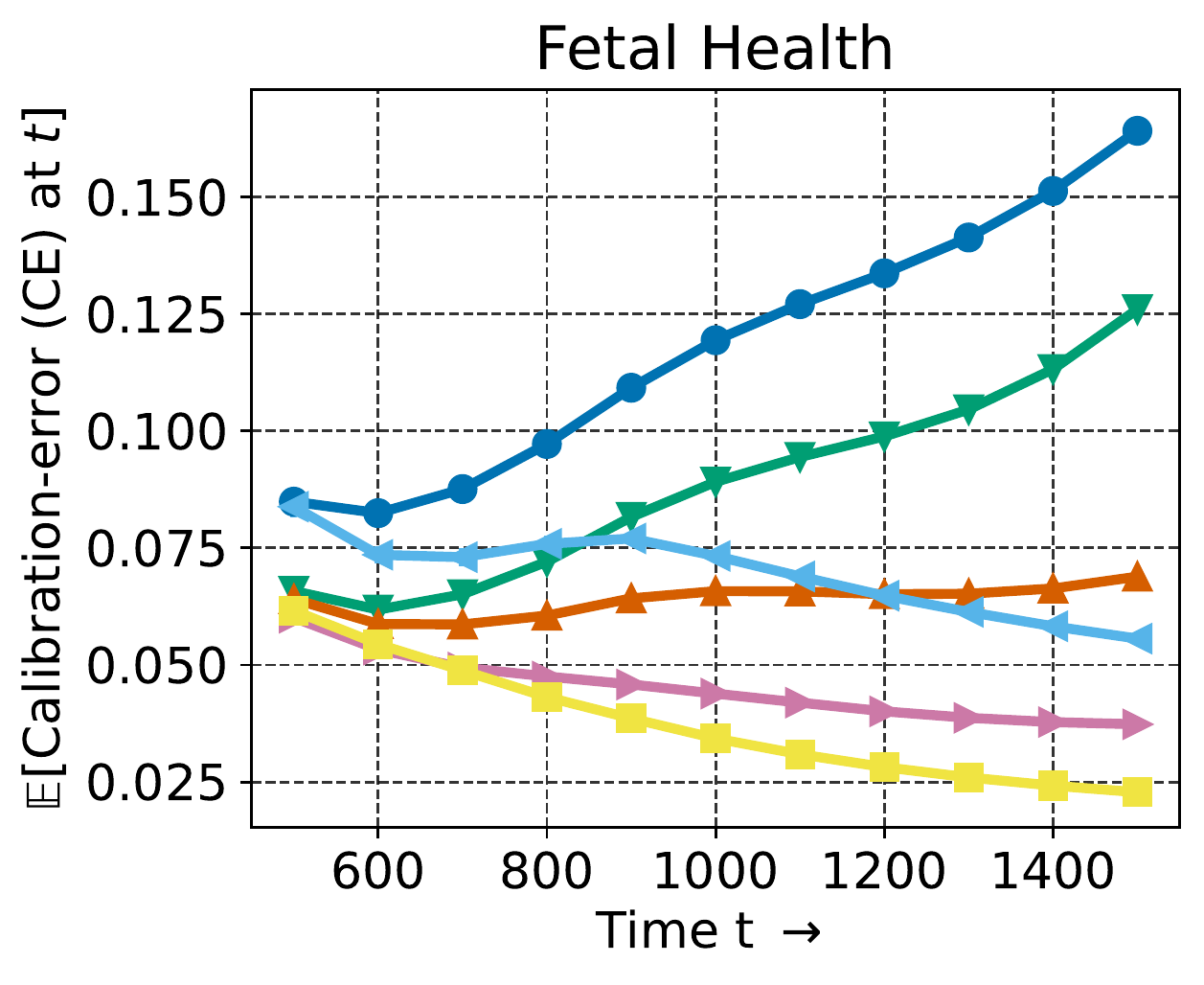}
        \caption{Calibration error for drifting data streams.}
    \end{subfigure}
            \caption{Results for the same experimental setup as Figures~\ref{fig:real-data-ce} and~\ref{fig:real-data-iid-ce}, but with $\epsilon = 0.05$.}
    \label{fig:real-data-eps-0.05}
\end{figure*}

\begin{figure*}[htp]
    \centering
    \begin{subfigure}{\linewidth}
    \centering
    \includegraphics[trim=0 10cm 0 0, clip, width=0.8\linewidth]{figs_comparisons_legend.pdf}
    \end{subfigure}
    \begin{subfigure}{\linewidth}
        \centering
        \includegraphics[trim=0 0 0 0, clip, width=0.24\linewidth]{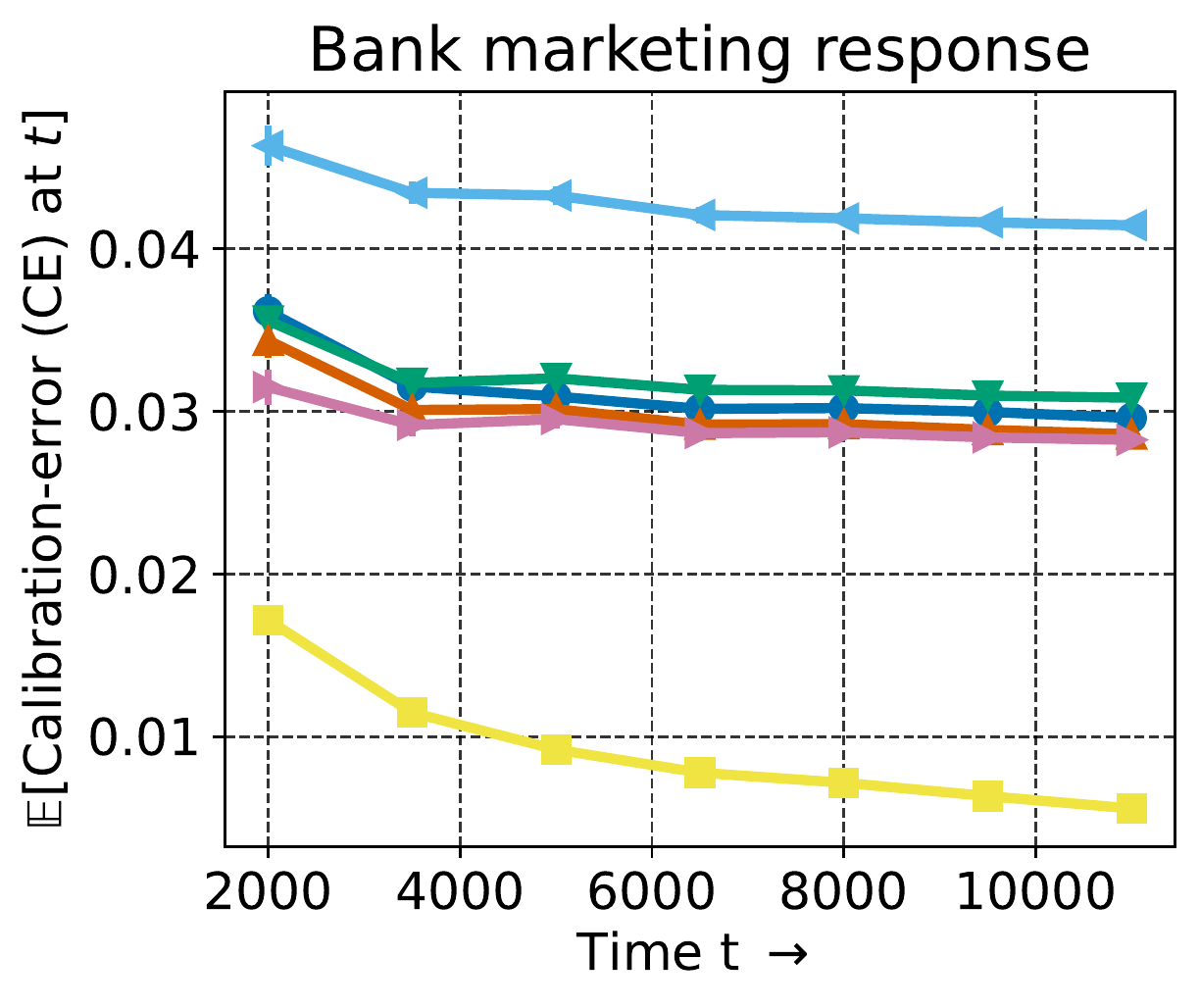}
        \includegraphics[trim=0 0 0 0, clip, width=0.24\linewidth]{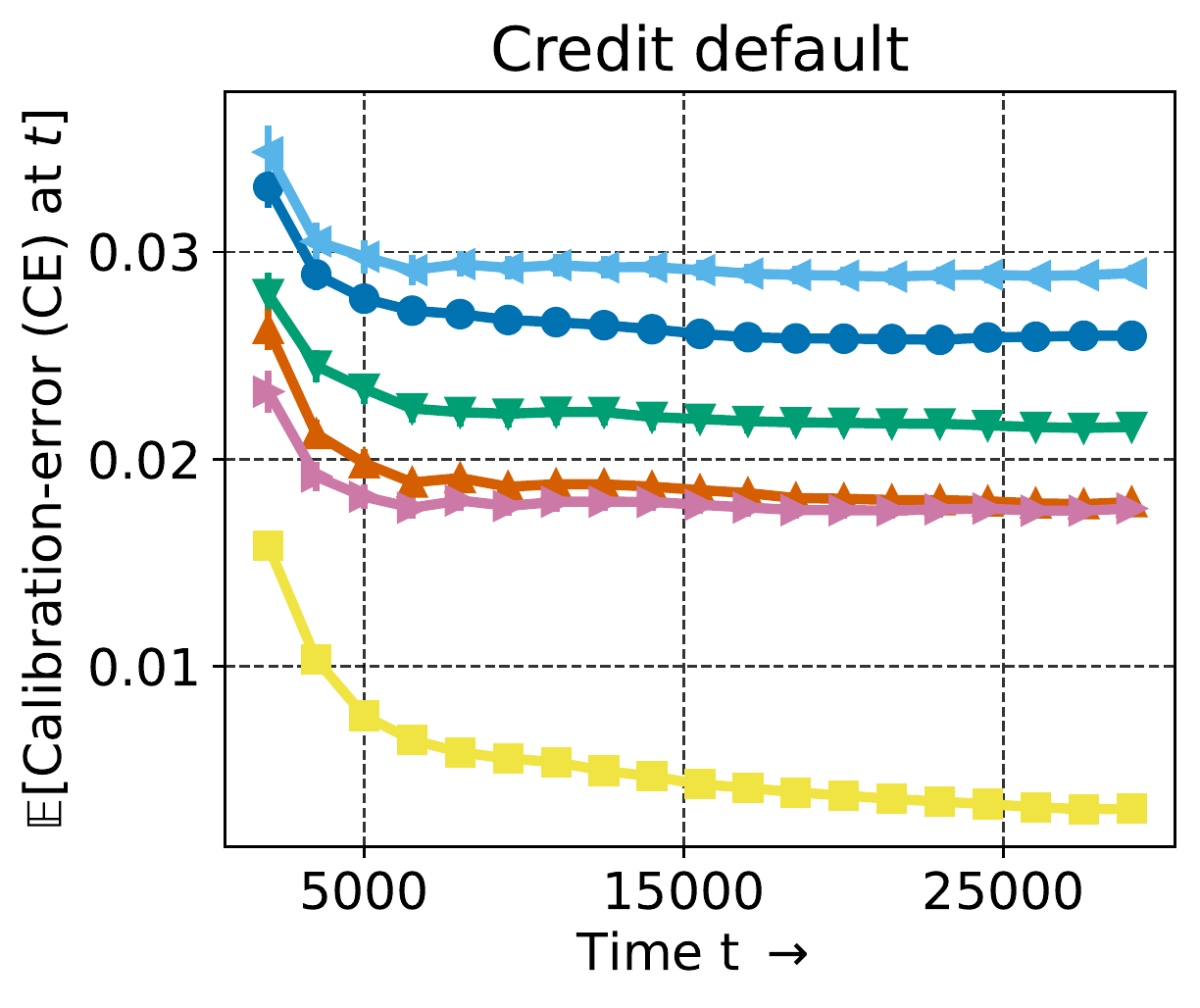}
        \includegraphics[trim=0 0 0 0, clip, width=0.24\linewidth]{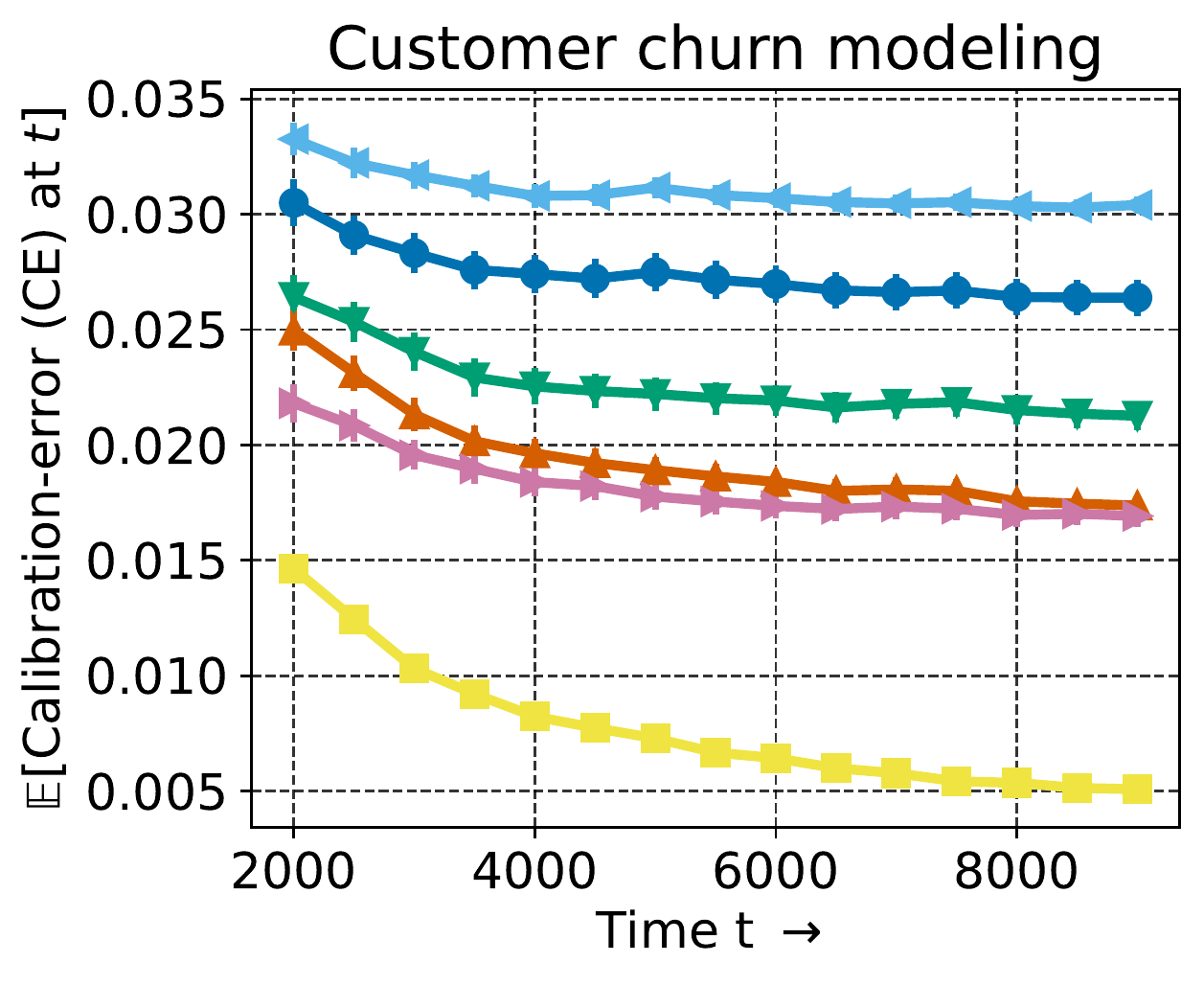}
        \includegraphics[trim=0 0 0 0, clip, width=0.24\linewidth]{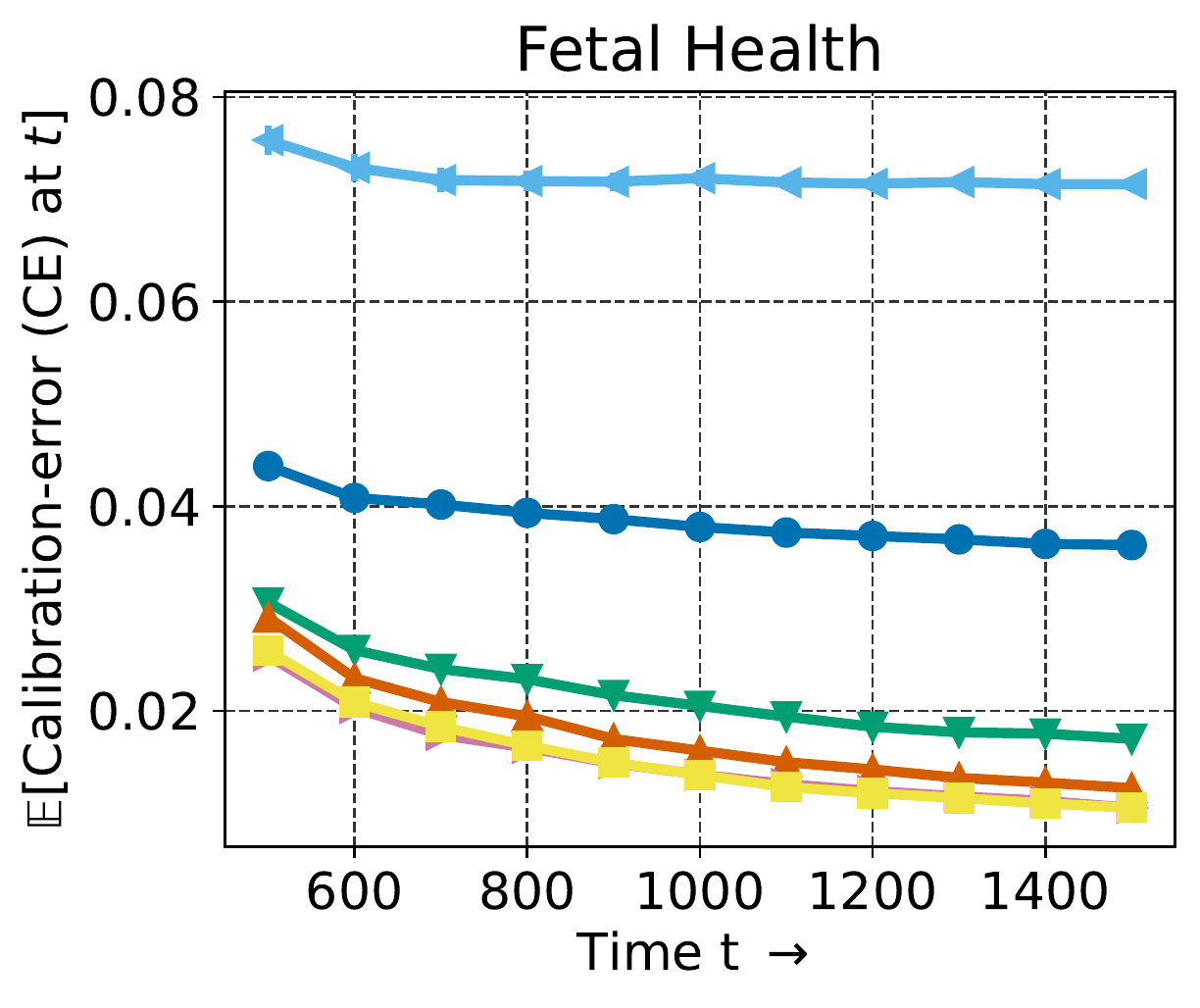}
        \caption{Calibration error for i.i.d.\ data streams.}
    \end{subfigure}
    \vskip .1cm
    \begin{subfigure}{\linewidth}
        \centering
        \includegraphics[trim=0 0 0 0, clip, width=0.24\linewidth]{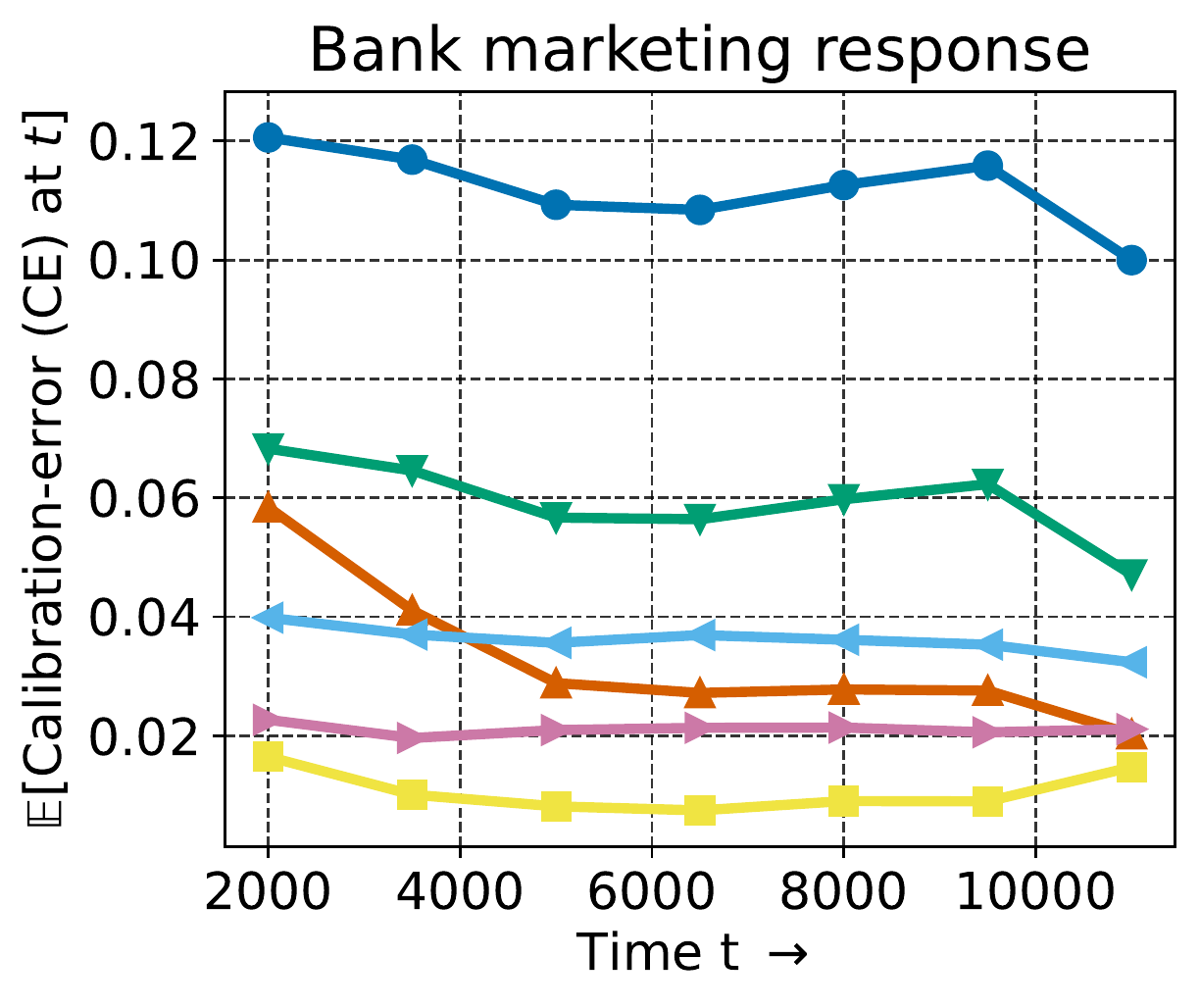}
        \includegraphics[trim=0 0 0 0, clip, width=0.24\linewidth]{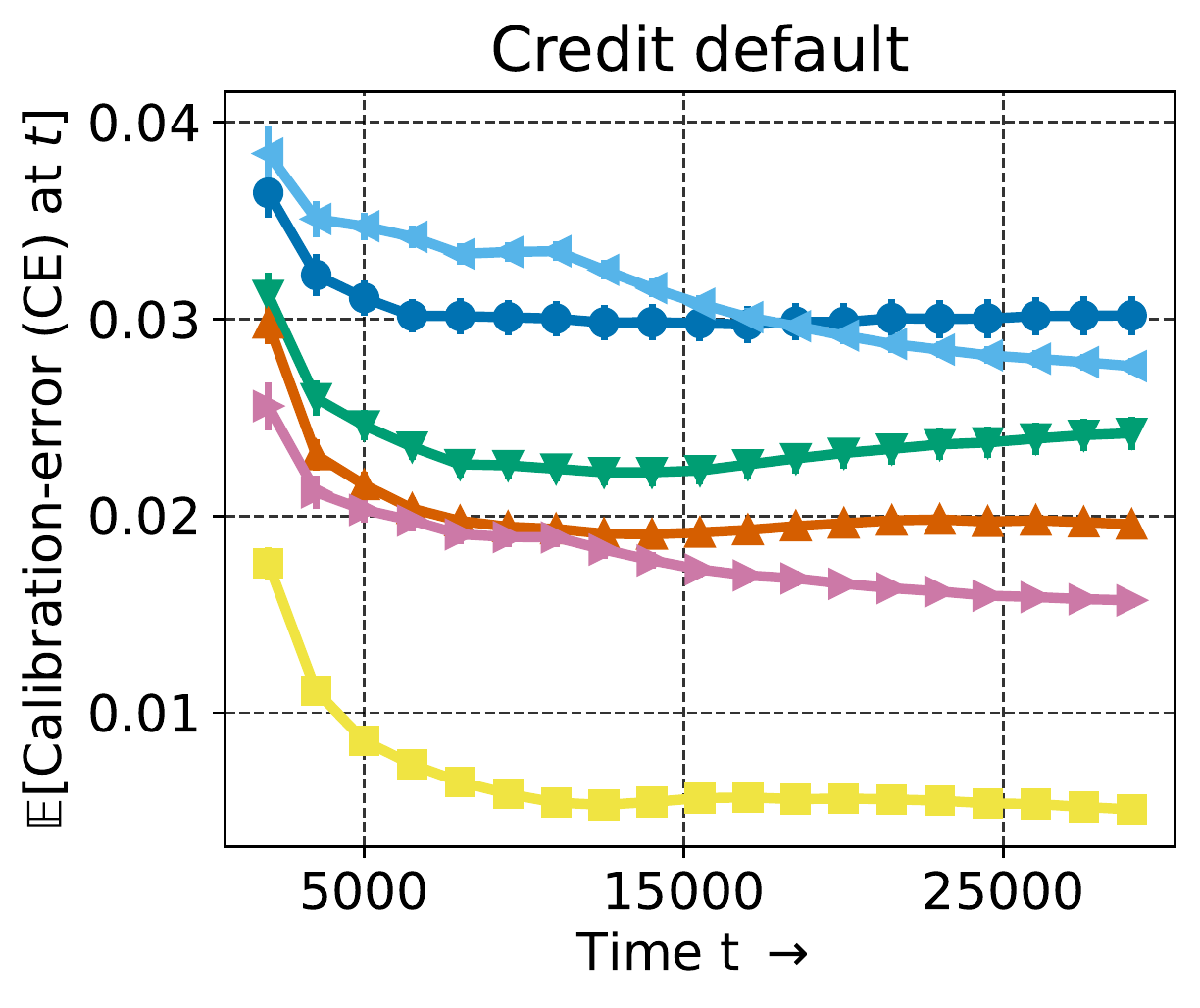}
        \includegraphics[trim=0 0 0 0, clip, width=0.24\linewidth]{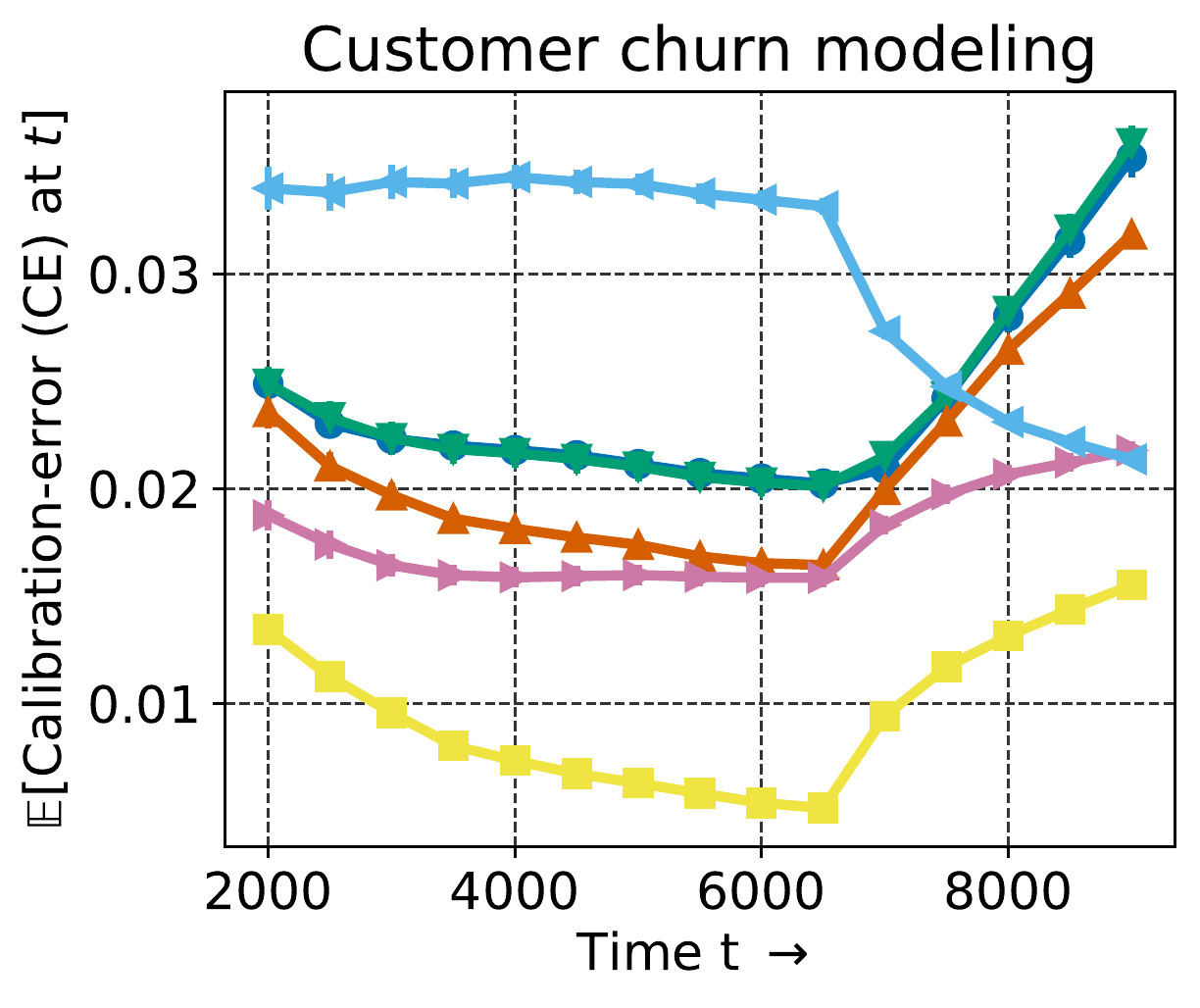}
        \includegraphics[trim=0 0 0 0, clip, width=0.24\linewidth]{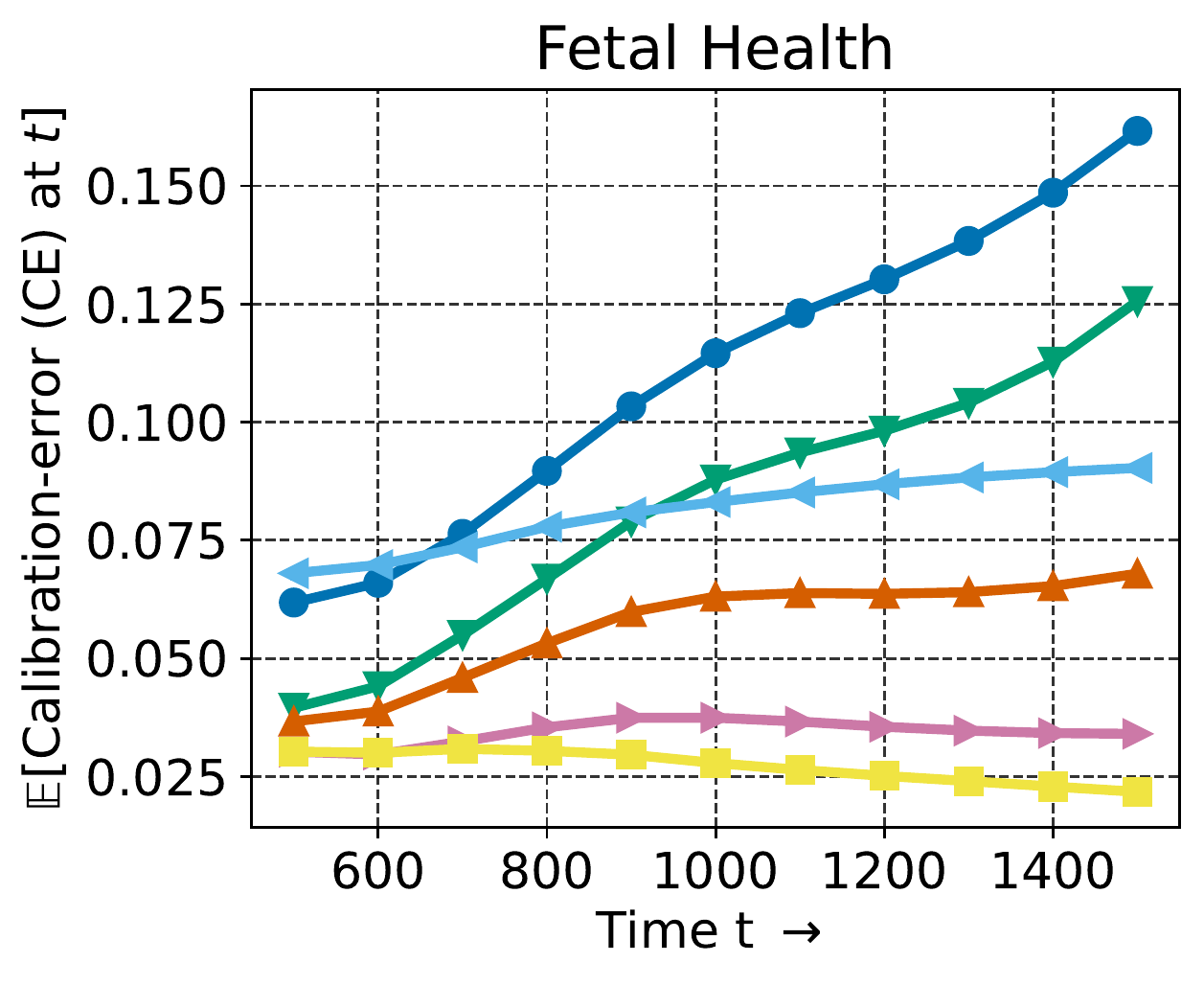}
        \caption{Calibration error for drifting data streams.}
    \end{subfigure}
            \caption{Results for the same experimental setup as Figures~\ref{fig:real-data-ce} and~\ref{fig:real-data-iid-ce}, but with $\epsilon = 0.2$.}
    \label{fig:real-data-eps-0.2}
\end{figure*}

\begin{figure*}[t]
    \centering
    \begin{subfigure}{\linewidth}
    \centering
    \includegraphics[trim=0 10cm 0 0, clip, width=0.8\linewidth]{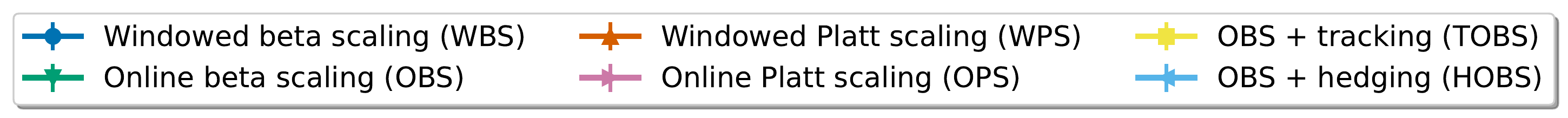}
    \end{subfigure}
    \begin{subfigure}{\linewidth}
        \centering
        \includegraphics[trim=0 0 0 0, clip, width=0.24\linewidth]{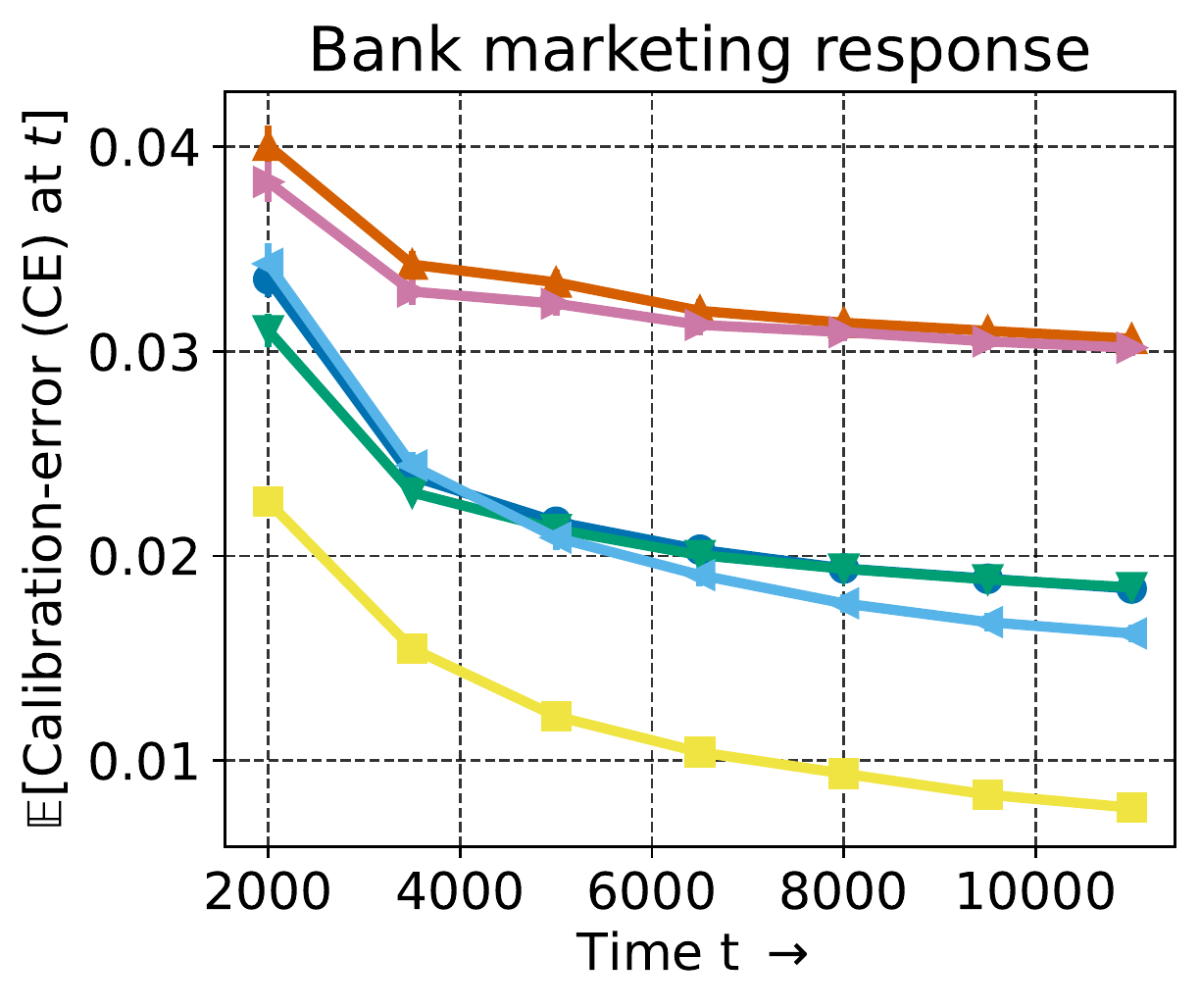}
        \includegraphics[trim=0 0 0 0, clip, width=0.24\linewidth]{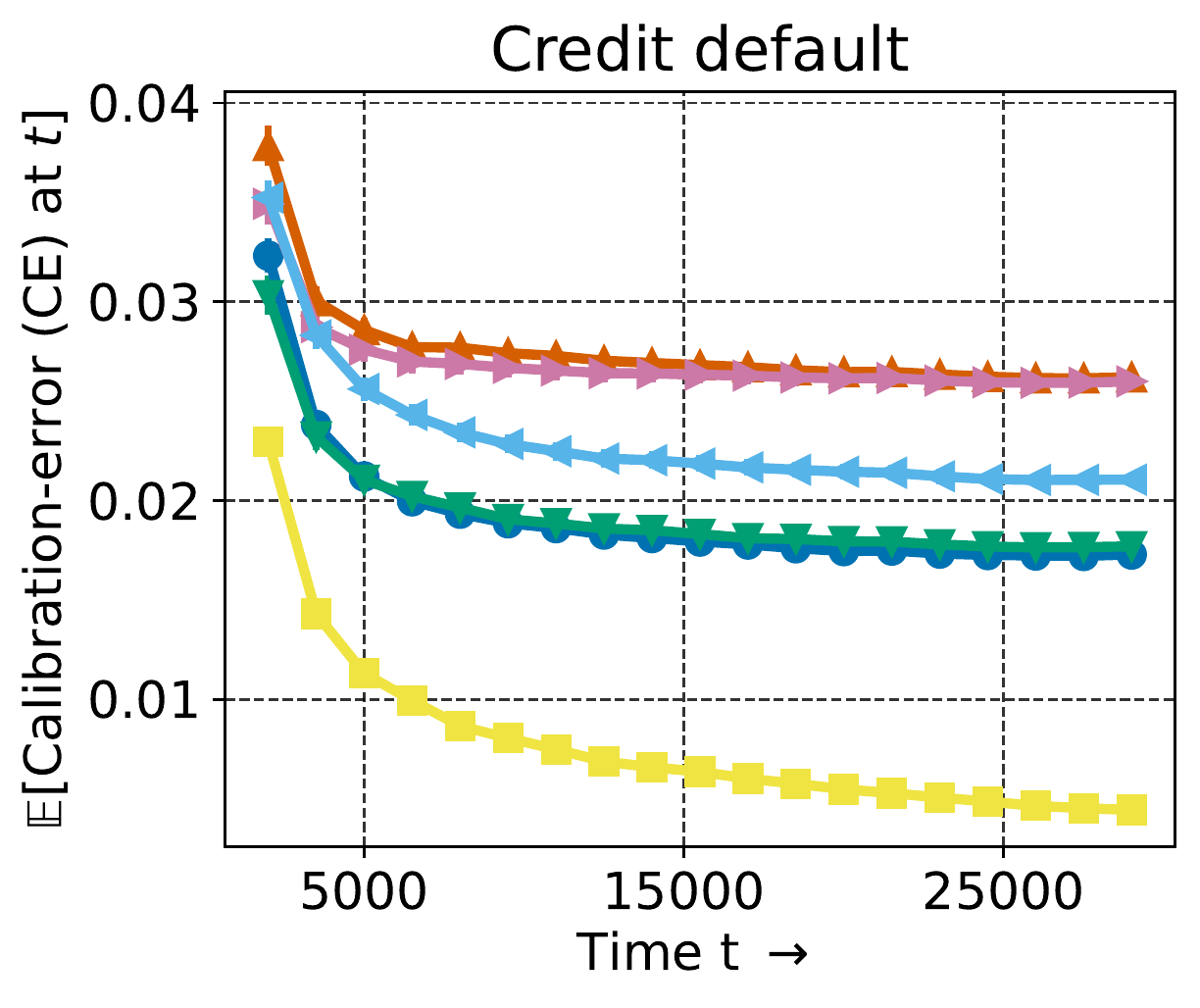}
        \includegraphics[trim=0 0 0 0, clip, width=0.24\linewidth]{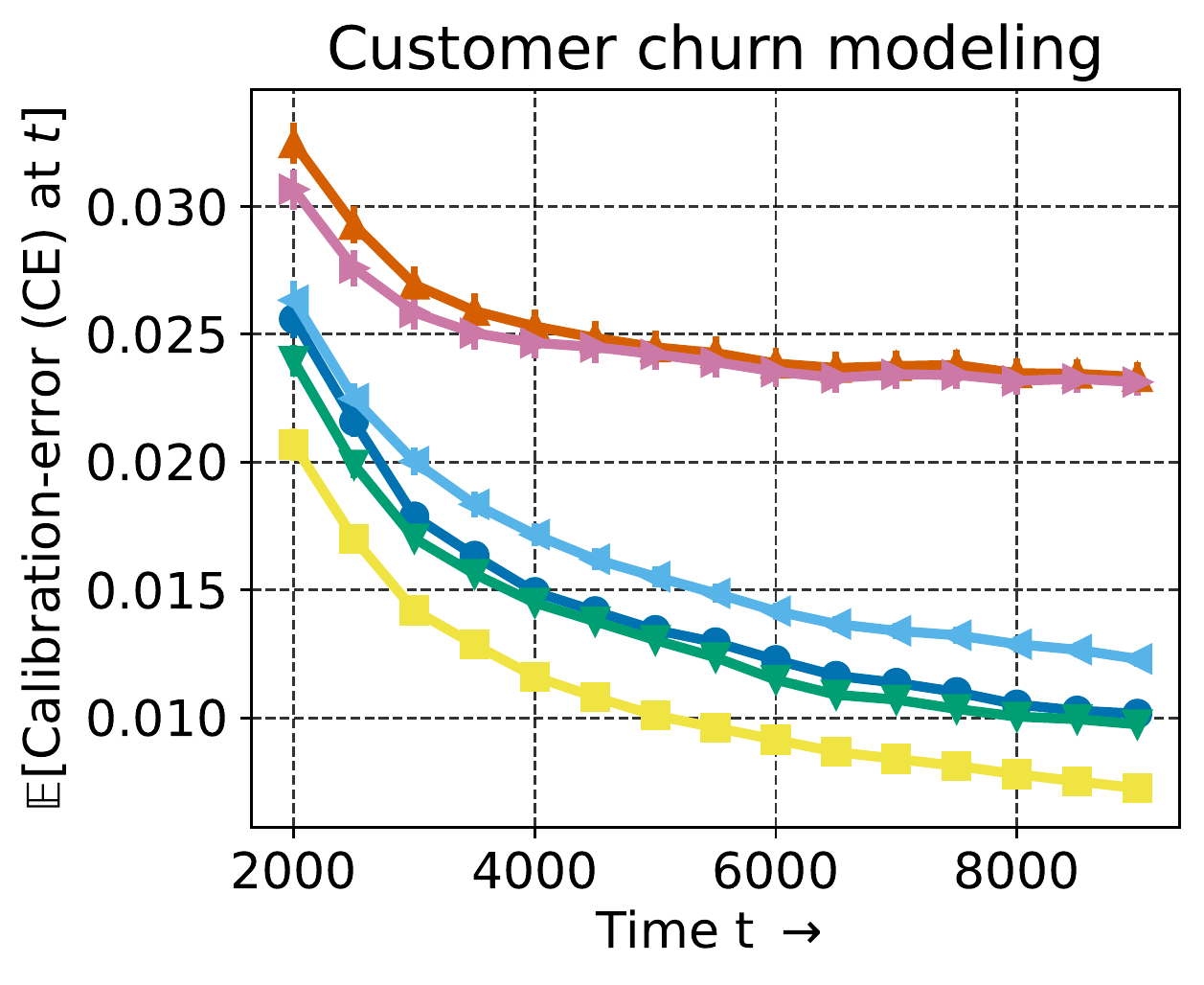}
        \includegraphics[trim=0 0 0 0, clip, width=0.24\linewidth]{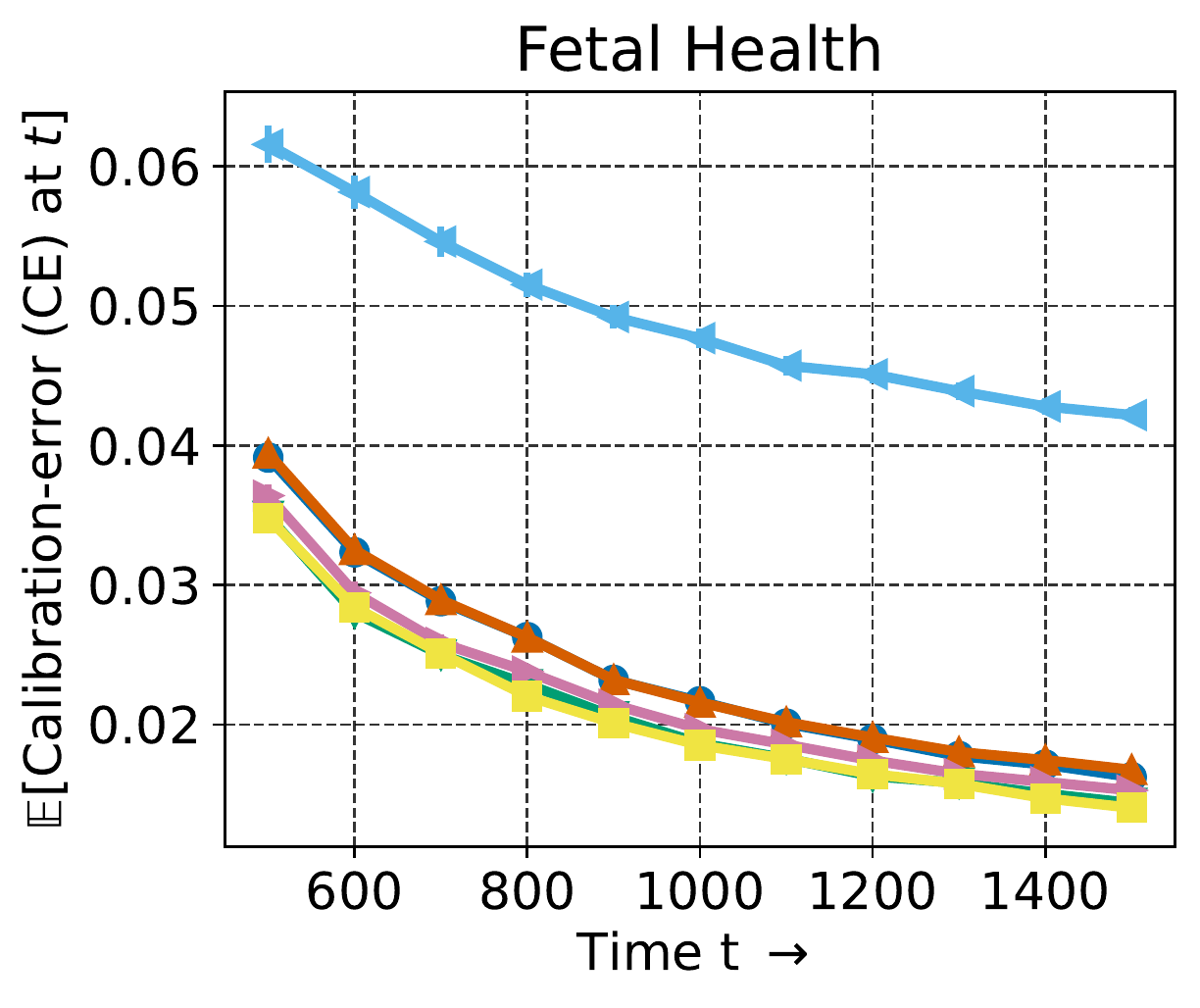}
        \caption{Calibration error for i.i.d.\ data streams.}
    \end{subfigure}
    \vskip .1cm
    \begin{subfigure}{\linewidth}
        \centering
        \includegraphics[trim=0 0 0 0, clip, width=0.24\linewidth]{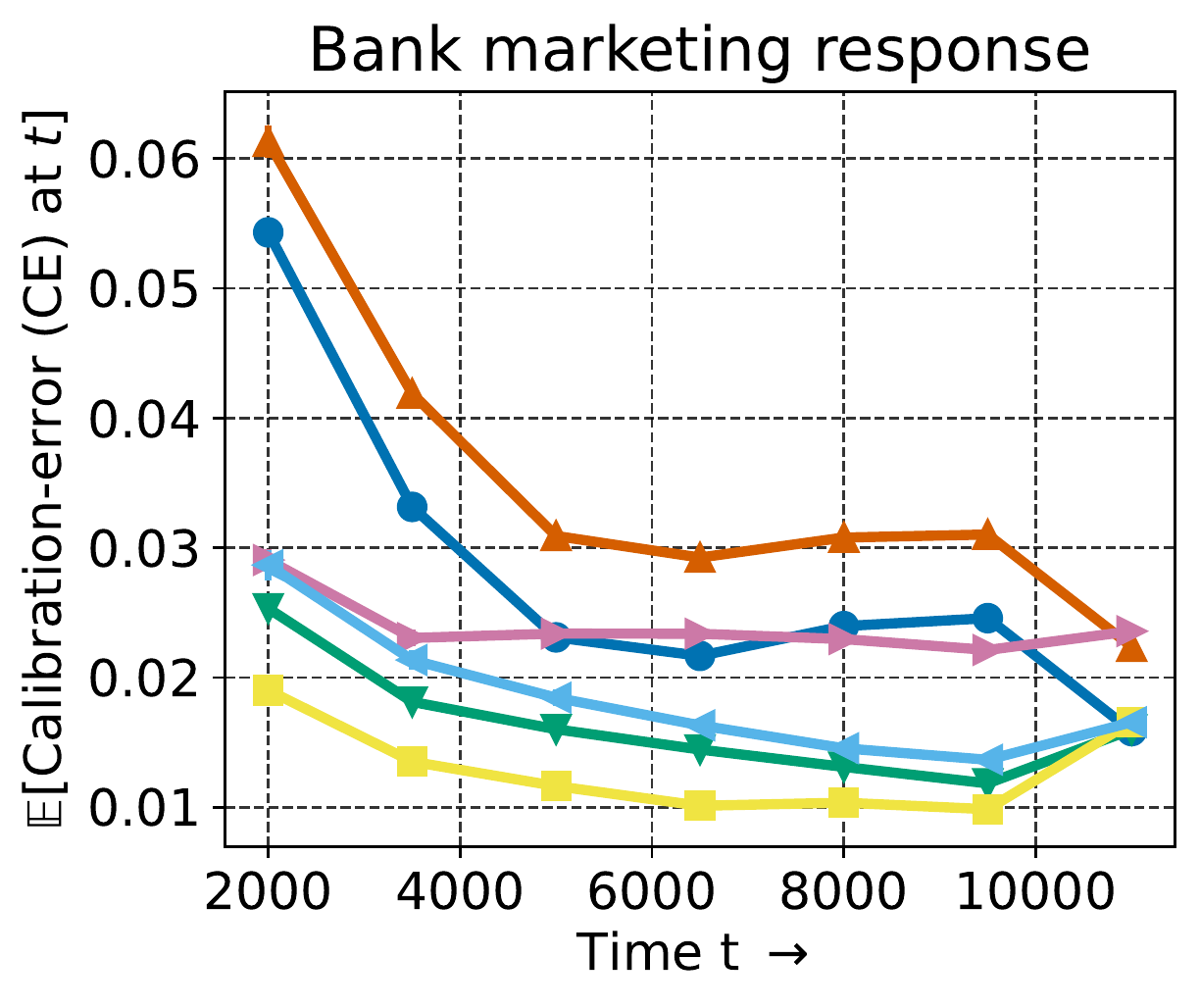}
        \includegraphics[trim=0 0 0 0, clip, width=0.24\linewidth]{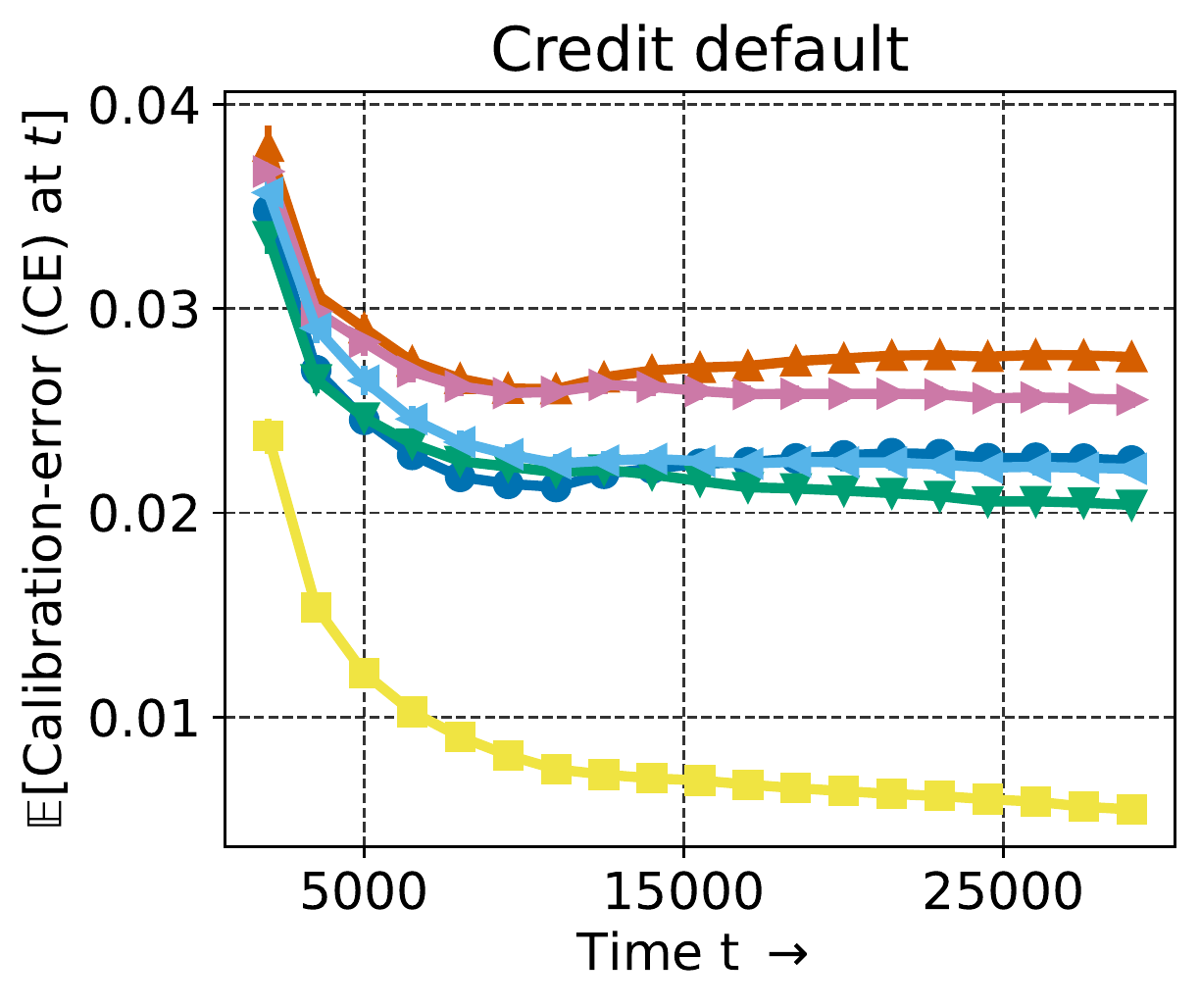}
        \includegraphics[trim=0 0 0 0, clip, width=0.24\linewidth]{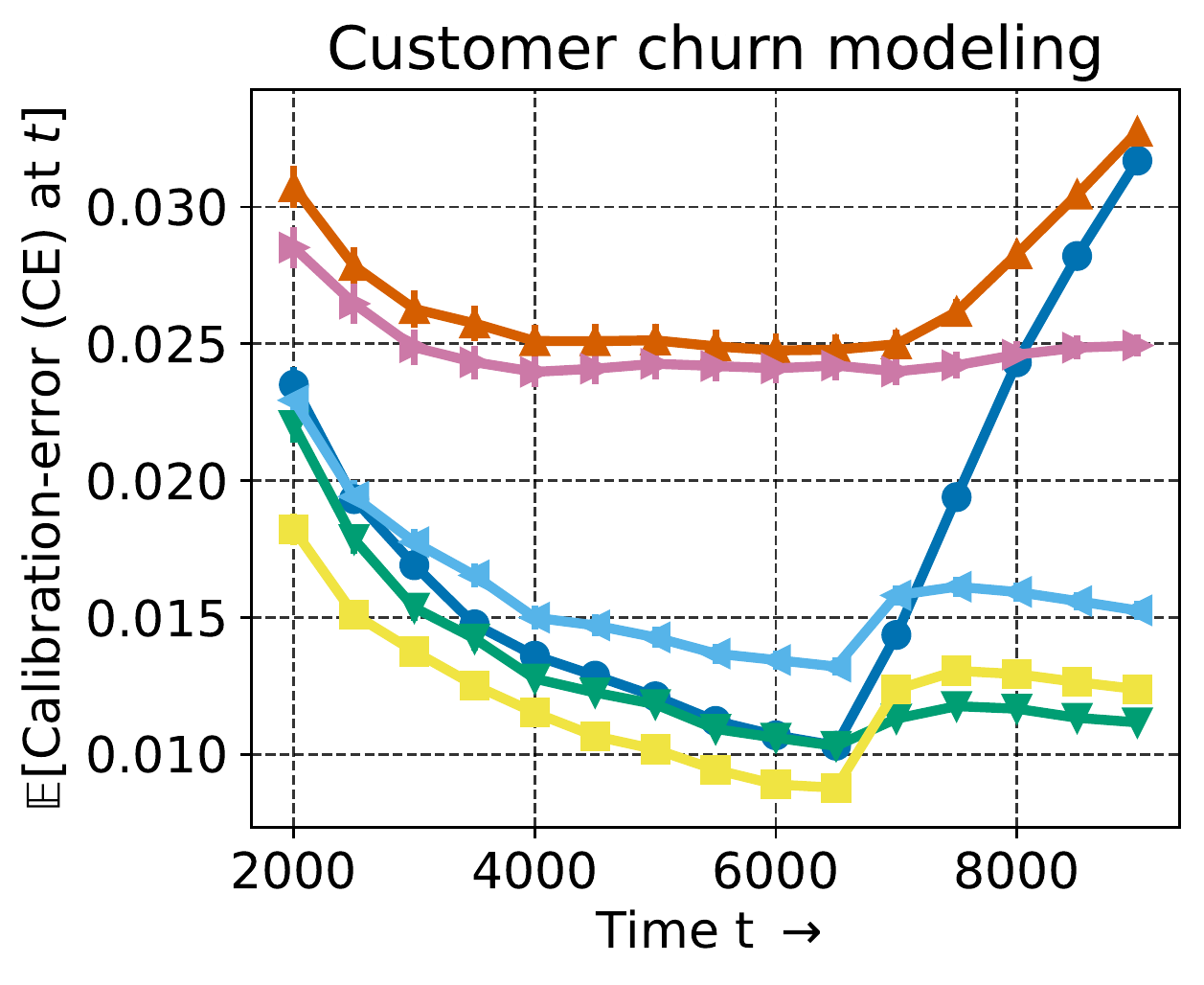}
        \includegraphics[trim=0 0 0 0, clip, width=0.24\linewidth]{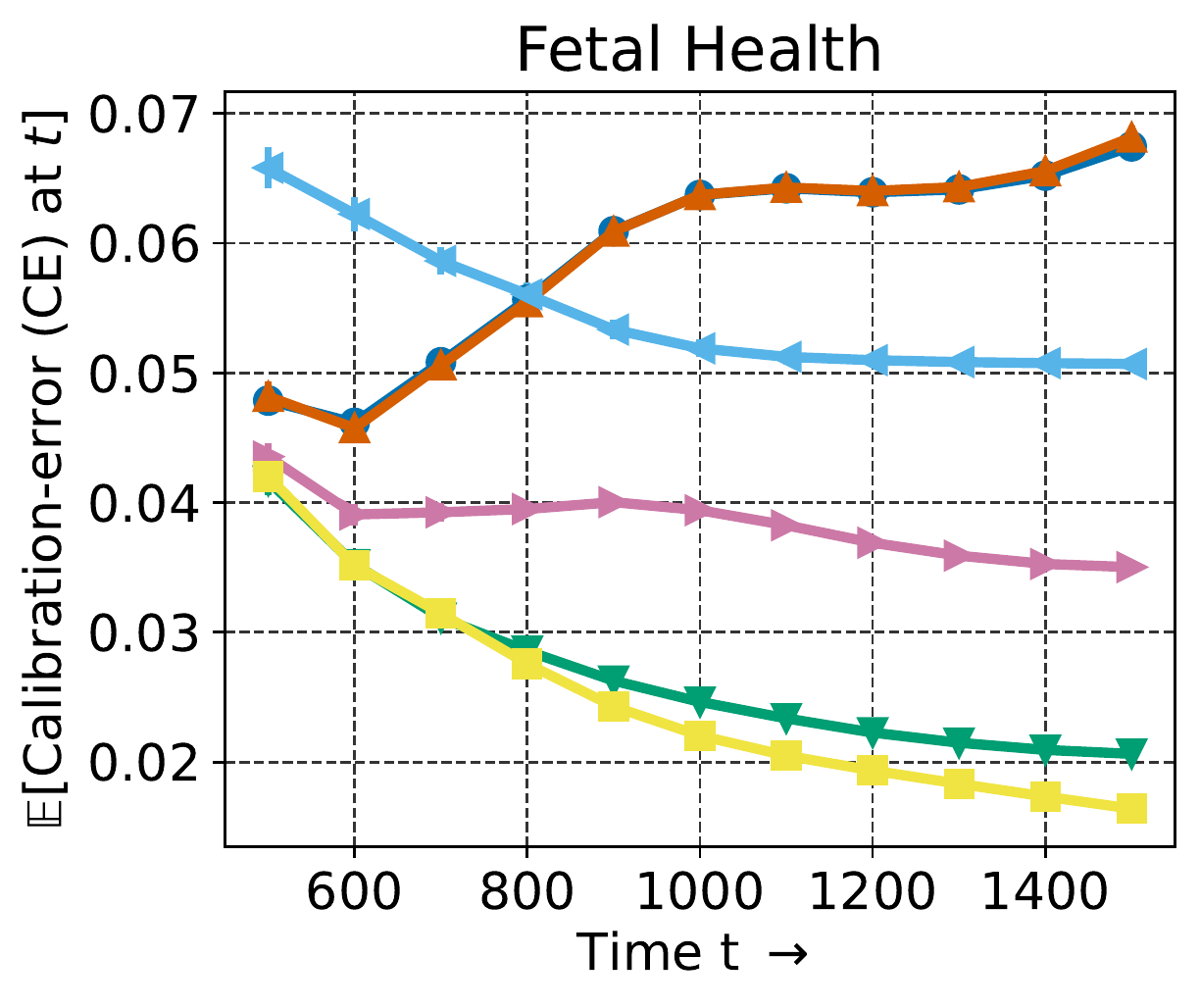}
        \caption{Calibration error for drifting data streams.}
    \end{subfigure}
            \caption{Performance of online beta scaling (OBS) and its calibeating variants on real datasets with and without distribution drift. OBS further improves upon OPS in most cases. In each plot, TOBS is the best-performing method.}
    \label{fig:real-data-tops}
\end{figure*}

\begin{figure*}[htp]
    \centering
    \begin{subfigure}{\linewidth}
    \centering
    \includegraphics[trim=0 10cm 0 0, clip, width=0.8\linewidth]{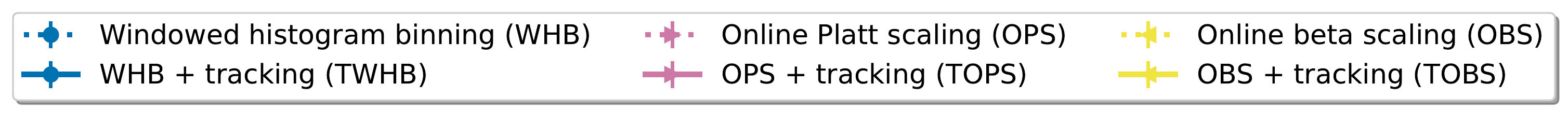}
    \end{subfigure}
    \begin{subfigure}{\linewidth}
        \centering
        \includegraphics[trim=0 0 0 0, clip, width=0.24\linewidth]{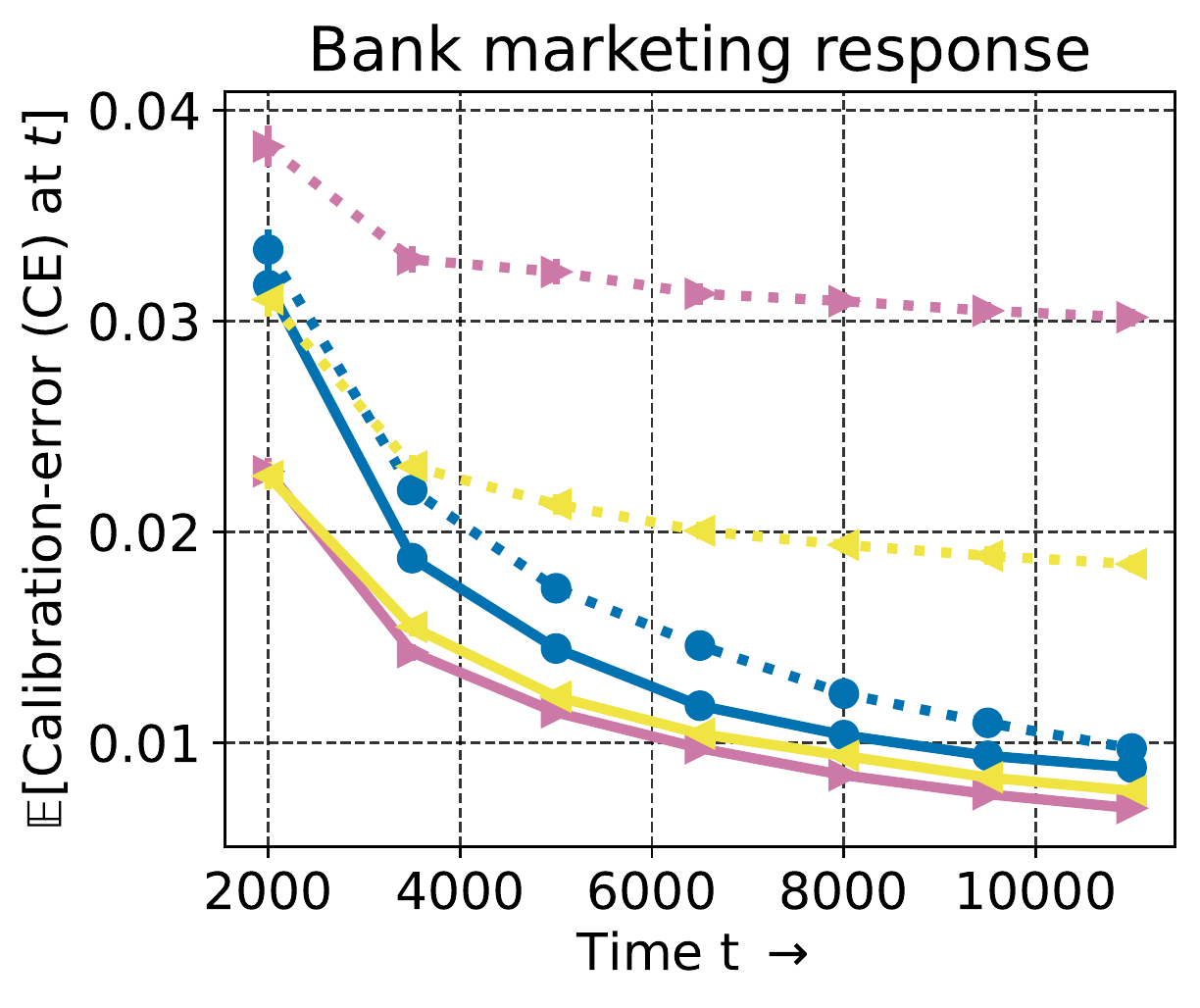}
        \includegraphics[trim=0 0 0 0, clip, width=0.24\linewidth]{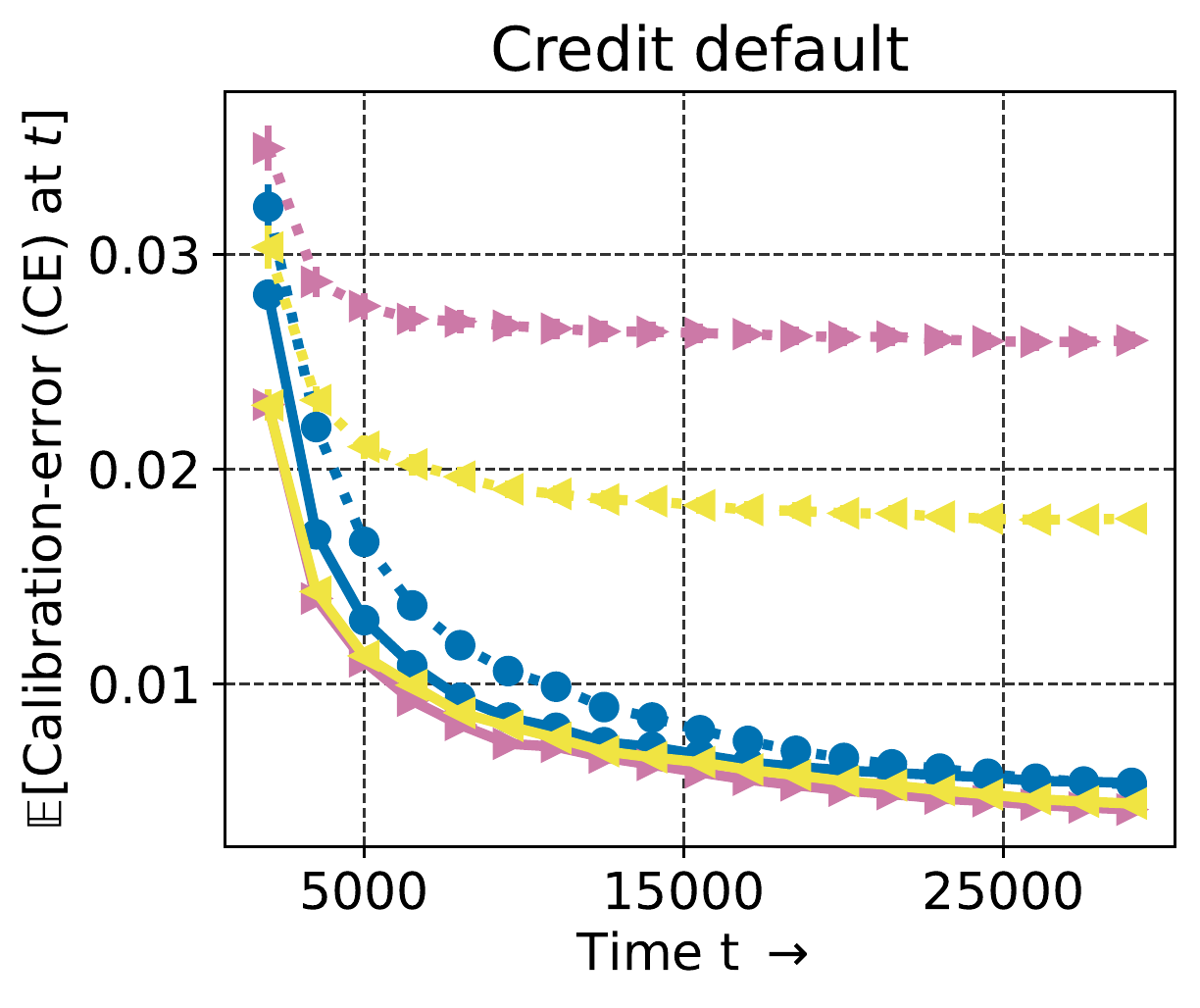}
        \includegraphics[trim=0 0 0 0, clip, width=0.24\linewidth]{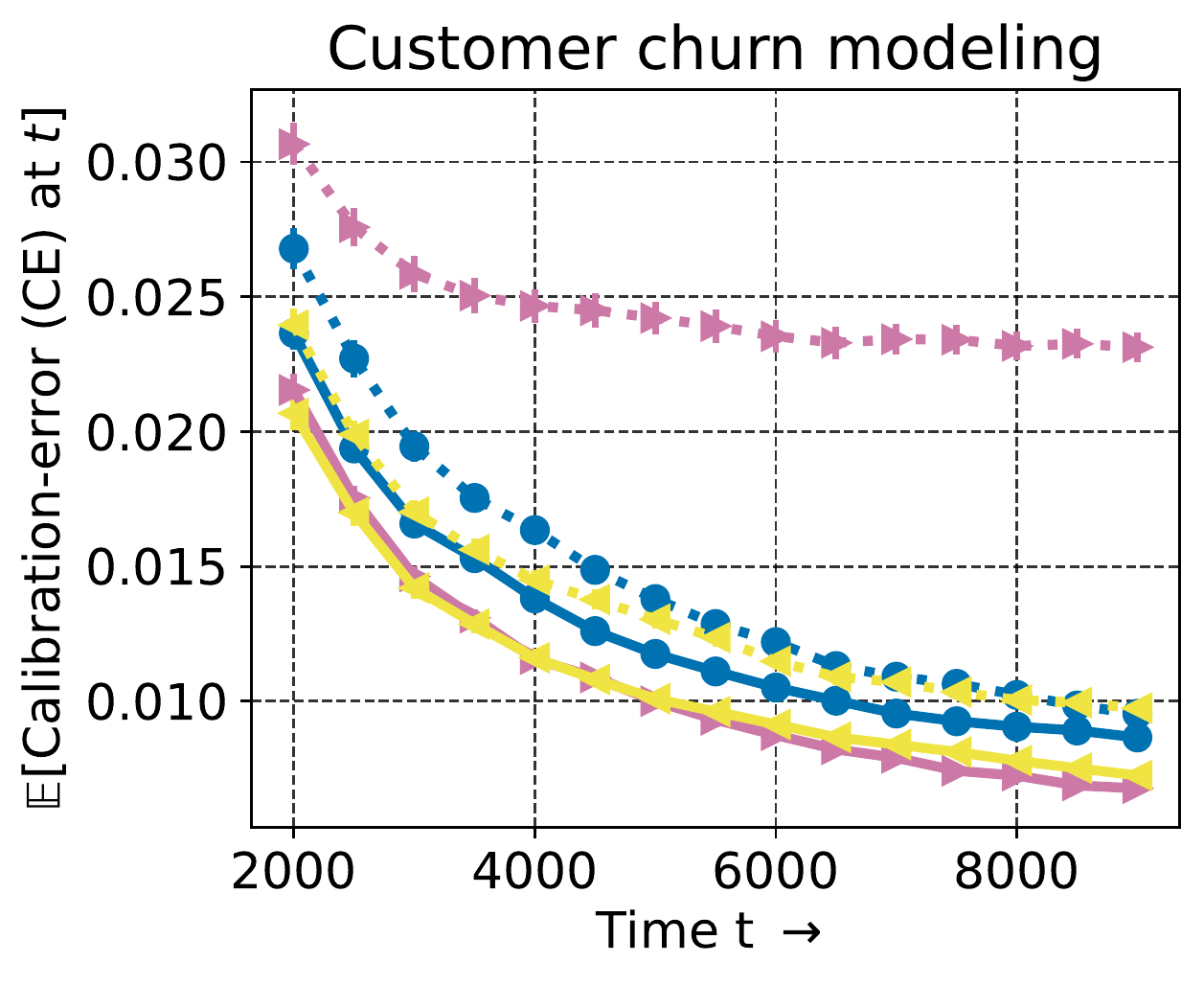}
        \includegraphics[trim=0 0 0 0, clip, width=0.24\linewidth]{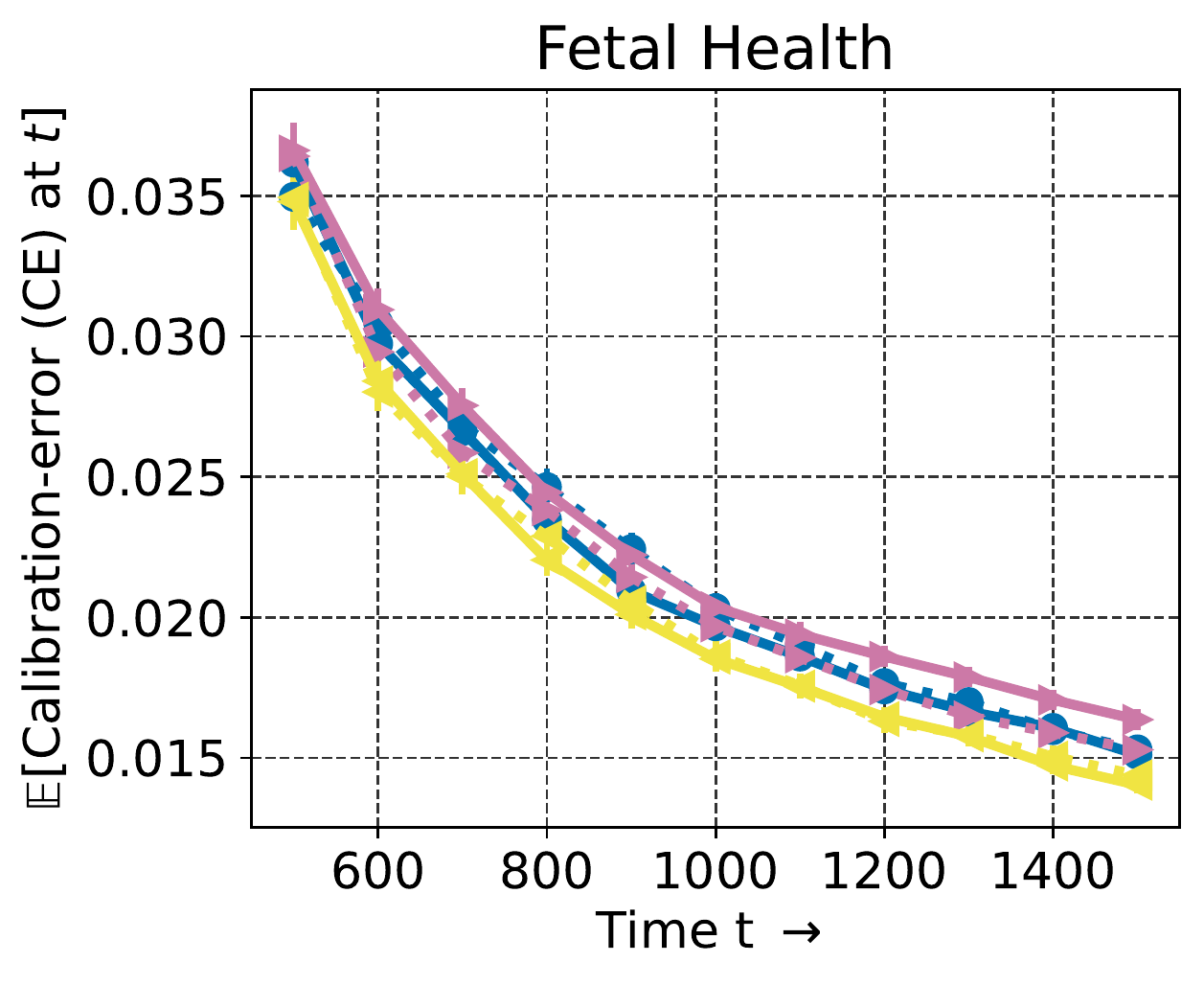}
        \caption{Calibration error for i.i.d.\ data streams.}
    \end{subfigure}
    \vskip .1cm
    \begin{subfigure}{\linewidth}
        \centering
        \includegraphics[trim=0 0 0 0, clip, width=0.24\linewidth]{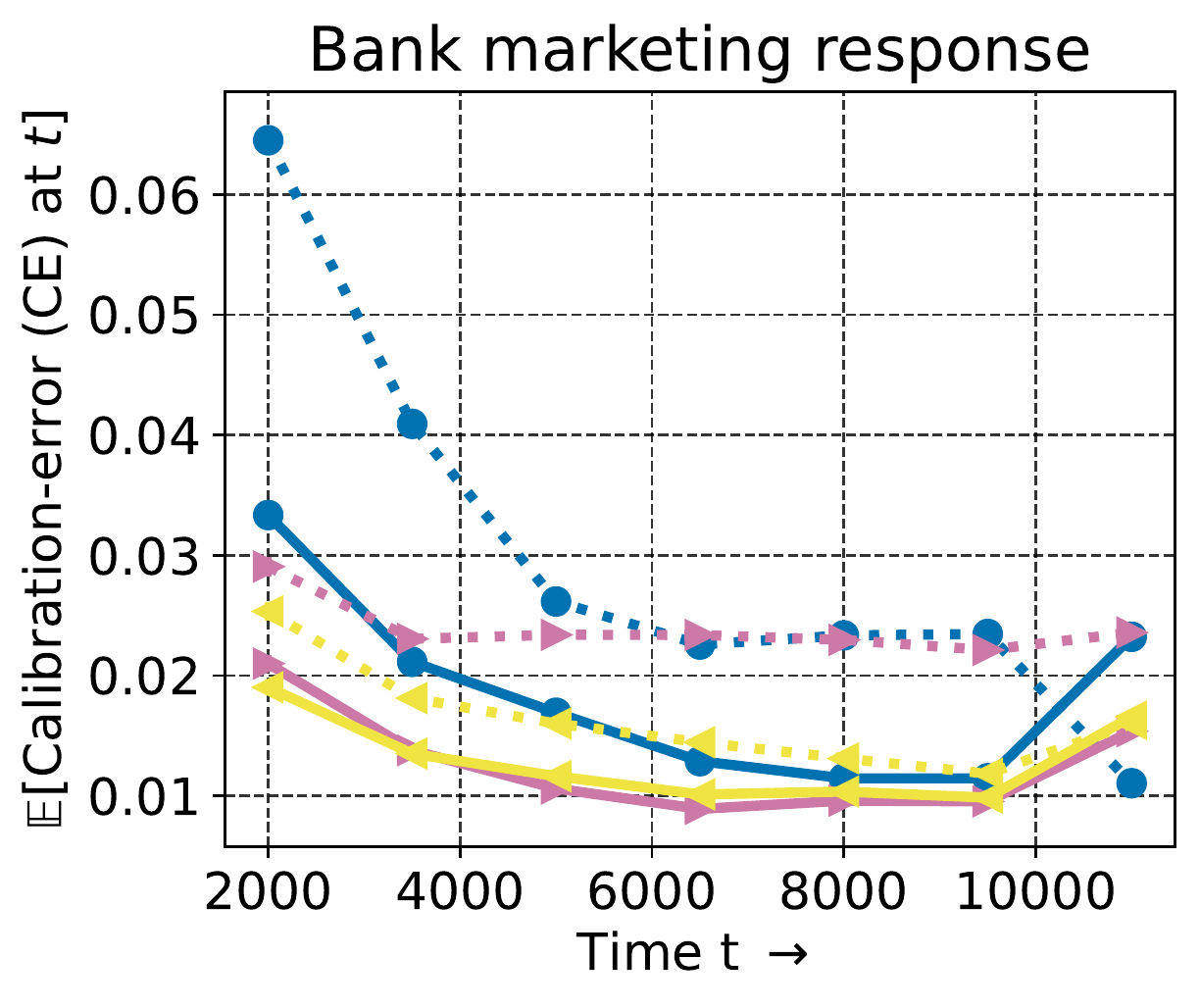}
        \includegraphics[trim=0 0 0 0, clip, width=0.24\linewidth]{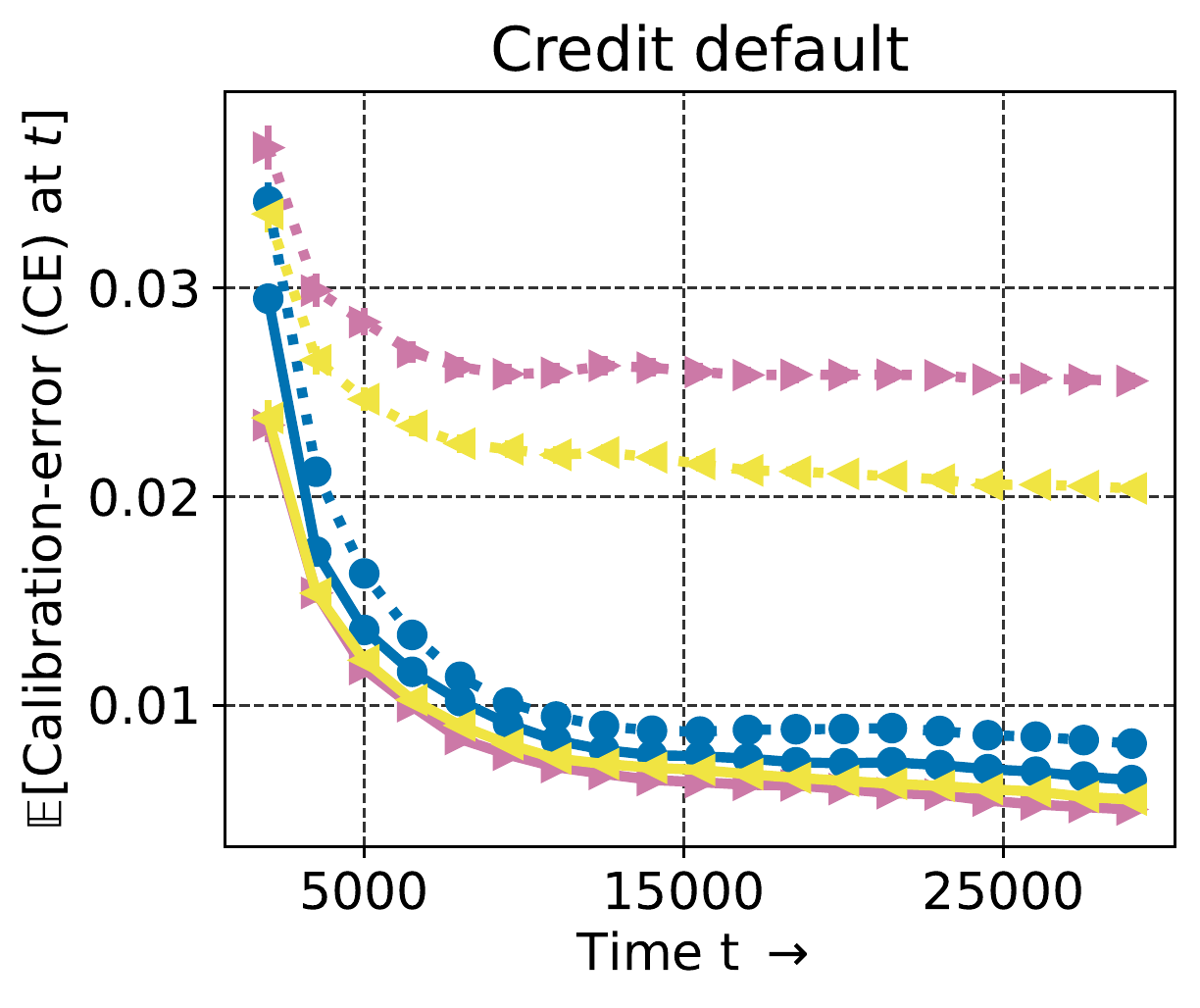}
        \includegraphics[trim=0 0 0 0, clip, width=0.24\linewidth]{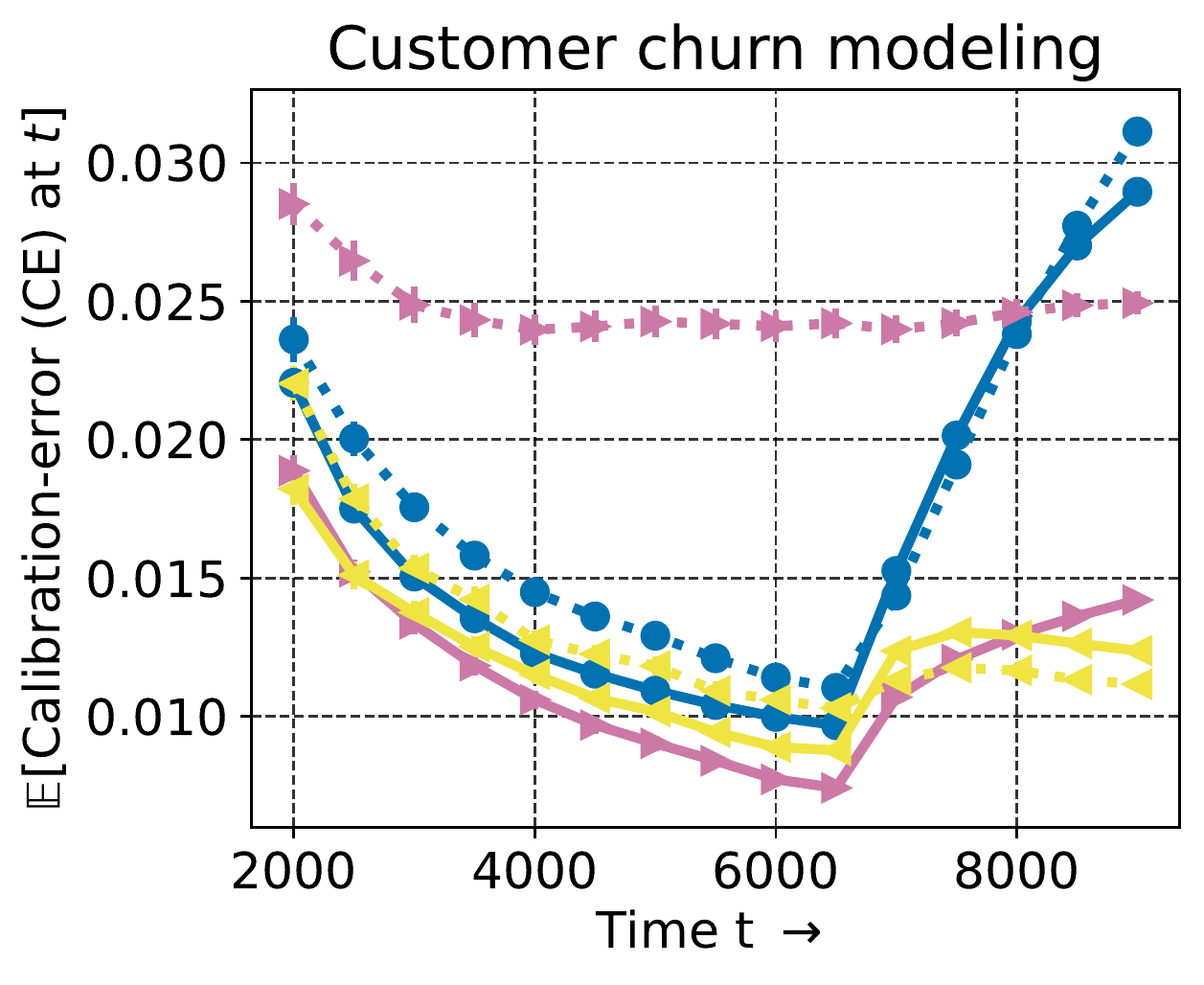}
        \includegraphics[trim=0 0 0 0, clip, width=0.24\linewidth]{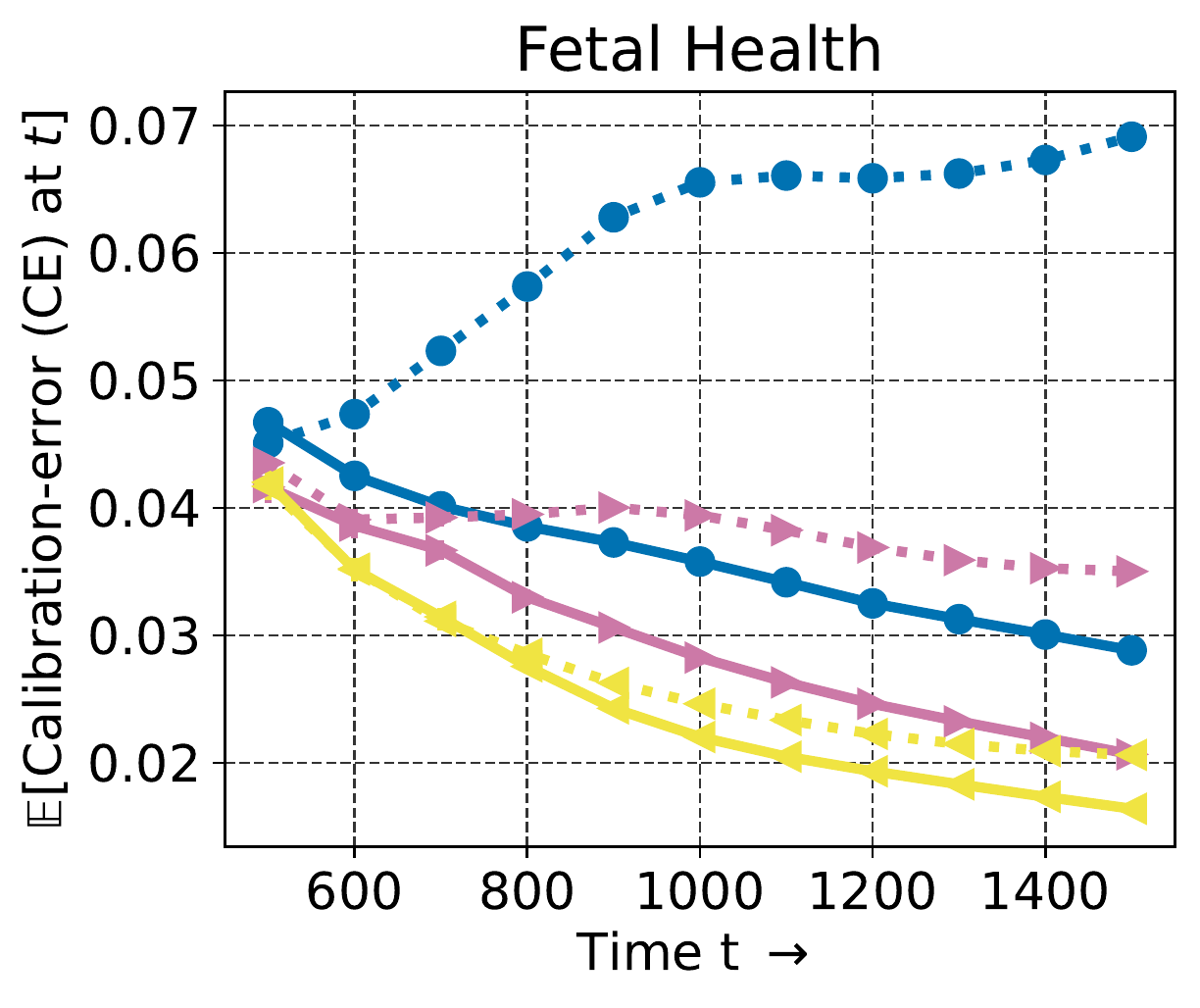}
        \caption{Calibration error for drifting data streams.}
    \end{subfigure}

        \caption{Comparing the performance of windowed histogram binning (WHB), online Platt scaling (OPS), online beta scaling (OBS), and their tracking variants on real datasets with and without distribution drifts. Among non-tracking methods (dotted lines), WHB performs well with i.i.d.\ data, while OBS performs well for drifting data. Among tracking methods (solid lines), TOBS and TOPS are the best-performing methods in every plot. Tracking typically does not improve WHB much since WHB is already a binning method (so tracking is implicit).}
    \label{fig:real-data-bs}
\end{figure*}

\clearpage

\begin{algorithm}[t]
\begin{algorithmic}
	\STATE {\bfseries Input: }$\mathcal{K} = \{(x, y): \|(x, y)\|_2 \leq 100\}$, time horizon $T$, and initialization parameter $(a_1^\ops, b_1^\ops) = (1, 0) =: \theta_1 \in \mathcal{K}$\;
        \STATE {\bfseries Hyperparameters:} $\gamma = 0.1$, $\rho = 100$
    \STATE Set $A_0 = \rho \mathbf{I}_2$
    \FOR{$t=1$ {\bfseries to} $T$}
    \STATE Play $\theta_t$, observe log-loss $l(m^{\theta_t}(f(\x_t)), y_t)$ and its gradient $\nabla_t := \nabla_{\theta_t}l(m^{\theta_t}(f(\x_t)), y_t)$
    \STATE $A_t = A_{t-1} + \nabla_t \nabla_t^\intercal$
    \STATE Newton step: $\widetilde{\theta}_{t+1} = \theta_t - \frac{1}{\gamma} A_t^{-1} \nabla_t$
    \STATE Projection: $(a_{t+1}^\ops, b_{t+1}^\ops) = \theta_{t+1} = \argmin_{\theta \in \mathcal{K}} (\widetilde{\theta}_{t+1}-\theta)^\intercal A_t(\widetilde{\theta}_{t+1}-\theta)$
    \ENDFOR
    \end{algorithmic}
 	\caption{Online Newton Step for OPS (based on \citet[Algorithm 12]{hazan2016introduction})} 
  \label{alg:ops-ons}
\end{algorithm}

\section{Online beta scaling}

\label{appsec:beta-scaling}
This is an extended version of Section~\ref{sec:beta-scaling}, with some repetition but more details. A recalibration method closely related to Platt scaling is beta scaling \citep{kull2017beyond}. The beta scaling mapping $m$ has three parameters $(a, b, c) \in \Real^3$, and corresponds to a sigmoid transform over  two pseudo-features derived from $f(\x)$: $\log(f(\x))$ and $\log(1-f(\x))$:
\begin{equation}
    m^{a, b, c}(f(\x)) := \sigmoid(a\cdot\log(f(\x)) + b\cdot\log(1-f(\x)) + c).%
    \label{eq:bs-model-class}
\end{equation}
Observe that enforcing $b = -a$ recovers Platt scaling since $\logit(z) = \log(z) - \log(1-z)$. The beta scaling parameters can be learnt following identical protocols as Platt scaling.
\begin{itemize}
    \item The \textbf{traditional method} is to optimize parameters by minimizing the log-likelihood (equivalently, log-loss) over a fixed held-out batch of points. 
    \item A \textbf{natural benchmark} for online settings is to update the parameters at some frequency (such as every 50 or 100 steps). At each update, the beta scaling parameters are set to the optimal value based on all data seen so far, and these parameters are used for prediction until the next update occurs. We call this benchmark windowed beta scaling (WBS); it is analogous to the windowed Platt scaling (WPS) benchmark considered in the main paper. 
    \item Our \textbf{proposed method} for online settings, called online Beta scaling (OBS), is to use a log-loss regret minimization procedure, similar to OPS. Analogously to \eqref{eq:regret-definition}, $R_T$ for OBS predictions $p_t^\obs = m^{a_t, b_t, c_t}(f(\x_t))$ is defined as
    \begin{equation}
        R_T(\text{OBS}) = \sum_{t=1}^T l(p_t^\obs, y_t) - \min_{(a, b, c) \in \Bcal}\sum_{t=1}^T l(m^{a,b, c}(f(\x_t)), y_t),
    \end{equation}  
    where $\Bcal := \{(a, b, c) \in \Real^3: a^2 + b^2 + c^2\leq B^2\}$ for some $B \in \Real$, and $l$ is the log-loss. We use online Newton step (Algorithm~\ref{alg:ops-ons}) to learn $(a_t, b_t, c_t)$, with the following initialization and hyperparameter values:  
    \begin{itemize}
        \item $\mathcal{K} = \{(x, y, z): \|(x, y, z)\|_2 \leq 100\}$, $(a_1^\obs, b_1^\obs, c_1^\obs) = (1, 1, 0)$;
        \item $\gamma = 0.1$, $\rho = 25$, $A_0 = \rho \mathbf{I}_3$.
    \end{itemize}
    These minor changes have to be made simply because the dimensionality changes from two to three. The empirical results we present shortly are based on an implementation with exactly these fixed hyperparameter values that do not change across the experiments (that is, we do not do any hyperparameter tuning). 
\end{itemize}
    
Due to the additional degree of freedom, beta scaling is more expressive than Platt scaling. In the traditional batch setting, it was demonstrated by \citet{kull2017beyond} that this expressiveness typically leads to better (out-of-sample) calibration performance. We expect this relationship between Platt scaling and beta scaling to hold for their windowed and online versions as well. We confirm this intuition through an extension of the real dataset experiments of Section~\ref{sec:experiments-real} to include WBS and OBS (Figure~\ref{fig:real-data-tops}). In the main paper we reported that the base model (BM) and fixed-batch Platt scaling model (FPS) perform the worst by a margin, so these lines are not reported again. We find that OBS performs better than both OPS and WBS, so we additionally report the performance of calibeating versions of OBS instead of OPS. That is, we replace OPS + tracking (TOPS) with OBS + tracking (TOBS), and OPS + hedging (HOPS) with OBS + hedging (HOBS).

A regret bound similar to Theorem~\ref{thm:ops-regret} can be derived for OBS by instantiating ONS and AIOLI regret bounds with $d=3$ (instead of $d=2$ as done for OPS). The calibeating theorems (\ref{thm:tops-sharpness-guarantee} and \ref{thm:cops-calibration-guarantee}) hold regardless of the underlying expert, and so also hold for OBS.

\section{F99 online calibration method}
\label{appsec:f99}
We describe the F99 method proposed by \citet{foster1999proof}, and used in our implementation of HOPS (Section~\ref{sec:hops}). The description is borrowed with some changes from \citet{gupta2022faster}. Recall that the F99 forecasts are the mid-points of the $\epsilon$-bins \eqref{eq:B-bins}: $B_1 = [0, \epsilon), B_2 = [\epsilon, 2\epsilon), \ldots, B_m = [1-\epsilon, 1]$. 
For $b \in [m] := \{1, 2, \ldots, m\}$ and $t \geq 1$, define:
\begin{align*}
    \text{(mid-point of  $B_b$) }~& m_b = (b-0.5)/m = b\epsilon - \epsilon/2,\\
    \text{(left end-point of  $B_b$) }~& l_b = (b-1)/m = (b-1)\epsilon, \\
    \text{(right end-point of  $B_b$) }~& r_b = b/m = b\epsilon,\\
\end{align*}  
F99 maintains some quantities as more data set is observed and forecasts are made. These are, 
\begin{align*}
    \text{(frequency of forecasting $m_b$) }~& N^t_b = \abs{\{\indicator{p_s = m_b}: s \leq t\}},\\% n^T_i = N_i^T/T, \\%\text{mean}(\indicator{a_t = b}: t \leq T), \\
    \text{(observed average when $m_b$ was forecasted) }~& p^t_b =\begin{cases} \sum_{s=1}^{t} y_s \indicator{p_s = m_b}/N_b^t ~~\text{ if } N^t_b > 0\\ \text{$m_b$~~~~~~~~~~~~~~~~~~~~~~~~~~~~~if $N^t_b = 0$},\end{cases}\\
    \text{(deficit) }~& d_b^t = l_b - p_b^t,\\
    \text{(excess) }~& e_b^t = p_b^t - r_b.
\end{align*}
The terminology ``deficit" is used to indicate that $p_b^t$ is smaller $l_b$ similarly. ``Excess" is used to indicate that $p_b^t$ is larger than $r_b$ similarly. The F99 algorithm is as follows. Implicit in the description is computation of the quantities defined above.

\noindent\fbox{%
    \parbox{\textwidth}{%
    \begin{center}
    \textbf{F99: the online adversarial calibration method of \citet{foster1999proof}}
    \end{center}
    \begin{itemize}
        \item 
    At time $t = 1$, forecast $p_{1} = m_1$.  
    \item At time $t+1$ ($t \geq 1)$, if
\begin{align*}
    \text{condition A: there exists an $b \in [m]$ such that $d_b^t \leq 0$ and $e_b^t \leq 0$,}
\end{align*} 
is satisfied, forecast $p_{t+1} = m_b$ for any $i$ that verifies condition A. Otherwise, 
\begin{align*}
\text{condition B: there exists a $b \in [m-1]$ such that $e_b^t > 0$ and $d_{b+1}^t > 0$,}
\end{align*}
must be satisfied (see Lemma 5 \citep{gupta2022faster}). For any index $b$ that satisfies condition B, forecast 
\begin{equation*}
    p_{t+1} =
    \left\{
    	\begin{array}{ll}
    		m_b  & \mbox{ with probability } \frac{d_{b+1}^t}{d_{b+1}^t + e_b^t} \\
    		m_{b+1} & \mbox{ with probability } \frac{e_b^t}{d_{b+1}^t + e_b^t}.
    	\end{array}
    \right.
\end{equation*}
These randomization probabilities are revealed before $y_{t+1}$ is set by the agent that is generating outcomes, but the actual $p_t$ value is drawn after $y_{t+1}$ is revealed. 
\end{itemize}
    }%
}

\begin{figure}[h]
    \centering
    \includegraphics[width=0.8\linewidth]{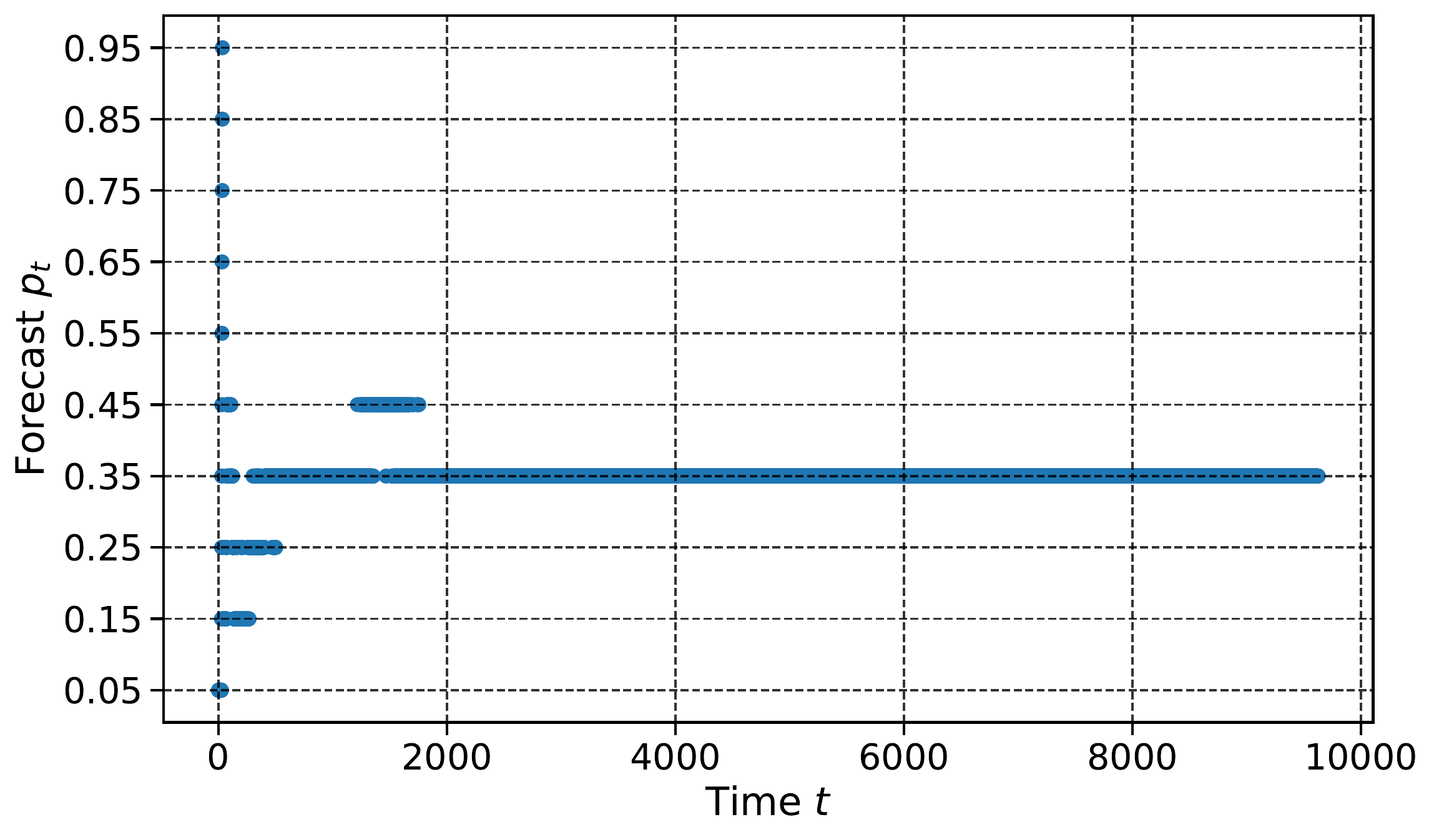}
    \caption{\citet{foster1999proof}'s $\epsilon$-calibrated forecaster on Pittsburgh's hourly rain data (2008-2012). The forecaster makes predictions on the grid $(0.05, 0.15, \ldots, 0.95)$. In the long run, the forecaster starts predicting $0.35$ for every instance, closely matching the average number of instances on which it rained ($\approx 0.37$).}
    \label{fig:foster99-pfg-rain}
\end{figure}

\section{Forecasting climatology to achieve calibration}
\label{appsec:climatology-experiment}
Although Foster and Vohra's result \citeyearpar{foster1998asymptotic} guarantees that calibrated forecasting is possible against adversarial sequences, this does not immediately imply that the forecasts are useful in practice. To see this, consider an alternating outcome sequence, $y_t = \indicator{t \text{ is odd }}$. The forecast $p_t = \indicator{t \text{ is odd }}$ is calibrated and perfectly accurate. The forecast $p_t = 0.5$ (for every $t$) is also calibrated, but not very useful.

Thus we need to assess how a forecaster guaranteed to be calibrated for adversarial sequences performs on real-world sequences. In order to do so, we implemented the F99 forecaster (described in Appendix~\ref{appsec:f99}), on Pittsburgh's hourly rain data from January 1, 2008, to December 31, 2012. The data was obtained from \url{ncdc.noaa.gov/cdo-web/}. All days on which the hourly precipitation in inches (HPCP) was at least $0.01$ were considered as instances of $y_t = 1$. There are many missing rows in the data, but no complex data cleaning was performed since we are mainly interested in a simple illustrative simulation. F99 makes forecasts on an $\epsilon$-grid with $\epsilon = 0.1$: that is, the grid corresponds to the points $(0.05, 0.15, \ldots, 0.95)$. We observe (Figure~\ref{fig:foster99-pfg-rain}) that after around $2000$ instances, the forecaster \emph{always} predicts $0.35$. This is close to the average number of instances that it did rain which is approximately $0.37$ (this long-term average is also called \emph{climatology} in the meteorology literature). Although forecasting climatology can make the forecaster appear calibrated, it is arguably not a useful prediction given that there exist expert rain forecasters who can make sharp predictions for rain that change from day to day.

\section{Proofs}
\label{appsec:proofs}

\begin{proof}[Proof of Theorem~\ref{thm:ops-regret}]
The regret bounds for ONS and AIOLI depend on a few problem-dependent parameters. %
\begin{itemize}
    \item The dimension $d = 2$.
    \item The radius of the reference class $B$. 
    \item Bound on the norm of the gradient, which for logistic regression is also the radius of the space of input vectors. Due to the assumption on $f(\x_t)$, the norm of the input is at most $\sqrt{\logit(0.01)^2 + 1^2} = \sqrt{\logit(0.99)^2 + 1^2} \leq 5$. 
\end{itemize}

The AIOLI bound \eqref{eq:aioli-regret} follows from Theorem 1, equation (4) of \citet{jezequel2020efficient}, setting $d=2$ and $R=10$.

The ONS bound \eqref{eq:ons-regret} follows from Theorem 4.5 of \citet{hazan2016introduction}, plugging in $G = 5$, $D = 2B$, and $\alpha = e^{-B}$ which is the known exp-concavity constant of the logistic loss over a ball of radius $B$ \citep{foster2018logistic}. 

\end{proof}

In writing the proofs of the results in Section~\ref{sec:calibeating}, we will use an object closely connected to sharpness called refinement. For a sequence of forecasts $p_{1:T}$ and outcome sequence $y_{1:T}$, the refinement $\Rcal$ is defined as \begin{equation}
\vspace{-0.1cm}
    \Rcal(p_{1:T}) := \frac{1}{T}\sum_{b=1}^m N_b\cdot \widehat{y}_b(1-\widehat{y}_b),\label{eq:refinement-defn}
\end{equation}
where $\widehat{y}_b$ is the average of the outcomes in every $\epsilon$-bin; see the beginning of Section~\ref{subsec:sharpness-defn} where sharpness is defined. The function $x (\in [0,1]) \mapsto x(1-x)$ is minimized at the boundary points $\{0, 1\}$ and maximized at $1/2$. Thus refinement is lower if $\widehat{y}_b$ is close to $0$ or $1$, or in other words if the bins discriminate points well. This is captured formally in the following (well-known) relationship between refinement and sharpness. 

\begin{lemma}[Sharpness-refinement lemma]\label{lemma:shp-ref}
    For any forecast sequence $p_{1:T}$, the refinement $\Rcal$ defined in \eqref{eq:refinement-defn} and the sharpness $\shp$ defined in \eqref{eq:sharpnes-definition} are related as: 
    \begin{equation*}
        \Rcal(p_{1:T}) = \Bar{y}_T - \shp(p_{1:T}), 
    \end{equation*}
    where $\Bar{y}_T = \frac{1}{T}\sum_{t=1}^T y_t$. 
\end{lemma}
\begin{proof}
    Observe that \[\Rcal(p_{1:T}) = \frac{1}{T}\sum_{b=1}^B N_b \widehat{y}_b - \frac{1}{T}\sum_{b=1}^B N_b \widehat{y}_b^2 = \frac{1}{T}\sum_{b=1}^B N_b \widehat{y}_b  - \shp(p_{1:T}).\] 
    The final result follows simply by noting that \[\sum_{b=1}^B N_b \widehat{y}_b = \sum_{b=1}^B\roundbrack{\sum_{t\leq T, p_t \in B_b} y_t} = \sum_{t=1}^T y_t.\] 
\end{proof}

We now state a second lemma, that relates $\Rcal$ to the Brier-score $\mathcal{BS}$ defined as
\begin{equation}
    \mathcal{BS}(p_{1:T}) := \frac{\sum_{t = 1}^T (y_t-p_t)^2}{T}.\label{eq:bs-defn}
\end{equation}
Unlike $\Rcal$ and $\shp$, $\mathcal{BS}$ is not defined after $\epsilon$-binning. It is well-known (see for example equation (1) of FH23) that if refinement is defined without $\epsilon$-binning (or if the Brier-score is defined with $\epsilon$-binning), then refinement is at most the Brier-score defined above. Since we define $\Rcal$ defined with binning, %
further work is required to relate the two.
\begin{lemma}[Brier-score-refinement lemma]
    \label{lemma:bs-ref}
   For any forecast sequence $p_{1:T}$ and outcome sequence $y_{1:T}$, the refinement $\Rcal$ and the Brier-score $\mathcal{BS}$ are related as 
   \begin{equation}
       \Rcal(p_{1:T}) \leq \mathcal{BS}(p_{1:T}) + \frac{\epsilon^2}{4} + \epsilon,
   \end{equation}
   where $\epsilon$ is the width of the bins used to define $\Rcal$ \eqref{eq:B-bins}. 
\end{lemma}
\begin{proof}
    Define the discretization function $\disc : [0,1] \to [0,1]$ as $\disc(p) = \text{mid-point}(B_b) \iff p \in B_b$. Note that for all $p \in [0,1]$, $\abs{p - \disc(p)} \leq \epsilon/2$. Based on standard decompositions (such as equation (1) of FH23), we know that
     \begin{equation}
         \Rcal(p_{1:T}) \leq \frac{\sum_{t = 1}^T (y_t - \disc(p_t^\tops))^2}{T}.\label{eq:bs-refinement} 
     \end{equation}
     We now relate the RHS of the above equation to $\mathcal{BS}$
    \begin{align*}
        \sum_{t = 1}^T (y_t - \disc(p_t))^2 &= \sum_{t = 1}^T (y_t - p_t + p_t - \disc(p_t))^2
        \\ &= T\cdot\mathcal{BS}(p_{1:T}) + \sum_{t = 1}^T (p_t - \disc(p_t))^2 + 2\sum_{t = 1}^T(y_t - p_t)(p_t - \disc(p_t))
        \\ &\leq T\cdot\mathcal{BS}(p_{1:T}) + T(\epsilon/2)^2 + 2\sum_{t = 1}^T\abs{y_t - p_t}(\epsilon/2). 
        \\ &\leq T\cdot\mathcal{BS}(p_{1:T}) + T(\epsilon/2)^2 + T\epsilon. 
    \end{align*}
The result of the theorem follows by dividing by $T$ on both sides. 
\end{proof}

\begin{proof}[Proof of Theorem~\ref{thm:tops-sharpness-guarantee}]
    The calibeating paper \citep{foster2022calibeating} is referred to as FH23 in this proof for succinctness.

    We use Theorem 3 of FH23, specifically equation (13), which gives an upper bound on the Brier-score of a tracking forecast ($\Bcal_t^\mathbf{c}$ in their notation) relative to the refinement \eqref{eq:refinement-defn} of the base forecast. In our case, the tracking forecast is TOPS, the base forecast is OPS, and 
    FH23's result gives, 
    \begin{equation}
        \mathcal{BS}(p_{1:T}^\tops) = \frac{\sum_{t = 1}^T (y_t - p_t^\tops)^2}{T} \leq \Rcal(p_{1:T}^\tops) + \frac{\log T + 1}{\epsilon T}.
    \end{equation}
    Using the Brier-score-refinement lemma~\ref{lemma:bs-ref} to lower bound $\mathcal{BS}(p_{1:T}^\tops)$ gives 
    \begin{equation}
        \mathcal{R}(p_{1:T}^\tops) - \frac{\epsilon^2}{4} - \epsilon \leq \mathcal{R}(p_{1:T}^\ops)  + \frac{\log T + 1}{\epsilon T}.
    \end{equation}
    Finally, using the sharpness-refinement lemma \ref{lemma:shp-ref}, we can replace each $\Rcal$ with $\bar{y}_T - \shp$. Rearranging terms gives the final bound.

\end{proof}

\begin{proof}[Proof of Theorem~\ref{thm:cops-calibration-guarantee}] The calibeating paper \citep{foster2022calibeating} is referred to as FH23 in this proof for succinctness. 

   \textbf{Sharpness bound \eqref{eq:cops-sharpness-guarantee}.}  %
   Theorem 5 of FH23 showed that the expected Brier-score for a different hedging scheme (instead of F99), is at most the expected refinement score of the base forecast plus $\epsilon^2 + \frac{\log T + 1}{\epsilon^2 T}$. In our case, the second term remains unchanged, but because we use F99, the $\epsilon^2$ needs to be replaced, and we show that it can be replaced by $\epsilon$ next. 
   
   Let us call the combination of OPS and the FH23 hedging method as FH23-HOPS, and the calibeating forecast as $p_{1:T}^\text{FH23-HOPS}$. The source of the $\epsilon^2$ term in Theorem 5 of FH23 is the following property of FH23-HOPS: for both values of $y_t \in \{0, 1\}$, 
   \begin{equation*}
       \Exp{t-1}{(y_t - p_t^\text{FH23-HOPS})^2 - (y_t - \text{Average}\{y_s: s < t, p_s^\ops = p_t^\ops, p_s^\text{FH23-HOPS} = p_t^\text{FH23-HOPS}\})^2} \leq \epsilon^2,
   \end{equation*}
   where $\Exp{t-1}{\cdot}$ is the expectation conditional on $(y_{1:t-1}, p^\text{FH23-hedging}_{1:t-1}, p^\ops_{1:t-1})$ (all that's happened in the past, and the current OPS forecast). For HOPS, we will show that 
   \begin{equation*}
       Q_t := \Exp{t-1}{(y_t - p_t^\text{HOPS})^2 - (y_t - \text{Average}\{y_s: s < t, p_s^\ops = p_t^\ops, p_s^\text{HOPS} = p_t^\text{HOPS}\})^2} \leq \epsilon,
   \end{equation*}
   for $y_t \in \{0, 1\}$, which would give the required result. 

   At time $t$, the F99 forecast falls into one of two scenarios which we analyze separately (see Appendix~\ref{appsec:f99} for details of F99 which would help follow the case-work). 
   \begin{itemize}
       \item 
    \textbf{Case 1. }This corresponds to condition A in the description of F99 in Section~\ref{appsec:f99}. There exists a bin index $b$ such that $q = \text{mid-point}(B_b)$ satisfies \[\abs{\text{Average}\{y_s: s < t, p_s^\ops = p_t^\ops, p_s^\text{HOPS} = q\} - q} \leq \epsilon/2.\]
    In this case, F99 would set $p_t^\text{HOPS} = q$ (deterministically) for some $q$ satisfying the above. Thus,
    \begin{align*}
        Q_t &= (y_t - q)^2 - (y_t - \text{Average}\{y_s: s < t, p_s^\ops = p_t^\ops, p_s^\text{HOPS} = q\})^2 
        \\ &\leq \max((y_t-q)^2-(y_t-q-\epsilon/2)^2, (y_t-q)^2-(y_t-q+\epsilon/2)^2)
        \\ &\leq (\epsilon/2)(2\abs{y_t-q} + \epsilon/2) < \epsilon, %
    \end{align*}
    irrespective of $y_t$, since $q \in [\epsilon/2, 1 - \epsilon/2]$. 
    \item 
    \textbf{Case 2. }This corresponds to condition B in the description of F99 in Section~\ref{appsec:f99}. 
   If Case 1 does not hold, F99 randomizes between two consecutive bin mid-points $m - \epsilon/2$ and $m - \epsilon/2$, where $m$ is one of the edges of the $\epsilon$-bins \eqref{eq:B-bins}. Define %
   $n_1 := \text{Average}\{y_s: s < t, p_s^\ops = p_t^\ops, p_s^\text{HOPS} = m-\epsilon/2\}$ and $n_2 := \text{Average}\{y_s: s < t, p_s^\ops = p_t^\ops, p_s^\text{HOPS} = m+\epsilon/2\}$. The choice of $m$ in F99 guarantees that $n_2 < m < n_1$, and the randomization probabilities are given by
   \begin{equation*}
       \mathbb{P}_{t-1}(p_t^\hops = m-\epsilon/2) = \frac{m - n_2}{n_1 - n_2}, \text{ and }\mathbb{P}_{t-1}(p_t^\hops = m+\epsilon/2) = \frac{n_1 - m}{n_1 - n_2},
   \end{equation*}
where $\mathbb{P}_{t-1}$ is the conditional probability in the same sense as $\mathbb{E}_{t-1}$. We now bound $Q_t$. If $y_t = 1$, 
\begin{align*}
    Q_t &= \Exp{t-1}{(y_t - p_t^\text{HOPS})^2 - (y_t - \text{Average}\{y_s: s < t, p_s^\ops = p_t^\ops, p_s^\text{HOPS} = p_t^\text{HOPS}\})^2} 
    \\ &= \frac{m - n_2}{n_1 - n_2} \roundbrack{(1-(m-\epsilon/2))^2 - (1-n_1)^2} + \frac{n_1 - m}{n_1 - n_2} \roundbrack{(1-(m+\epsilon/2))^2 - (1-n_2)^2}
    \\ &= \underbrace{\frac{m - n_2}{n_1 - n_2} \roundbrack{(1-m)^2 - (1-n_1)^2} + \frac{n_1 - m}{n_1 - n_2} \roundbrack{(1-m)^2 - (1-n_2)^2}}_{=: A_1} 
    \\ &\qquad + \underbrace{2\cdot (\epsilon/2) \cdot \frac{(m-n_2)(1-m) - (n_1-m)(1-m)}{n_1 - n_2} }_{=: A_2} 
    \\ &\qquad + (\epsilon/2)^2 \cdot \frac{n_1-n_2}{n_1 - n_2}. 
\end{align*}
$A_1$ and $A_2$ simplify as follows.
\begin{align*}
     A_1  &= \frac{(m - n_2)(n_1 - m)(2 - (n_1 + m))  + (n_1 - m)(n_2 - m)(2 - (n_2 + m))}{n_1 - n_2}
     \\ &= \frac{(m - n_2)(n_1 - m)(n_2 - n_1)}{n_1 - n_2} < 0, 
\end{align*}
since $n_2 < m < n_1$. 
\begin{align*}
     A_2  &= \epsilon\cdot\frac{(m - n_2)(1-m)}{n_1 - n_2} + \epsilon\cdot\frac{(m - n_1)(1-m)}{n_1 - n_2} %
     \\&< \epsilon\cdot\frac{(m - n_2)(1-m)}{n_1 - n_2} & \text{(since $m < n_1$)}
     \\& < \epsilon(1-m).
\end{align*}
Overall, we obtain that for $y_t = 1$, 
\begin{equation*}
    Q_t  %
    < \epsilon(1-m)  + (\epsilon^2/4) < \epsilon, 
\end{equation*}
where the final inequality holds since $m$ is an end-point between two bins, and thus $m \geq \epsilon$. We do the calculations for $y_t = 0$ less explicitly since it essentially follows the same steps: 
\begin{align*}
    Q_t &= \Exp{t-1}{(0 - p_t^\text{HOPS})^2 - (0 - \text{Average}\{y_s: s < t, p_s^\ops = p_t^\ops, p_s^\text{HOPS} = p_t^\text{HOPS}\})^2} 
    \\ &= \frac{m - n_2}{n_1 - n_2} \roundbrack{(m-\epsilon/2)^2 - n_1^2} + \frac{n_1 - m}{n_1 - n_2} \roundbrack{(m+\epsilon/2)^2 - n_2^2}
    \\ &= \frac{(m - n_2)(m-n_1)(m+n_1) + (n_1 - m)(m-n_2)(m+n_2)}{n_1 - n_2} %
    +\epsilon \cdot \frac{(n_2-m)m + (n_1-m)m}{n_1 - n_2} + \frac{\epsilon^2}{4}
    \\ &< 0 + \epsilon m + (\epsilon^2/4) < \epsilon. 
\end{align*}
   \end{itemize}
Finally, by Proposition 1 of FH23 and the above bound on $Q_t$, we obtain, 
\begin{equation}
    \Exp{}{\Rcal(p_{1:T}^\hops)} \leq \Exp{}{\mathcal{BS}(p_{1:T}^\hops)} \leq \epsilon + \Rcal(p_{1:T}^\ops) + \frac{\log T + 1}{\epsilon^2 T}. \label{eq:result-reused-for-bs-thm}
\end{equation}
Using the sharpness-refinement lemma \ref{lemma:shp-ref}, we replace each $\Rcal$ with $\bar{y}_T - \shp$. Rearranging terms gives the sharpness result.

\textbf{Calibration bound \eqref{eq:cops-calibration-guarantee}.} %
Recall that the number of bins is $m = 1/\epsilon$. For some bin indices $b, b' \in \{1, 2, \ldots, m\}$, let $S_{b \to b'} = \{t \leq T : p_t^\ops \in B_b, p_t^{\hops} = \text{mid-point}(B_{b'})\}$ be the set of time instances at which the \ops~forecast $p_t^\ops$ belonged to bin $b$, but the \hops~forecast $p_t^\hops$ belonged to bin $b'$ (and equals the mid-point of bin $b'$). Also, let $S_b = \{t \leq T : p_t^\ops \in B_b\}$ be the set of time instances at which the $p_t^\ops$ forecast belonged to bin $b$. Thus $S_b = \bigcup_{b'=1}^m S_{b \to b'}$. Also define $N_b^\ops = \abs{S_b}$ and $N_b^\hops = \abs{\{t \leq T : p_t^\hops = \text{mid-point}(B_{b})\}}$. %

Now for any specific $b$, consider the sequence $(y_t)_{t \in S_b}$. On this sequence, the \hops~forecasts correspond to F99 using just the outcomes (with no regard for covariate values once the bin of $p_t^\ops$ is fixed). Thus, within this particular bin, we have a usual CE guarantee that F99's algorithm has for any arbitrary sequence:
\begin{equation}
    \underbrace{\Exp{}{\frac{1}{N_b^\ops}\sum_{b' = 1}^m \abs{\sum_{t \in S_{b\to b'}} (y_t - p_t^\hops)}}}_{\text{this is the expected CE over the $S_b$ instances}} \leq \frac{\epsilon}{2} + \frac{2}{\sqrt{\epsilon \cdot N_b^\ops}}. \label{eq:f99-inside-b}
\end{equation}
This result is unavailable in exactly this form in \citet{foster1999proof} which just gives the reduction to Blackwell approachability, after which any finite-sample approachability bound can be used. The above version follows from Theorem~1.1 of \citet{perchet2014approachability}. The precise details of the Blackwell approachability set, reward vectors, and how the distance to the set can be translated to CE can be found in \citet[Section 4.1]{gupta2022faster}. 

Jensen's inequality can be used to lift this CE guarantee to the entire sequence:%
\begin{align*}
    \Exp{}{\ce(p_{1:T}^\hops)} &= \Exp{}{\sum_{b=1}^m \frac{\abs{N_b^\hops(\widehat{y}_b^\hops - \widehat{p}_b^\hops) }}{T}}
    \\ &= \Exp{}{\sum_{b=1}^m \frac{\abs{\sum_{t =1}^T(y_t-p_t^\hops) \indicator{p_t^\hops \in B_b}}}{T}}
    \\ &= \Exp{}{\sum_{b=1}^m \frac{\abs{\sum_{b'=1}^m\sum_{t \in S_{b' \rightarrow b}}(y_t-p_t^\hops)}}{T}}
    \\ &\leq \Exp{}{\sum_{b=1}^m \sum_{b'=1}^m \frac{\abs{\sum_{t \in S_{b' \rightarrow b}}(y_t-p_t^\hops)}}{T}} \qquad  \text{(Jensen's inequality)}
    \\ &= \sum_{b'=1}^m\Exp{}{\sum_{b=1}^m  \frac{\abs{\sum_{t \in S_{b' \rightarrow b}}(y_t-p_t^\hops)}}{T}}
    \\ &\leq \sum_{b'=1}^m \frac{N_{b'}^\ops \roundbrack{\epsilon/2 + 2/\sqrt{\epsilon \cdot N_{b'}^\ops}}}{T} \qquad\qquad \text{(by \eqref{eq:f99-inside-b})}
    \\ &= \frac{\epsilon}{2} + \frac{2}{\sqrt{\epsilon}} \cdot \frac{\sum_{b'=1}^m\sqrt{N_{b'}^\ops}}{\sum_{b'=1}^m N_{b'}^\ops} \qquad\qquad \text{(since $T = \sum_{b'=1}^B N_{b'}^\ops$)}
    \\ &\overset{(\star)}{\leq} \frac{\epsilon}{2} + \frac{2}{\sqrt{\epsilon}}\cdot\sqrt{\frac{m}{T}} = \frac{\epsilon}{2} + 2\sqrt{\frac{1}{\epsilon^2 T}},
\end{align*}
as needed to be shown. The inequality $(\star)$ holds because, by Jensen's inequality (or AM-QM inequality),
\[
\sqrt{\frac{\sum_{b'=1}^m N_{b'}^\ops}{m}} \geq \frac{\sum_{b'=1}^m \sqrt{N_{b'}^\ops}}{m},\] 
so that 
\[\frac{\sum_{b'=1}^m\sqrt{N_{b'}^\ops}}{\sum_{b'=1}^m N_{b'}^\ops}  = \frac{\sum_{b'=1}^m\sqrt{N_{b'}^\ops}}{\sqrt{\sum_{b'=1}^m N_{b'}^\ops}}\cdot\frac{1}{\sqrt{\sum_{b'=1}^m N_{b'}^\ops}} \leq \frac{\sqrt{m}}{\sqrt{\sum_{b'=1}^m N_{b'}^\ops}} = \sqrt{m/T}.
\]
\end{proof}

\begin{theorem}
For adversarially generated data, the expected Brier-score of HOPS forecasts using the forecast hedging algorithm of \citet{foster1999proof} is upper bounded as 
\begin{equation}%
    \Exp{}{\mathcal{BS}(p_{1:T}^\hops)} \leq \mathcal{BS}(p_{1:T}^\ops) + \roundbrack{2\epsilon + \frac{\epsilon^2}{4}
    + \frac{\log T+1}{\epsilon^2T}}.\label{eq:cops-brier-score-guarantee}
\end{equation}
\label{thm:cops-brier-score-guarantee}
\end{theorem}
\begin{proof}
    In the proof of the sharpness result of Theorem~\ref{thm:cops-calibration-guarantee}, we showed equation~\eqref{eq:result-reused-for-bs-thm}, which immediately yields \eqref{eq:cops-brier-score-guarantee} since $\Rcal(p_{1:T}^\ops) \leq \mathcal{BS}(p_{1:T}^\ops) + \epsilon + \epsilon^2/4$ by Lemma~\ref{lemma:bs-ref}.
\end{proof}

\clearpage

\end{document}